\documentclass[a4paper, 10pt]{article}

\usepackage[margin=1in]{geometry}
\usepackage[utf8]{inputenc}
\usepackage[english]{babel}
\usepackage[toc,acronym,nonumberlist,numberline,shortcuts]{glossaries}
\usepackage[labelfont=bf]{subcaption} 
\DeclareCaptionLabelSeparator{nspace}{\hspace{4mm}} 
\usepackage{url}
\usepackage{hyperref}
\usepackage{multirow}
\usepackage{enumerate}
\usepackage{makeidx}  
\usepackage{graphics} 
\usepackage{epsfig}  
\usepackage{amsmath} 
\usepackage{amssymb} 
\usepackage{mdwlist} 
\usepackage{amsthm}
\usepackage[round]{natbib}
\usepackage[]{algorithm2e}
\usepackage{placeins}
\usepackage{float}
\usepackage{mathtools,tabularx,ragged2e} 
\DeclarePairedDelimiter{\abs}{\lvert}{\rvert}
\DeclareMathAlphabet{\mathcal}{OMS}{cmsy}{m}{n}

\mathcode`\*="8000
{\catcode`\*\active\gdef*{\cdot}}
\usepackage[T1]{fontenc}
\fontfamily{lmss}\selectfont
\newcommand{\figref}[1]{Fig.~\ref{#1}}

\newcommand{\tabref}[1]{Tab.~\ref{#1}}
\newcommand{\secref}[1]{Sect.~\ref{#1}} 
\newcommand{\eqnref}[1]{Eq.~\ref{#1}} 

\theoremstyle{definition}\newtheorem{definition}{Definition}[section]
\theoremstyle{lemma}
\theoremstyle{theorem}
\theoremstyle{theorem}\newtheorem{proposition}{Proposition}
\theoremstyle{definition}\newtheorem{remark}{Remark}

\title{\LARGE \bf 
A model-based framework for learning transparent swarm behaviors
}
\author{Mario Coppola, 
Jian Guo, 
Eberhard Gill, 
Guido C. H. E. de Croon}

\date{}
\begin{document}
\maketitle


\vspace{2cm}

\begin{abstract}
  This paper proposes a model-based framework to automatically and efficiently design understandable and verifiable behaviors for swarms of robots.
  The framework is based on the automatic extraction of two distinct models: 
  1) a neural network model trained to estimate the relationship between the robots' sensor readings and the global performance of the swarm, and 
  2) a probabilistic state transition model that explicitly models the local state transitions (i.e., transitions in observations from the perspective of a single robot in the swarm) given a policy.
  The models can be trained from a data set of simulated runs featuring random policies.
  The first model is used to automatically extract a set of local states that are expected to maximize the global performance.
  These local states are referred to as \emph{desired} local states.
  The second model is used to optimize a stochastic policy so as to increase the probability that the robots in the swarm observe one of the desired local states.
  Following these steps, the framework proposed in this paper can efficiently lead to effective controllers.
  This is tested on four case studies, featuring aggregation and foraging tasks.
  Importantly, thanks to the models, the framework allows us to understand and inspect a swarm's behavior.
  To this end, we propose verification checks to identify some potential issues that may prevent the swarm from achieving the desired global objective.
  In addition, we explore how the framework can be used in combination with a ``standard'' evolutionary robotics strategy (i.e., where performance is measured via simulation), or with online learning.
\end{abstract}

\newpage

\section{Introduction}
\label{sec:introduction}

  The goal of swarm robotics is to design behaviors that enable several relatively simple robots to collaborate toward a common objective.
  The complexity of this paradigm stems from the fact that each robot can only sense and act according to local information, yet the goals are only observable at the global level.
  The complexity of swarming makes the manual design of successful controllers difficult.
  Machine learning approaches, however, offer an attractive way to do so automatically.
  State of the art techniques in the field primarily adopt evolutionary algorithms that optimize a policy with respect to a centrally measured objective \citep{brambilla2013swarm, francesca2016automatic}.
  Other learning methods, such as reinforcement learning, have received less attention.
  This is because of issues such as low sample efficiency or credit assignment problems, whereby it is not clear how to relate a swarm's global performance with the local states and actions of its individual robots.
  Alternatively, with evolutionary strategies, the global performance of a population of controllers is assessed across multiple generations in order to find a suitable candidate.
  The approach has been used to optimize multiple controller architectures, including neural networks \citep{duarte2016evolution} and behavior trees \citep{Jones2018,winfield2019understandableswarmbehaviors}.
  Neural network controllers can approximate complex functions efficiently.
  However, their ``black box'' architecture can be a disadvantage in this context, particularly for safety-critical control.
  They are difficult to interpret and understand without resorting to experimentation on the swarm.
  Behavior trees and other explicit controllers are more transparent.
  Behavior trees, in particular, feature an attractive mixture of expressiveness and understandability \citep{collendanchise2017bts}.
  They are human-readable and can be understood, inspected, and even fine-tuned.
  In prior research on single robots, this has also given the ability to modify a behavior by hand to better transfer from simulation to the real world \citep{scheper2016behaviortreesforevolution}.
  Nevertheless, understanding a behavior tree within the context of a swarm of robots also requires careful qualitative and experimental analysis \citep{winfield2019understandableswarmbehaviors}.
  This is because the local state transitions experienced by a robot are not modeled.
  The local objectives of the robots, and how these correlate with the global performance, are also not explicitly known.

  This paper proposes a novel model-based framework in order to automatically design the behavior of a swarm of robots.
  The framework makes use of two separate models:
  1) a neural network model that maps the local states of the robots in the swarm to the global performance value, and 
  2) a probabilistic transition model that describes the local transitions as experienced by a single robot in the swarm.
  Both models are trained automatically from a data set of random behaviors.
  The first model is used to extract the local states that contribute to maximizing the swarm's global performance.
  We refer to these local states as \emph{desired} local states.
  These are extracted automatically, reducing a global goal to locally observable constituents.
  Then, the second model is used to determine a policy that maximizes the probability that a robot in the swarm transitions to any one of the desired local states.
  The transition model enables us to examine the behavior of the swarm, given a policy.
  A set of conditions is proposed in order to identify, based on the models, potential issues that could prevent the swarm from achieving a collective goal.
  Overall, the proposed framework has two main advantages: 
  1) the extraction of understandable and verifiable controllers, and  
  2) increased evaluation and optimization efficiency through the models.
  The framework can also be combined with a standard evolutionary algorithm (i.e., one which uses simulations to measure the swarm's performance) for increased efficiency.
  It can also be used online, such that each robot in the swarm locally optimizes its behavior according to a model of its experiences that it generates onboard.
  Implementations of the framework in both of these contexts are also provided in this paper.

  This paper is structured as follows.
  In \secref{sec:ch5_relatedwork}, we place our contribution in the context of recent swarm robotics and machine learning research.
  In \secref{sec:ch5_method}, the model-based framework is explained. 
  In \secref{sec:ch5_results}, its performance is analyzed via four case studies, involving aggregation and foraging.
  It is also shown how the proposed model-based framework can be used in:
  1) a hybrid model-based evolutionary algorithm, and 
  2) a model-based online learning approach.
  \secref{sec:ch5_discussion} discusses the advantages and limitations of the framework, highlighting potential future extensions and research.
  \secref{sec:b5_conclusion} provides concluding remarks.


\section{Related work and research context}
\label{sec:ch5_relatedwork}

  This work explores how to automatically design understandable and verifiable swarm behaviors.
  This is required in order to examine whether a swarm of robots, given a policy, will achieve the intended performance once deployed.
  The recent works by \cite{Jones2018,winfield2019understandableswarmbehaviors} studied how to automatically evolve behavior trees for a similar purpose.
  However, albeit behavior trees are human-readable, they do not model the effects of local actions and their impact.
  In this work, Probabilistic Finite State Machines (PFSMs) are used to dictate the behavior.
  PFSMs are transparent architectures that offer one significant advantage for this context: they probabilistically model the local state transitions due to actions.
  This is a considerable benefit when it comes to understanding and analyzing a behavior.
  For instance, the works by \cite{martinoli2003stickpulling}, \cite{correll2006collective}, and \cite{winfield2010adaptiveforaging} showed how a PFSM, directly extracted from the Finite State Machine used by each robot, can be used to predict a swarm's global behavior over time.
  In the work by \cite{coppola2019provable}, a PFSM model was used to verify whether a swarm will eventually achieve its goal.
  PFSM models have also been used to optimize the policy of a swarm \citep{berman2009optimized, berman2011optimization, coppola2019pagerank}.
  In this paper, following the optimization methodology previously published in \cite{coppola2019pagerank}, the policy will be optimized by using a probabilistic state machine that models the local state transitions of a single robot in a swarm.
  However, the probabilistic model is now automatically extracted from experience data.

  To automatically learn a swarm behavior, it is necessary to have a metric that measures the performance of the whole swarm.
  In state of the art automatic design methods, this metric is typically measured centrally, and directly used to assess the performance of a specific policy.
  However, this means that the learning procedure must either be performed offline \citep{duarte2016evolution}, or be reliant on a centralized sensor (e.g. an overhead camera) in a controlled environment \citep{winfield2019understandableswarmbehaviors}.
  The robots cannot continue learning once deployed without this central sensor, which reduces flexibility to untrained environments and novel situations.
  In this work, we thus take a different approach: a feed-forward neural network is trained to estimate the global performance of the swarm from the local states of its robots.
  This allows us to extract the local objectives of the individual robots, such that, if the robots achieve these local objectives, the global performance will benefit.
  This is a transparent step that results in the explicit representation of the local goals of robots.
  With this, the robot's policy can then be analyzed in relation to its fulfillment of local objectives.
  The model can be trained offline and the desired local states that are extracted from it can then even be used during online learning, which liberates the need to centrally assess the performance of the swarm during operation.
  To the best of our knowledge, this is the first work where a function is trained to explicitly relate microscopic (i.e., local) states with macroscopic (i.e., global) performance, and additionally use this to optimize the behavior of the swarm, also online.

  The challenge of developing a function that synthesizes a global performance metric into measurements that are observable by onboard sensors is not unique to the swarming regime.
  It applies to any situation whereby the objective can only be partially measured by an agent.
  We briefly discuss some related literature on this topic from outside of the swarming domain, but that aim to solve this challenge for Partially Observable Markov Decision Problems (POMDPs).
  Recently, \cite{faust2019autorl} and \cite{chiang2019autorl} proposed AutoRL.
  The proposed solution is to complement the reinforcement learning process with an evolutionary procedure that optimizes the parameters of an explicit reward function.
  For swarming, however, this additional learning layer could be computationally problematic and potentially subject to over-fitting.
  Other solutions to similar problems have been proposed.
  In the works by \cite{florensa2017reverseRL} and \cite{ivanovic2019backwardRL}, a robot learns to reach a final goal by learning backward from start states that are progressively far away.
  This strategy is also not easily suitable to swarming, which is expected to feature a prohibitive set of starting states that scales with the size of the swarm, whereas we aim for a scalable solution.
  In addition, swarm robotics often does not deal with a specific final state, but rather with a general performance metric.
  \cite{andrewng2000irl} proposed Inverse RL (IRL) to learn a reward function from a known optimal policy.
  \cite{sosic2017inverserlswarms} also applied IRL to the swarming domain to extract local controllers.
  However, these works assume a policy to be known, which is most often not the case.
  In a similar vein, demonstrations can be used to represent the actions of an optimal policy. 
  \cite{li2016turing} used demonstrations in a competitive evolutionary setup to learn to replicate the behavior of a swarm.
  However, this practice is not applicable when the optimal policy is not known, which is most often the case.
  Recently, \cite{brown2019extrapolating} ranked sub-optimal demonstrations in order to extract the underlying rules, leading to results that exceed the demonstrations.
  At a conceptual level, we will apply a similar strategy, as we extract local objectives from random swarm behaviors and optimizing the policy accordingly.

\section{Framework description}
\label{sec:ch5_method}

  The framework proposed in this paper is depicted in the diagram in \figref{fig:overview}.
  It is based on three core steps:
  \begin{enumerate*}
    \item Learn the micro-macro relationship (Model 1) and use it to extract a set of desired local states, denoted $\mathcal{S}_{des}$;
    \item Extract the explicit local transition model (Model 2);
    \item Use the above to optimize and verify a policy.
  \end{enumerate*}
  The three steps are detailed in \secref{sec:model1}, \secref{sec:model2}, and \secref{sec:optimization}, respectively.
  Model 1, the ``micro-macro'' model, is a function that relates the local states of the robots in the swarm to the global performance.
  In this work, the function is modeled with a neural network.
  Training this function requires centralized data of the performance of the swarm during sample runs.
  However, once trained, the model can estimate the global performance of the swarm from the local states of the robots.
  Model 2 models the local state transitions experienced by a robot in the swarm.
  It can be learned offline or it can be directly estimated/tuned onboard by the robots during a task, enabling them to adapt their policy accordingly.
  Model 2 is a microscopic probabilistic transition model, which allows us to analyze it.
  To this end, \secref{sec:verification} introduces verification conditions.

  \begin{figure}[t]
    \centering
    \includegraphics[width=\textwidth]{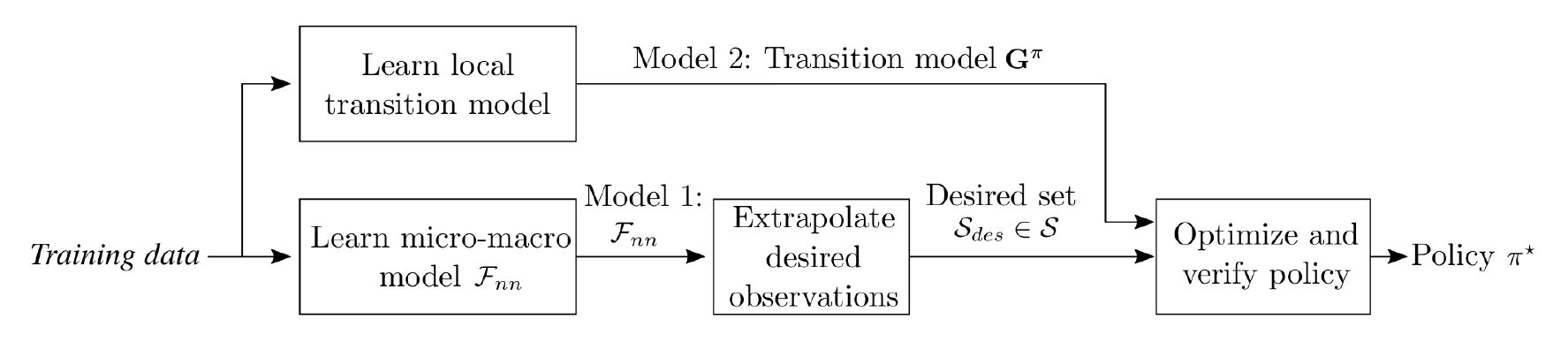}
    \caption{Overview of framework proposed in this paper.}
    \label{fig:overview}
  \end{figure}

\subsection{Model 1: Micro-macro neural network model}
\label{sec:model1}
  For an arbitrary swarming task, let $F_g(t)$ be a global performance metric for the performance of the swarm at time $t\geq0$, which needs to be maximized within a time $t \leq T$.
  Let $\mathcal{S} = \{s_1, s_2, \ldots , s_N\}$ be a discrete set of $N$ local states that a robot in the swarm can observe using its onboard sensors.
  At a given time $t$, a robot in the swarm will be in a local state $s\in\mathcal{S}$.
  The robot cannot measure $F_g(t)$, since this requires knowledge that is not measurable by its onboard sensors. 
  For example, a swarm may be tasked with self-organizing into a pattern while the robots can only sense their nearby neighbors, and not the whole pattern.
  Therefore, instead of guiding the robots to optimize $F_g(t)$, we will instead guide them to be in local states such that $F_g(t)$ is maximized.
  The set of local states that maximize $F_g(t)$ are grouped in the set $\mathcal{S}_{des}\in\mathcal{S}$, which is referred to as the set of desired local states.\footnote{Since we are the first to propose this kind of approach, we make a number of simplifications, such as discrete local states and action spaces and binary desirability of states.
  Generalization beyond such simplifications is discussed in \secref{sec:ch5_discussion}.}

  To automatically extract the set $\mathcal{S}_{des}$ for an arbitrary swarm, we need to model the relationship between the local states of the robots and the global performance.
  This can be a complex nonlinear relationship, and so a neural network is used to automatically learn this micro-macro relationship from training data, limiting the a priori knowledge required.
  The neural network is represented by the function $\mathcal{F}_{nn}$ in \eqnref{eq:fg}.
  \begin{equation}
    \hat{F_g}(t) = \mathcal{F}_{nn}(\mathbf{P_s}(t))
    \label{eq:fg}
  \end{equation}
  In \eqnref{eq:fg}, $\hat{F_g}(t)$ is the estimated $F_g(t)$, and $\mathbf{P_s}(t)$ is a vector of length $N$ indicating the distribution of local states in a swarm at time $t$.
  Based on a training data set, the network $\mathcal{F}_{nn}$ will be trained to estimate $F_g(t)$ for unseen scenarios.
  $\mathcal{F}_{nn}$ has $N$ input neurons and has one output.

  Once the micro-macro model is trained, it can be used to extrapolate the set of desired local states $\mathcal{S}_{des}$ that are expected to maximize $F_g$.
  We use an optimizer to find an input vector $\mathbf{B}$ such that $\hat{F_g}(t)$ is maximized.
  The input vector is a binary vector $\mathbf{B}$ of size $N$.
  This allows us to extract the local states that aid to maximize $\hat{F_g}$, without bias for any single local state that may have been more prevalent according to the data.
  These local states will compose the set $\mathcal{S}_{des}$.
  In our particular implementation for this paper, the optimization was performed using an evolutionary optimizer in order to avoid local maxima.
  Specifically, we used the evolutionary toolbox provided by the Distributed Evolutionary Algorithms in Python (DEAP) package \citep{DEAP_JMLR2012}.

  As an example, consider a swarm of robots that needs to reach consensus.
  $F_g(t)$ measures the highest fraction of robots in the swarm that is in agreement at time $t$.
  The robots, however, can only observe whether they are in agreement with their direct neighbors or not, i.e. $\mathcal{S} = \{s_{agree}, \lnot s_{agree}\}$.
  $\mathcal{F}_{nn}$ would indicate that being in agreement with your direct neighborhood is correlated with a high $\hat{F_g}(t)$.
  Therefore, we can then find that an input $\mathbf{B}=\{1,0\}$ maximizes $\hat{F_g}(t)$, meaning that $\mathcal{S}_{des} = \{s_{agree}\}$.

  \begin{remark}
    In the currently proposed form, $\mathcal{F}_{nn}$ is a feed-forward network and it does not use information from previous time steps for its estimate.
    This means that its estimate is independent of the actions of the robots, their dynamics, past information, and the number of robots in the swarm (albeit a bias may still exist as a result of the training data).
    Although this may create a disadvantage for the case where the performance metric has a temporal aspect, there are also advantages to it for certain contexts.
    Namely, the results of this procedure can be transferred across different environments and across different swarms, provided that the local states (i.e., observations using the onboard sensors) and the performance metric remain unchanged.
    This will be empirically evaluated in \secref{sec:ch5_results}.
    Its further potential is discussed in \secref{sec:ch5_discussion}.
  \end{remark}

\subsection{Model 2: Transition model}
\label{sec:model2}
  Following the procedure in \secref{sec:model1}, a set of desired local states correlated with a high global performance is automatically extracted.
  This reduces the problem to a local variant, wherein the policy of a robot needs to be such that its probability of transitioning to a local state $s\in\mathcal{S}_{des}$ is maximized.
  
  A robot in the swarm can perform actions from an action set $\mathcal{A} = \{a_1, a_2,\ldots, a_M \}$ of size $M$.
  Let $\pi$ be a tabular probabilistic policy of size $N \times M$ that, for a given local state $s\in\mathcal{S}$, indicates the probability of taking an action $a\in\mathcal{A}$.
  The transition model can be used to understand, quantify, and potentially verify the policy and its impact on the swarm.
  As in \citep{coppola2019pagerank}, an optimum policy can be found efficiently using PageRank centrality as a fitness measure.
  The main advantage of this approach is that the policy optimization is performed without simulation.
  The layout of the model is inspired by the random surfer model by \cite{brin1998anatomy}, which can be summarized by the ``Google'' matrix $\mathbf{G_\pi}$:

  \begin{equation}
    \mathbf{G_\pi} = \mathbf{\alpha} \mathbf{H_\pi} + (\mathbf{I}-\mathbf{\alpha}) \mathbf{E}.
    \label{eq:G}
  \end{equation}
  
  $\mathbf{G_\pi}$ is a cumulative model composed of two parts:

  \begin{itemize}  
    \item $\mathbf{H_\pi}$ is a stochastic matrix of size $N \times N$ holding the probabilities of local state transitions that result from the actions taken by an arbitrary robot in the swarm, following policy $\pi$.

    \item $\mathbf{E}$ is a stochastic matrix of size $N \times N$ holding the probabilities of local state transitions that occur when the robot has not taken an action.
    These transitions are thus attributed to the environment (which includes the other robots in the swarm).
  \end{itemize}
  
  The diagonal matrix $\alpha$ balances the probability of local state transitions in $\mathbf{H_\pi}$ and $\mathbf{E}$ for each row.
  In other words, each diagonal entry in $\alpha$, denoted $\alpha(s_i)$ is the ratio of likelihood of occurrence between $\mathbf{H_\pi}$ and $\mathbf{E}$ when in a local state $s_i$.
  Cumulatively, $\mathbf{G_\pi}$ models all transitions that can occur to a robot in the swarm, either as a result of its policy or as a result of its environment.

  To model the expected transition due to a specific action, $\mathbf{H_\pi}$ can be decomposed into a series of sub-models, each showing the impact of one action from the action set $\mathcal{A}$.
  Let $\mathbf{A}$ be a set of $M$ sparse stochastic matrices of size $N \times N$, such that $\mathbf{A} = \{ \mathbf{A_1}, \mathbf{A_2},\ldots,\mathbf{A_M}\}$.
  The elements of a matrix $\mathbf{A}_k$ hold the probability of a local state transition from $s_i\in\mathcal{S}$ to $s_j\in\mathcal{S}$ given an action $a_k \in \mathcal{A}$, for $k = 1,\ldots,M$.
  Specifically:
  \begin{equation}
    \mathbf{A_k}(s_{i},s_{j}) = P(s_{i,j}|a_k),
  \end{equation}
  where $s_{i,j}$ indicates a transition event from $s_i\in\mathcal{S}$ to $s_j\in\mathcal{S}$.
  It follows that $\mathbf{A}$ can be combined into a cumulative model of all transitions as a result of a policy $\pi$, which is the matrix $\mathbf{H}_\pi$.
  An element in $\mathbf{H}_\pi$ is given by:
  \begin{align}
    \mathbf{H_\pi}(s_{i},s_{j})
    &= P((s_{i,j} \cap a_1) \cup (s_{i,j} \cap a_2) \cup \ldots \cup ({s_{i,j} \cap a_M})) = \sum_{k=1}^{M} P(s_{i,j} \cap a_k)_\pi
  \end{align}
  where:
  \begin{equation}
    P(s_{i,j}\cap a_k)_\pi 
    = P(s_{i,j}|a_k) * \pi(s_i,a_k)
    = \mathbf{A_k}(s_{i},s_{j}) * \pi(s_i,a_k)
    \label{eq:newmodel}
  \end{equation}
  In the above, $\pi(s_i,a_k)$ is the probability of taking $a_k$ under local state $s_i$.
  This format separates the transition of an action into separate models, which serves to predict the effect of a different policy.
  The transition models can be numerically estimated from experience data as a maximum likelihood model \citep{sutton1990integrated}.
  Any time a transition happens when taking an action $a_k$, its transition is assigned to the corresponding model $\mathbf{A_k}$.
  Any time a transition happens when no action is taken, its transition is assigned to $\mathbf{E}$.
  The diagonal elements of the diagonal matrix $\mathbf{\alpha}$ are  estimated by the ratio of active events to environment events.

  \begin{remark}
    One limitation of the transition model used in this paper is that it requires the discretization of both the local state space and action space into the discrete sets $\mathcal{S}$ and $\mathcal{A}$.
    This is done to enable the model-based optimization and analysis explained in \secref{sec:optimization} and \secref{sec:verification}.
    However, the use of alternative models could be explored in future work. 
    This is further discussed in \secref{sec:ch5_discussion}.
  \end{remark}

\subsection{Model-based policy optimization}
\label{sec:optimization}
  An optimum policy $\pi^\star$ is found by maximizing the probability that a robot in the swarm transitions to a local state $s\in\mathcal{S}_{des}$.
  This is evaluated by the metric $F_{pr}$,
  \begin{equation} 
    F_{pr}(\pi) = 
    \frac{\sum_{\forall s \in \mathcal{S}_{des}} PR_\pi(s)/\abs{S_{des}}}
    {\sum_{\forall s \in \mathcal{S}} PR_\pi(s)/\abs{S}},
    \label{eq:pagerank}
  \end{equation}
  where $PR_\pi(s)$ indicates the PageRank centrality of local state $s$ given a model $\mathbf{G_\pi}$.
  PageRank centrality is a measure for the probability of transitioning to a given node in a graph that represents the transition model $\mathbf{G_\pi}$.
  Conceptually, \eqnref{eq:pagerank} serves to maximize the average PageRank of the local states in the set $\mathcal{S}_{des}$ over the set of all local states $\mathcal{S}$.
  This optimizes the policy based on the transition model efficiently and without 1) a more computationally expensive procedure based on simulated experience, or 2) trial and error.

\subsection{Model-based analysis and verification}
\label{sec:verification}
  This section explains how the transition model and the desired local states can be used in order to check properties of the swarm's behavior.
  A set of task-agnostic conditions are presented to determine the swarm's ability to achieve its goal, whereby the goal is defined as a state where all robots achieved a desired local state.

  Consider a swarm of $m$ robots.
  Let $\mathcal{P}$ be the set of all global states that the swarm can be in.
  Let $\mathcal{P}_{des} \in \mathcal{P}$ be a set of all global states for which all robots in the swarm have a local state $s\in\mathcal{S}_{des}$, and let $p(t)$ be the global state of the swarm at time $t$.
  It is assumed here that $\mathcal{P}_{des} \neq \emptyset$.
  The swarm is defined as successful as per Definition \ref{def:success}, which is indicative of a good performance by the swarm as obtained through the procedure based on Model 1 from \secref{sec:model1}.
  \begin{definition}[Successful]
    \label{def:success}
    A swarm with a global state $p(t)$ is \textbf{successful} if $p(t)\in\mathcal{P}_{des}$.
  \end{definition}
  We wish to determine whether the swarm will eventually be successful when starting from any arbitrary initial global state $p(t=0)$.
  One way to determine this would be to analyze all possible global states of the swarm and transitions between them.
  However, such an analysis can become intractable as the size of the swarm, and its global state space, increases \citep{dixon2012towards}.
  We thus adopt a different strategy that uses the local transition model (Model 2) to identify two types of potential counter-examples: deadlocks and livelocks.
  In this paper, we define them in Definition \ref{def:deadlock} and \ref{def:livelock}, respectively.
  If no livelocks and deadlocks exist, then the swarm can keep transitioning and potentially achieve all global states, thus including the successful states contained in $\mathcal{P}_{des}$.
  Note, however, that the local checks shown in this paper are not guaranteed, for any arbitrary task, to identify all possible livelocks and deadlocks.
  For a given task, however, more detailed and ad hoc checks can be constructed, as was for instance done for the pattern formation task presented in our previous work \citep{coppola2019provable}.

  \begin{definition}[Deadlock]
  \label{def:deadlock}
    A \textbf{deadlock} is a state $p(t)\not\in\mathcal{P}_{des}$ whereby the swarm does not transition to a different state, i.e., $p(t_2) = p(t_1)\:\forall\:t_2 > t_1$.
  \end{definition}

  \begin{definition}[Livelock]
  \label{def:livelock}
    Consider an arbitrary subset of global states 
    $\mathcal{P}_{lock}\in\mathcal{P}$, where
    $\mathcal{P}_{lock} \cap \mathcal{P}_{des} = \emptyset$ and $1 < \abs{\mathcal{P}_{lock}} < \abs{\mathcal{P}}$.
    At time $t>t_0$ a swarm is in a \textbf{livelock} if it can only transition between states that are within the set $\mathcal{P}_{lock}$, i.e., $p(t>t_0) \in \mathcal{P}_{lock}$.
  \end{definition}

  \subsubsection{Identifying deadlocks from local states}
  \label{sec:verification_deadlocks}
    Deadlocks indicate that the swarm has stopped from changing global state. Potential deadlocks can be identified with the help of the transition model (Model 2).
    Let $\mathcal{S}_{static}$ be a set of local states for which a robot does not transition to any new local state via its policy.
    $\mathcal{S}_{static}$ can be determined from the empty rows of the model $\mathbf{H_\pi}$, indicating that no state transitions are possible.
    \begin{equation}
      \mathcal{S}_{static} = 
      \{ 
      s\in\mathcal{S} : \sum_j^N \mathbf{H_\pi}(s,s_j) = 0 
      \}
      \label{eq:ostatic}
    \end{equation}
    The following three outcomes are possible:
    \begin{itemize*}
      \item If $\mathcal{S}_{static} = \emptyset$, then a deadlock is not present (according to the model), because all robots can move and have a nonzero probability of transition.
      \item If $\mathcal{S}_{static} \cup \mathcal{S}_{des} = \mathcal{S}_{des}$, then the swarm may cease to transition.
      However, this is only for global states $p(t)\in\mathcal{P}_{des}$, which does therefore not constitute a deadlock under Definition \ref{def:deadlock}.
      \item If $\mathcal{S}_{static} \cup \mathcal{S}_{des} \neq \mathcal{S}_{des}$, then a deadlock may be present.
      This is caused by the local states in $\mathcal{S}_{static}$ that are not part of $\mathcal{S}_{des}$.
    \end{itemize*}

    To determine the deadlocks that may occur based on the above, it is necessary to analyze the local states in $\mathcal{S}_{static}$ and whether (and, if desired, how) they can coexist at the global level.
    This is a global level check that however only requires analysis of a subset of global states.
    In addition, the check does not necessarily scale with the size of the swarm.
    For any task where the swarm is not expected to exist in one connected cluster, then local connected deadlocks may be identified for small clusters and it is not necessary to examine with the swarm size $m$.
    During policy optimization, if a deadlock is caused by a new policy, causing $\mathcal{S}_{static} \cup \mathcal{S}_{des} \neq \mathcal{S}_{des}$, then the policy of the robots may be adapted to resolve this.

  \subsubsection{Locally identifying livelocks}
  \label{sec:verification_livelocks}
    In a livelock, the swarm transitions between global states that do not include or lead to success.
    In this section, we propose checks to identify the presence of livelocks using the local transition model.
    Note that the checks proposed in this section are general and thus not guaranteed to detect \emph{all} possible livelocks for an arbitrary task.
    However, more detailed and ad hoc checks can be constructed with this type of model, as was done for the pattern formation task presented in our previous work \citep{coppola2019provable}, and in future work we encourage the development of further checks.

    Let $\mathcal{S}_t$ be the set of the local states of the robots in the swarm at an arbitrary time $t$.
    By the condition in Proposition \ref{prop:1} it can be established whether there is any set of initial local states $\mathcal{S}_{t=0}$ that can potentially prevent the swarm from achieving a global state $p\in\mathcal{P}_{des}$.
    This can stem from the fact that local states exist which endlessly cycle between states $s \not \in S_{des}$, thus preventing the swarm from achieving a successful global state as defined in Definition \ref{def:success}.
    In the following, let $G_\mathcal{S}^{H_\pi}$ be a weighted graph made from the transition model $\mathbf{H_\pi}$, and let $G_\mathcal{S}^{E}$ be a weighted graph made from the transition model $\mathbf{E}$.

    \begin{proposition}
      \label{prop:1}
      If in $G_\mathcal{S}^{H_\pi}$ there is a path from all local states $s\not\in\mathcal{S}_{des}$ to all local states $s\in\mathcal{S}_{des}$, then a swarm will be successful independently of $\mathcal{S}_{t=0}$.
    \end{proposition}
    \begin{proof}
      The model $G_\mathcal{S}^{H_\pi}$ only considers the local state transitions that the robots make themselves.
      Consider a robot $i$ in the swarm with an arbitrary local state $s_i\in\mathcal{S}$.
      If the condition holds, then, by following the policy $\pi$, the robot has a probability $\rho > 0$ of transitioning to any and all local states in the set $\mathcal{S}_{des}$.
      Applying this to all robots in the swarm, it follows that all robots will have a nonzero probability of eventually transitioning to a local state $s\in\mathcal{S}_{des}$, according to the model.
      Therefore, a desired global state $p\in\mathcal{P}_{des}$ can be achieved independently of $\mathcal{S}_{t=0}$.
      As robots can transition to all local states in $\mathcal{S}_{des}$, they are not bounded to any specific desired local state.
    \end{proof}

    If the condition in Proposition \ref{prop:1} is true, then, if an initial global state $p(0)$ exists for which the swarm cannot be successful, then this is not due to any local state in $\mathcal{S}_{t=0}$.
    A case where the condition in Proposition \ref{prop:1} fails is if static local states exist which are not desired, i.e., $\mathcal{S}_{static} \cup \mathcal{S}_{des} \neq \mathcal{S}_{des}$, where $\mathcal{S}_{static}$ is defined as per \eqnref{eq:ostatic}.
    In this situation, a robot with a local state $s\in\mathcal{S}_{static}-\mathcal{S}_{des}$ may not transition to a local state $s\in\mathcal{S}_{des}$.
    We can then check whether the environment is such that the robot could transition to non-static local states, and subsequently transition to a desired local state.
    This is addressed by Proposition \ref{prop:2}.
    Note that in the following we assume that all diagonal entries of $\alpha$ are between 0 and 1, i.e., $0.0 < \alpha(s) < 1.0\: \forall \: s\in\mathcal{S}$.
    
    \begin{proposition}
    \label{prop:2}
    If the following two conditions apply:
    \begin{enumerate*}
      \item
      In $G_\mathcal{S}^E$, all local states
      $s\in\mathcal{S}_{static}-\mathcal{S}_{des}$
      have a direct transition to any local state 
      $s\not\in\mathcal{S}_{static}$.
      \item 
      In $G_\mathcal{S}^{H_\pi}$,
      there is a path 
      from all local states $s\not\in\mathcal{S}_{static}$ 
      to all local states $s\in\mathcal{S}$.
    \end{enumerate*}
    then a global state $\mathcal{P}_{des}$ can be achieved independently from $\mathcal{S}_{t=0}$.
    \end{proposition}
    \begin{proof}
      Consider a robot with a local state $s$ where $s\in\mathcal{S}_{static}-\mathcal{S}_{des}$.
      If the first condition is fulfilled, 
      then the robot has a probability $\rho_1$ of transitioning to a local state $s\not\in\mathcal{S}_{static}$, where $\rho_1>0$.
      If now the robot has a local state $s\in\mathcal{S}_{des}$, then we are done.
      If, instead, the robot has a local state $s\not\in\mathcal{S}_{static}\cup\mathcal{S}_{des}$, then, if $G_\mathcal{S}^{H_\pi}$ fulfills the second condition, the robot has a probability $\rho_2 > 0$ of transitioning to any local state.
      This includes any local state in the set ${S}_{des}$.
      If the robot transitions back to a local state $s\in\mathcal{S}_{static}$, then the process is reinitiated.
    \end{proof}

    \begin{remark}
      The policy optimization procedure can alter the policy $\pi$ in such a way as to alter the outcomes of the conditions.
      During optimization, this can be prevented by setting a constraint 
      $\pi(s,a)>\varepsilon>0\;\forall\;s\in\mathcal{S}, a\in\mathcal{A}$, where $\varepsilon$ is a minimum bound on the probability.
    \end{remark}

\section{Performance analysis}
\label{sec:ch5_results}

In this section, the framework is evaluated through four case studies, introduced in \secref{sec:studycases}, featuring aggregation and foraging scenarios.
The case studies were designed to test different aspects of the framework, in its current implementation, in tasks of increasing and/or different complexity, with a global performance metric that is not measurable by the individual robots.
The framework will first be analyzed in a standalone fashion.
For this case, training data will be generated from the execution of random policies, which will be used to train the models and optimize the policy.
\secref{sec:simulation} details the simulation environment and the data gathering methodology.
The results are evaluated in \secref{sec:results_models} and \secref{sec:results_evaluation}.
The swarms are analyzed, using the models, in \secref{sec:results_verification}.
Following the standalone analysis, \secref{sec:evo} and \secref{sec:online}, explore its use for hybrid model-based evolutionary learning and model-based online learning, respectively.

\subsection{Description of case studies}
\label{sec:studycases}

\subsubsection*{Case study A: Aggregation task with non-directional neighbor sensors}
  In this task, a swarm of $n$ freely moving robots in an arena should aggregate as much as possible within a maximum time $T$.
  The robots behave like accelerated particles and use repulsion to avoid collisions with nearby neighbors. The robots also reflect off walls.

  \begin{itemize}
    \item \emph{Goal:}
      The global fitness function to maximize is
      \begin{equation}
        F_g^A(t) = \frac{n}{c(t)},
        \label{eq:fa}
      \end{equation}
      where $c(t)$ is the number of connected clusters at time $0\leq t \leq T$, and $n$ is the total number of robots in the swarm. 
      A connected cluster is defined as a group of robots that forms one group with a connected graph topology, whereby a single robot by itself also counts as one cluster.
      Note that $F_g^A(t)$ is subject to $1 \leq F_g^A(t) \leq n$.
      This creates a lower bound for a poorly performing swarm and a higher bound if a larger swarm aggregates, which is representative of a greater difficulty.
      A top view of the task is shown in \figref{fig:picA_full}.

    \item \emph{Onboard sensors:}
      Each robot can sense how many neighbors are within a sensing range $r_{\max}$ from itself, for a maximum of $m_{\max}$ neighbors.
      The number of neighbors is denoted $m$, where $m\in\mathbb{Z^+}$ and $0\leq m \leq m_{\max}$.
      Therefore, $\mathcal{S} = \{0,1,\ldots,m_{\max}\}$.
      This set up is representative of an antenna that the robots use to sense each other, with a maximum range $r_{\max}$ and a maximum number of possible connection nodes $m_{\max}$.
      The sensory setup is depicted in \figref{fig:picA}.
      
    \item \emph{Actions:}
      Each robot can choose whether it should a) move randomly, or b) stay still.
      It makes this choice based on policy $\pi$ with frequency $f_c$, or if its current local state changes.
      $\pi$ is a vector of $(m_{\max}+1) \times 1$, where each entry indicates the probability of choosing ``move'', given $m\in\mathbb{Z^+}$ neighbors, where $0\leq m \leq m_{\max}$.
      When ``move'' is selected, the robot moves in a randomly chosen direction.
 \end{itemize}
 
  In the results in this paper, the following parameters were used: 
  $f_c = 0.5~Hz$,
  $m_{\max} = 7$,
  $T = 200~s$.

\subsubsection*{Case study B1: Aggregation task with directional neighbor sensors}
  In this task, a group of robots must form an aggregate as in case study A, yet the specifics of the robots are significantly different, thereby increasing the difficulty of the task.
  Each robot moves as a directed particle within its own frame of reference. According to the action space, it is always commanded to have a given forward speed, and can only dictate its turning rate.
  The robots avoid collisions using repulsion, which is added to the commanded velocity.

  \begin{itemize}
    \item \emph{Goal:} 
      The fitness function is the same as given by \eqnref{eq:fa} for case study A.
      A top view of this task is shown in \figref{fig:picB_full}.

    \item \emph{Onboard sensors:} 
      Each robot can sense the presence of neighbors within sectors, up to a range $r_{\max}$ away.
      This is representative of, for instance, infrared sensors.
      The local state space is $\mathcal{S} = \{ s_1,\ldots,s_n \}$ where a state $s_i$ describes the sectors in which neighbors are sensed.
      It follows that $\abs{\mathcal{S}}=2^q$, where $q$ is the number of sectors.
      The sensory setup is depicted in \figref{fig:picB} for the case where $q=4$, which we used in this work.

    \item \emph{Actions:}
      At all times, a robot moves forward with a forward velocity $v_{cmd}$ along its local body frame.
      It can control its turning rate $\dot{\psi}$ from the action set $\mathcal{A}$, where the following was defined:
      \begin{equation}
        \mathcal{A} = \{ -1.0, -0.7, -0.3, -0.1, 0.1, 0.3, 0.7, 1.0\}. \nonumber
      \end{equation}
      The policy $\pi$ is a stochastic matrix of size
      $\abs{\mathcal{S}}\times\abs{\mathcal{A}}$.
      A robot selects a new action with a frequency $f_n$, or if its local state changes.
      It is not possible for the robots to not select an action from $\mathcal{A}$ (with the exception of wall avoidance within the arena, whereby the robots rotate away from the wall).
      This means that the robots are always prompted to move with a nonzero turning rate, which can create a challenge for aggregation.
      Collision avoidance commands are added to the commanded velocity vector from the action space.
  \end{itemize}

  The following parameters were used in the results:
  $f_n = 0.5~Hz$,
  $q = 4$,
  $v_{cmd} = 0.5~m/s$,
  $T = 200~s$.

\subsubsection*{Case study B2: Aggregation task variant with directional neighbor sensors}
  This task is a modified version of case study B1, intended to further increase the difficulty.
  The goal and onboard sensors are the same as in B1.
  However, the actions that the robots take are different.

  \begin{itemize}
    \item \emph{Goal:}
      The fitness function is the same as given by \eqnref{eq:fa} for case studies A and B1.
      A top view of this task is shown in \figref{fig:picB_full}.

    \item \emph{Onboard sensors:} 
      This is the same as in case study B1, depicted in \figref{fig:picB}.

    \item \emph{Actions:}
      The action space used is the same as B1:
      \begin{equation}
        \mathcal{A} = \{ -1.0, -0.7, -0.3, -0.1, 0.1, 0.3, 0.7, 1.0\}. \nonumber
      \end{equation}
      The definition of the probabilistic policy $\pi$ is also the same. 
      However, the following is different.
      The forward speed is $v_{cmd} = v_{\mathrm{mean}} * \abs{a_k}$, where $a_k$ is the currently selected value from $\mathcal{A}$.
      In addition, $a_k$ is also the turning rate $\dot{\psi}$.
      When a robot selects a turning rate, the corresponding rate is applied for a time $t_1$, after which the robot moves straight until a time $t_2$ with velocity $v_{cmd}$, where $t_1+t_2=1/f_n$, or until its local state changes.
  \end{itemize}
  
  The following parameters were used:
  $f_n = 0.5~Hz$,
  $q = 4$,
  $v_{\mathrm{mean}} = 0.5~m/s$,
  $t_1 = 1~s$,
  $t_2 = 1~s$,
  $T = 200~s$.

\subsubsection*{Case study C: Foraging task}
  
  In this task, a swarm of $n$ robots in an arena is tasked with maximizing the food at the nest at a final time $T$.
  Each robot at the nest consumes $e_n$ food items per second.
  A robot that ventures out to forage for food eats $e_f$ food items every $t_c$ seconds that it spends searching.
  If a food item falls within its sensor range, then the robot collects it and returns to the nest.
  When collected, a new food item appears in a random location, such that the total number of food items in the arena remains constant.
  A robot cannot hold more than one food item at any given time.

  \begin{itemize}

    \item \emph{Goal:}
      The global fitness function is
      \begin{equation}
        F_g^C(t) = f(t) - f(0),
      \end{equation}
      where $f(t)$ is the total number of food items at the nest at time $0 \leq t \leq T$, and $f(0)$ is the food at the nest at the start of the task.
      A top view of the task is shown in \figref{fig:picC_full}.

    \item \emph{Onboard sensors:} 
      A foraging robot can sense food items within a maximum range $r_{\max}$ as shown in \figref{fig:picC}.
      Upon returning to the nest, the robot can sense the difference in food between when it left and when it returned.
      This is saturated by a value $f_{\max}$ and discretized evenly along $\abs{\mathcal{S}}$ steps, i.e, $\mathcal{S} = \{ -f_{\max}, \ldots , f_{\max}\}$.
      A robot cannot measure the global fitness $F_g^C(t)$ but only the difference, i.e., $f_g^C(t) = F_g^C(t_{\mathrm{arrival}}) - F_g^C(t_{\mathrm{departure}})$, where $-f_{\max}\leq f_g^C(t) \leq f_{\max}$.
      When at the nest, the robot measures the difference between its observations every time that it reiterates whether to explore or not, as below.

    \item \emph{Actions:}
      When at the nest, a robot can make the choice to forage or to remain at the nest.
      The policy $\pi$ is a vector of length $\abs{\mathcal{S}} \times 1$, where each entry in $\pi$ indicates the probability of choosing ``explore'' given the current local state.
      When ``explore'' is selected, the robot leaves the nest and does not come back until food is found.
      When at the nest, a robot reassess its choice with a frequency $f_n$.
  
  \end{itemize}

  The following parameters were used:
  $f(0)= 15$,
  $f_n = 0.1~Hz$,
  $T = 500~s$,
  $e_f = 0.1$ ,
  $t_c = 10~s$,
  $e_n = 0.02$, and 
  $\abs{\mathcal{S}} = 30$ with $f_{\max} = 5$.

  \begin{figure}[t]
  \centering
    \begin{subfigure}[t]{0.30\textwidth}
      \centering
      \includegraphics[width=\textwidth]{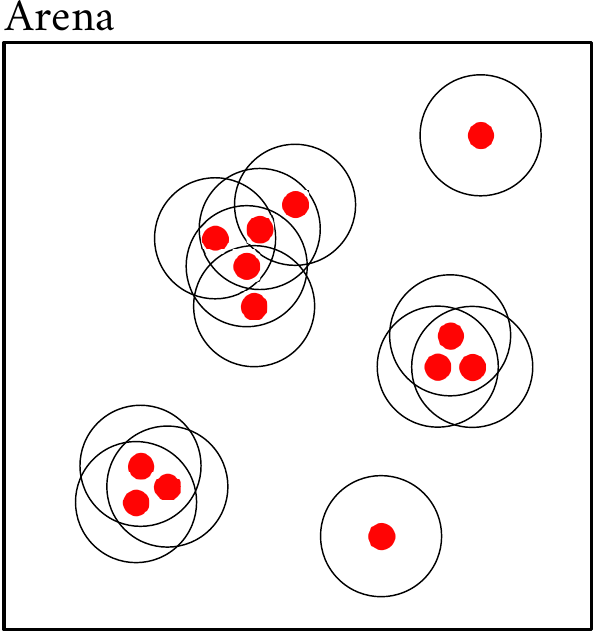}
      \caption{Top view of case study A}
      \label{fig:picA_full}
    \end{subfigure}\hfill
    \begin{subfigure}[t]{0.32\textwidth}
      \centering
      \includegraphics[width=\textwidth]{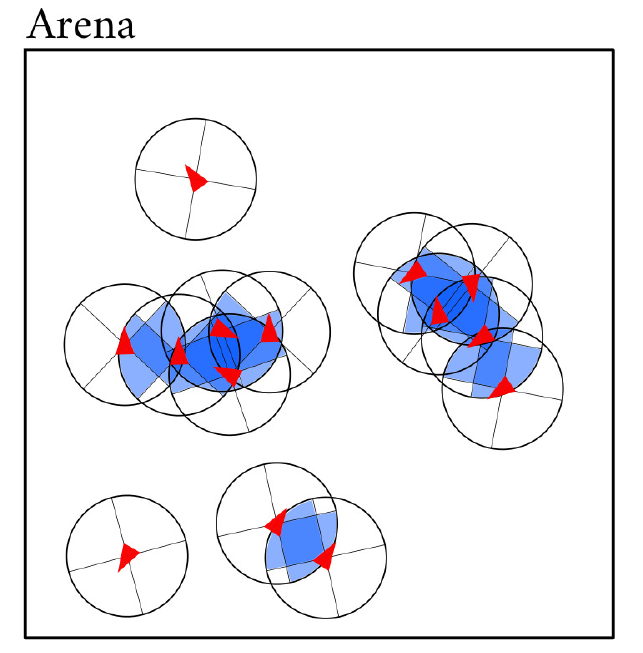}
      \caption{Top view of case studies B1 and B2}
      \label{fig:picB_full}
    \end{subfigure}\hfill
    \begin{subfigure}[t]{0.32\textwidth}
      \centering
      \includegraphics[width=\textwidth]{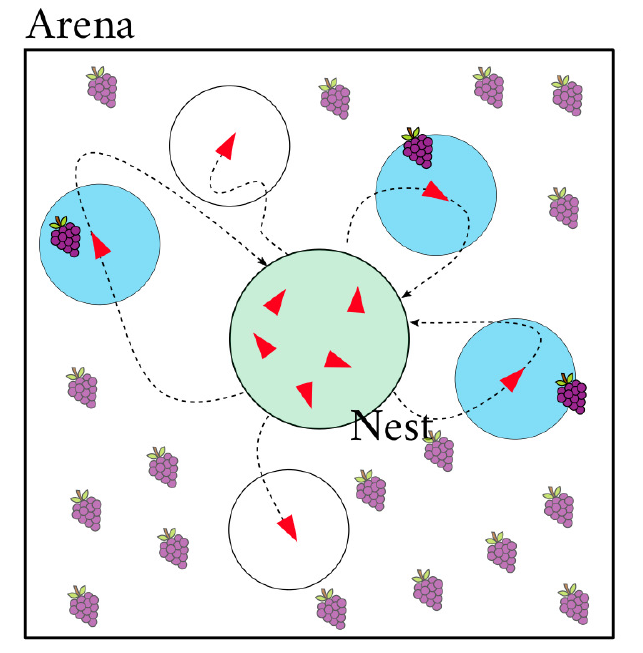}
      \caption{Top view of case study C}
      \label{fig:picC_full}
    \end{subfigure}\\
    \vspace{0.5cm}
    \begin{subfigure}[t]{0.30\textwidth}
      \centering
      \includegraphics[width=0.65\textwidth]{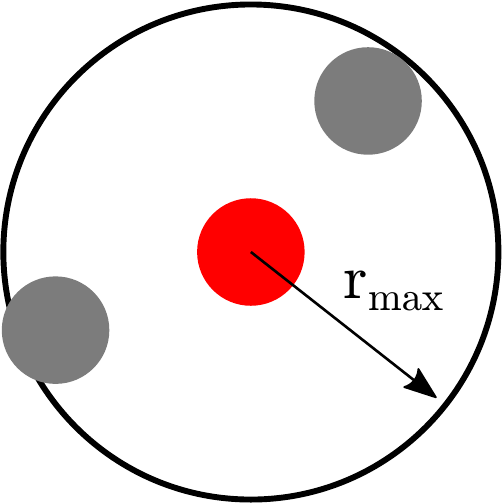}
      \caption{Sensor for case study A}
      \label{fig:picA}
    \end{subfigure}\hfill
    \begin{subfigure}[t]{0.30\textwidth}
      \centering
      \includegraphics[width=0.65\textwidth]{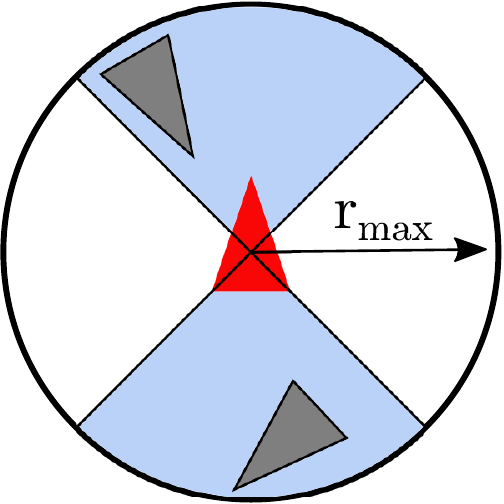}
      \caption{Sensor for case studies B1 and B2}
      \label{fig:picB}
    \end{subfigure}\hfill
    \begin{subfigure}[t]{0.30\textwidth}
      \centering
      \includegraphics[width=0.65\textwidth]{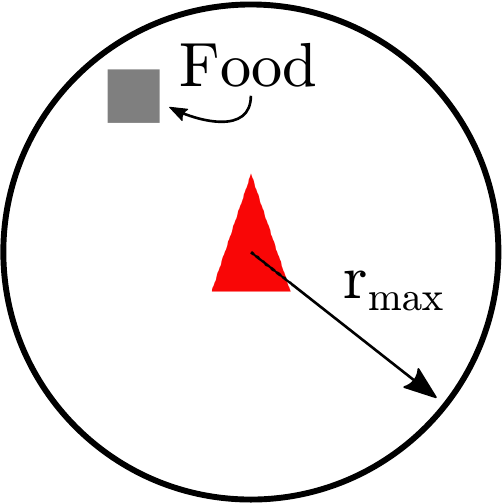}
      \caption{Sensor for case study C}
      \label{fig:picC}
    \end{subfigure}

    \caption{Global views and sensor depictions for each task. Robots are indicated in red.}
    \label{fig:studycases}
  \end{figure}

\subsection{Test environment and data gathering}
\label{sec:simulation}

  \subsubsection{Test environment}
    All four tasks are evaluated using the C++ based simulation environment \emph{Swarmulator}.
    This is the dedicated swarm simulator that was also used in our prior publications \citep{coppola2019pagerank,coppola2019provable}.
    Each robot is simulated on an independent detached thread so that timing artifacts are avoided and automatically randomized.
    A link to the relevant code is provided at the end of this paper.

  \subsubsection{Training data}
    The training data was gathered from 500 independent simulations, simulated for time $T$ in arenas of size $20 \times 20$ meters.
    In each simulation run, a swarm of $n$ robots follows a randomly generated policy $\pi$ based on the policy spaces introduced in \secref{sec:studycases}.
    Random swarm sizes were used with $1 \leq n \leq 30$, chosen from a discrete uniform distribution.
    For case study C, this was changed to $1 \leq n \leq 20$ to limit congestions at the nest.
    The global fitness $F_g(t)$ and the local states $s\in\mathcal{S}$ of each robot is logged with a frequency of $2~Hz$ with respect to the simulated time clock.
    This makes for $2T$ data points per simulation, and a training set size of $\approx1000T$ to train Model 1 (the micro-macro model).
    In addition, all local state transitions by the robots during these simulations were logged and used to generate Model 2 (the transition model).

\subsection{Implementation and results of model training}
\label{sec:results_models}

\subsubsection{Model 1: implementation, training, and analysis}
  \begin{table}
    \center
    \caption{Summary table with properties for each case study.}
    \label{tab:nndetails}
    \begin{tabular}{| c | c | c | c | c | c | }
      \hline
      \textbf{Case study} & $\abs{\mathbf{P}_s}$ & Hidden layers & Layer size & Learning Rate & T(s) \\ \hline
      A  & 8  & 3 & 30 & 1e-5 & 200\\
      B1 & 16 & 3 & 30 & 1e-5 & 200\\
      B2 & 16 & 3 & 30 & 1e-5 & 200\\
      C  & 30 & 3 & 100 & 1e-6 & 500\\\hline
    \end{tabular}
  \end{table}
  \begin{figure}[t]
    \centering
    \begin{subfigure}[t]{0.49\textwidth}
      \centering
      \includegraphics[width=\textwidth]{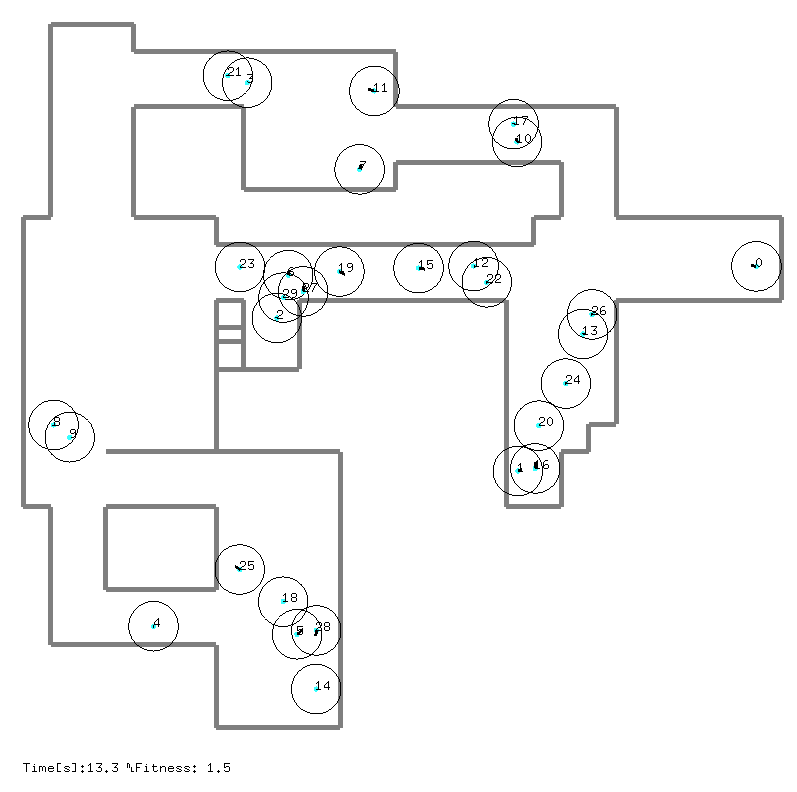}
    \end{subfigure}
    \hfill
    \begin{subfigure}[t]{0.49\textwidth}
      \centering
      \includegraphics[width=\textwidth]{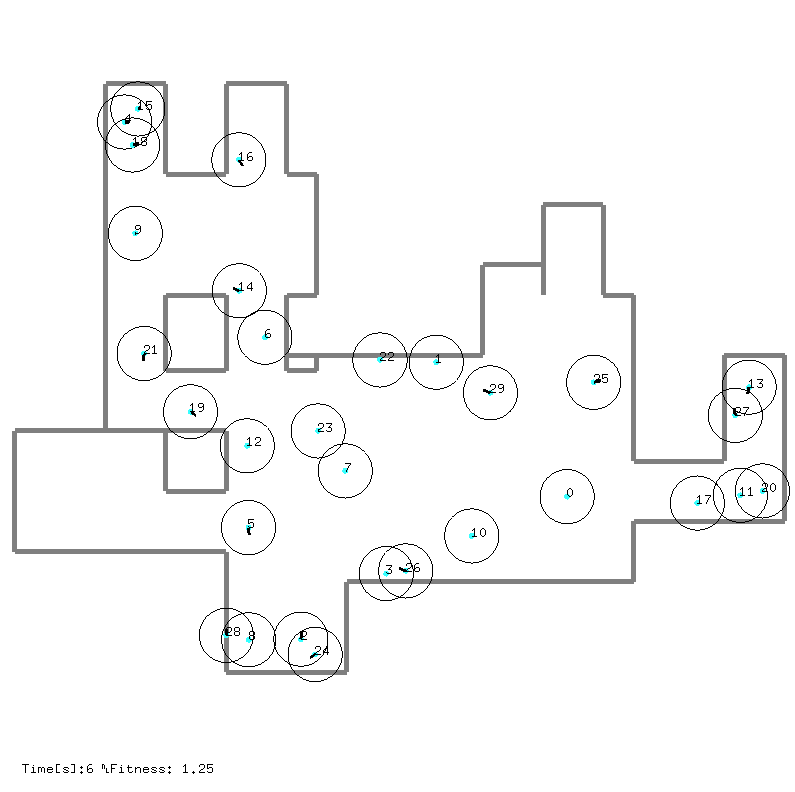}
    \end{subfigure}
    \caption{Top view of two randomly generated arenas, as for VS 3, with 30 robots from case study A. These screenshots are taken directly from the simulator.}
    \label{fig:VS3arenas}
  \end{figure}
  For each task, a neural network $\mathcal{F}_{nn}$ is constructed with Rectified Linear Unit (ReLU) layers.
  All networks were trained using the Adam optimizer as implemented in PyTorch.\footnote{See \cite{pytorch} for more information.}
  The specifics of each network are summarized in \tabref{tab:nndetails}.
  Each neural network was trained using the training data generated as per \secref{sec:simulation}.    
  Then, three validation sets were made for each task.
  \begin{itemize}
    \item \textbf{Validation Set 1 (VS 1)}, generated by 100 runs using the same setup as for the training data set. 
    As for the training set, each run featured a randomly generated policy and a random number of robots.
    \item \textbf{Validation Set 2 (VS 2)}, generated by 100 runs using the same setup as for the training data set, but with a smaller arena size of $10\times 10$ meters. 
    As for the training set, each run featured a randomly generated policy and a random number of robots.
    \item \textbf{Validation Set 3 (VS 3)}, generated by 100 runs using the same setup as for the training data set, but with randomly generated arenas featuring multiple rooms and corridors.
    Each run featured a randomly generated policy, a random number of robots, and a random arena.

    Two examples of randomly generated arenas are shown in \figref{fig:VS3arenas}.
  \end{itemize}
  \renewcommand{\labelenumi}{\theenumi.}
  The intent behind VS 2 and VS 3 is to determine how well a network can generalize to different environments.
  The performance of the networks is validated by assessing the mean correlation between the true $F_g$ and its estimate $\hat{F}_g$ across a validation set.
  The correlation measures how valid $\mathcal{F}_{nn}$ is as a tool to extract the set of desired local states $\mathcal{S}_{des}$.
  A positive correlation indicates that the network has learned to predict when the global fitness increases and/or decreases.
  A correlation of 1.0 implies that the trend in $F_g$ is perfectly estimated.

  The results of this analysis are shown in \figref{fig:nn}.
  In all cases, a positive correlation is achieved for the validation sets.
  For case study A, B1, and B2 the correlation is almost maximized (Figures \ref{fig:nnA}, \ref{fig:nnB1}, and \ref{fig:nnB2}, respectively).
  The sample runs in Figures \ref{fig:sampleA}, \ref{fig:sampleB1}, and \ref{fig:sampleB2} show that the global fitness is accurately predicted.
  The correlation is lower for case study C, shown in \figref{fig:nnC}.
  This is expected because of the temporal component of this task, which is not modeled by the neural network.
  Nevertheless, also for case study C, the general trend was still captured.
  This is exemplified by analyzing the prediction for the individual case, as shown in \figref{fig:sampleC}.
  We remark that it is likely that these results could be further improved with more dedicated hyperparameter tuning, which we encourage in future work.
  Two conclusions can be drawn from these results.
  
  \begin{enumerate}
  
    \item The networks can also generalize to validation sets generated with different arenas.
    This means that the trained network can be exported outside of the direct environment for which it has been trained, and that by adding more diverse training data we could also be able to generalize the predictions further.
    This is especially valid for the aggregation tasks, for which the fitness function does not feature temporal effects.

    \item Even though the behavior of the robots in tasks B1 and B2 is different, the trained networks are almost interchangeably applicable.
    This can be appreciated in Figures \ref{fig:nnB1} and \ref{fig:nnB2} (green and black lines).
    This is because both tasks feature the same global fitness and the same local sensors, showing that the performance of $\mathcal{F}_{nn}$ is not dependent on the dynamics.
  
  \end{enumerate}

  \begin{figure}[t]
    \centering
    \begin{subfigure}[t]{0.49\textwidth}
      \centering
      \includegraphics[width=\textwidth]{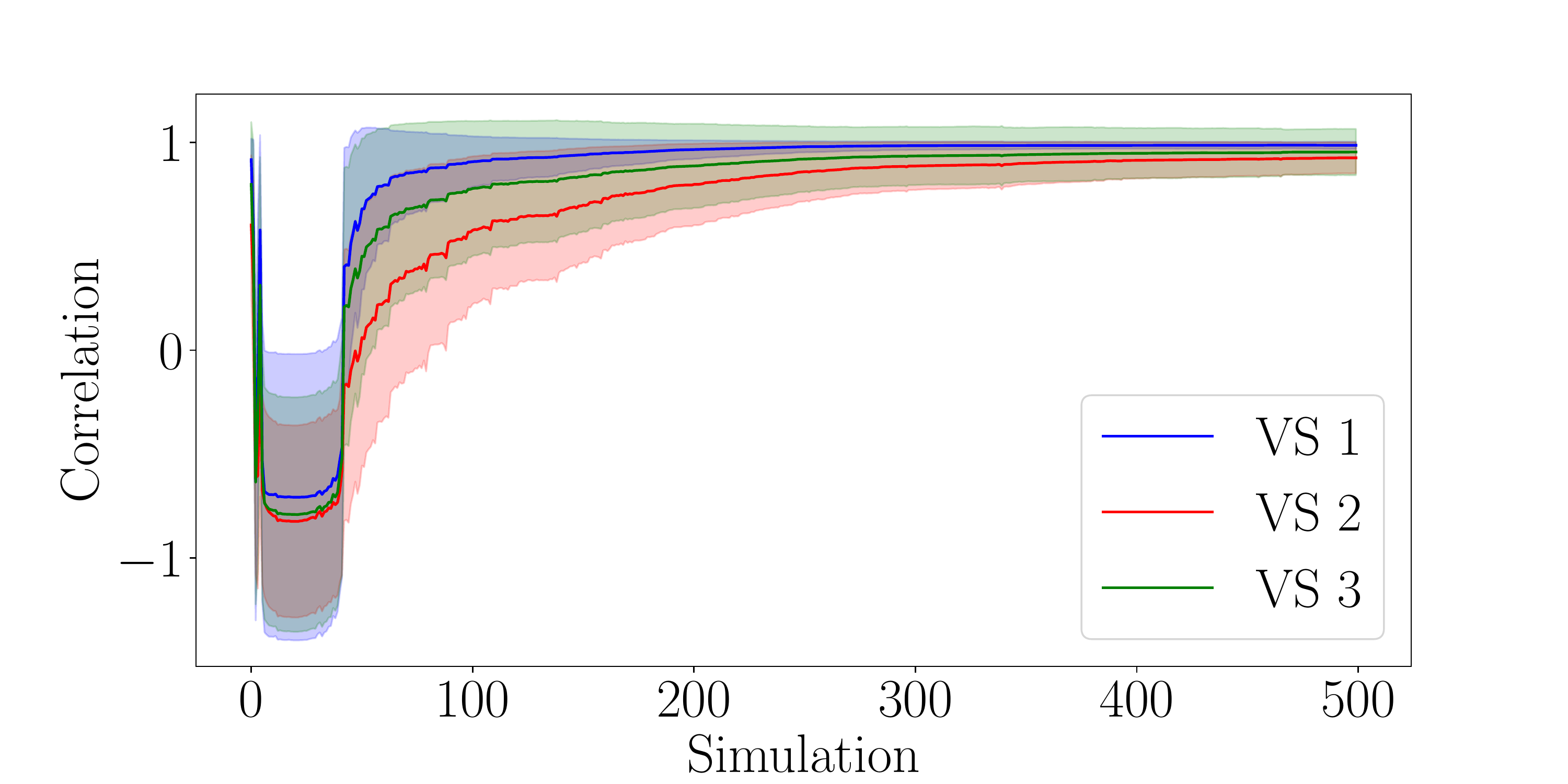}
      \caption{Case study A}
      \label{fig:nnA}
    \end{subfigure}\hfill
    \begin{subfigure}[t]{0.49\textwidth}
      \centering
      \includegraphics[width=\textwidth]{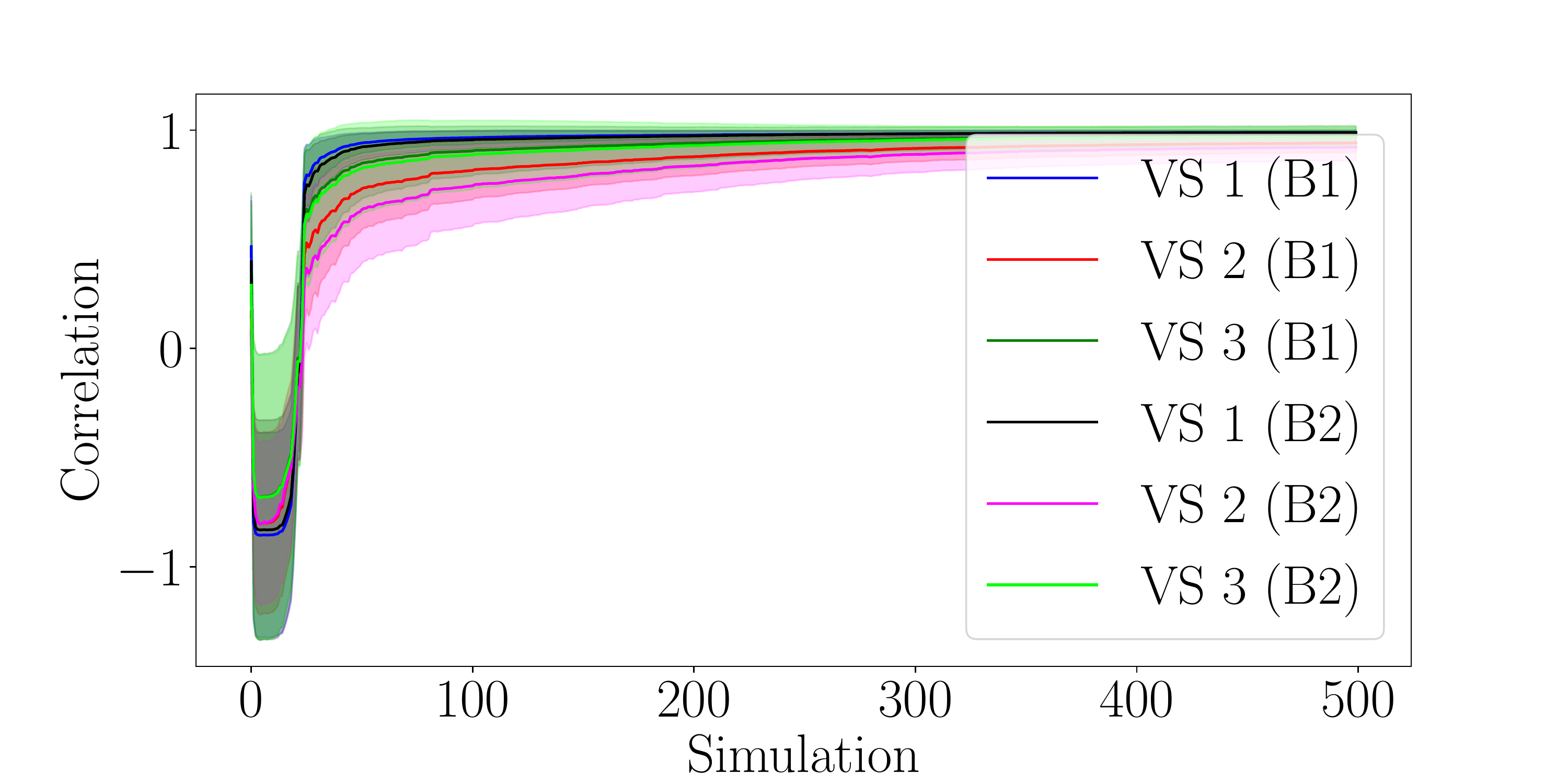}
      \caption{Case study B1}
      \label{fig:nnB1}
    \end{subfigure}
    \begin{subfigure}[t]{0.49\textwidth}
      \centering
      \includegraphics[width=\textwidth]{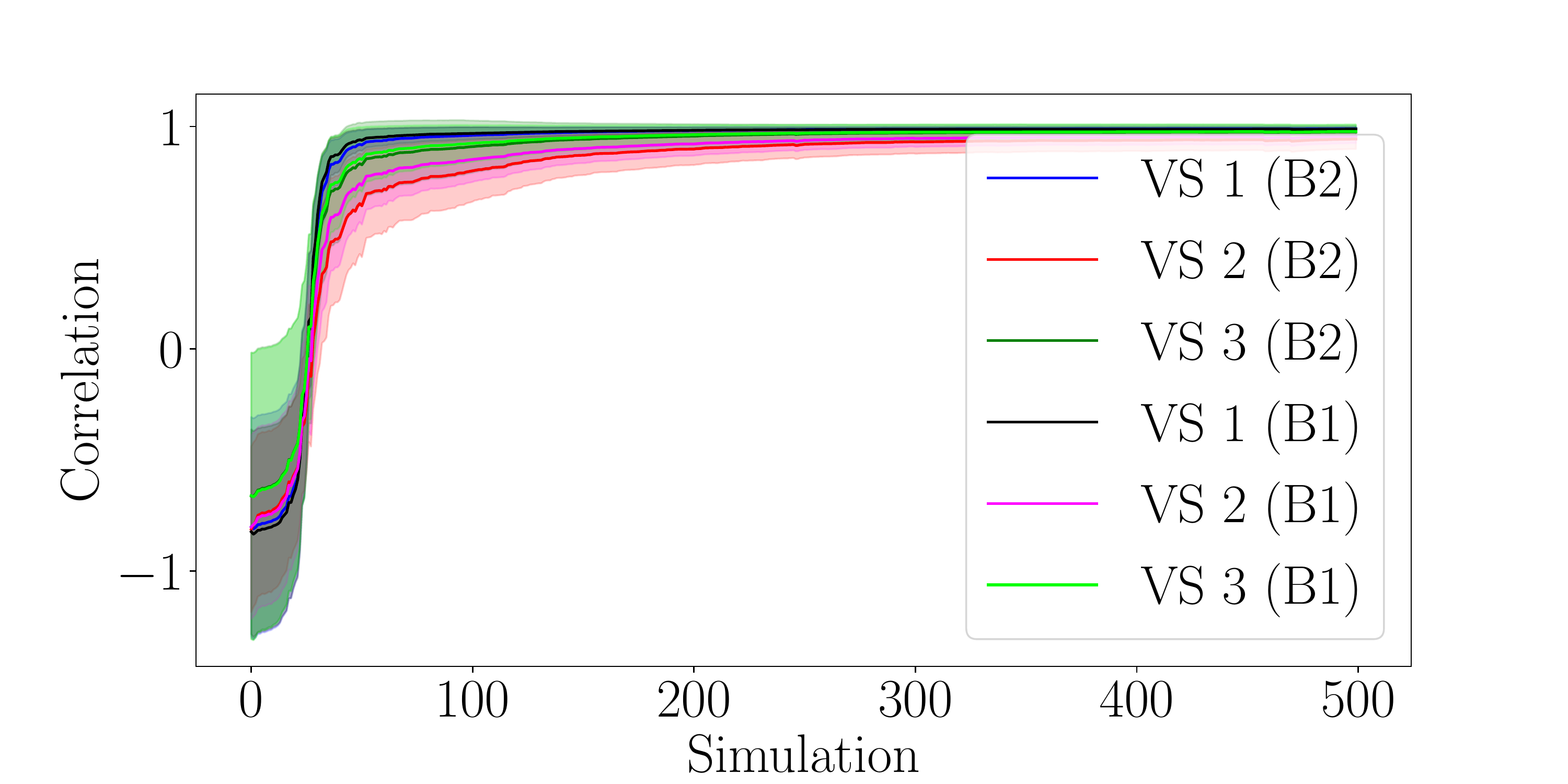}
      \caption{Case study B2}
      \label{fig:nnB2}
    \end{subfigure}\hfill
    \begin{subfigure}[t]{0.49\textwidth}
      \centering
      \includegraphics[width=\textwidth]{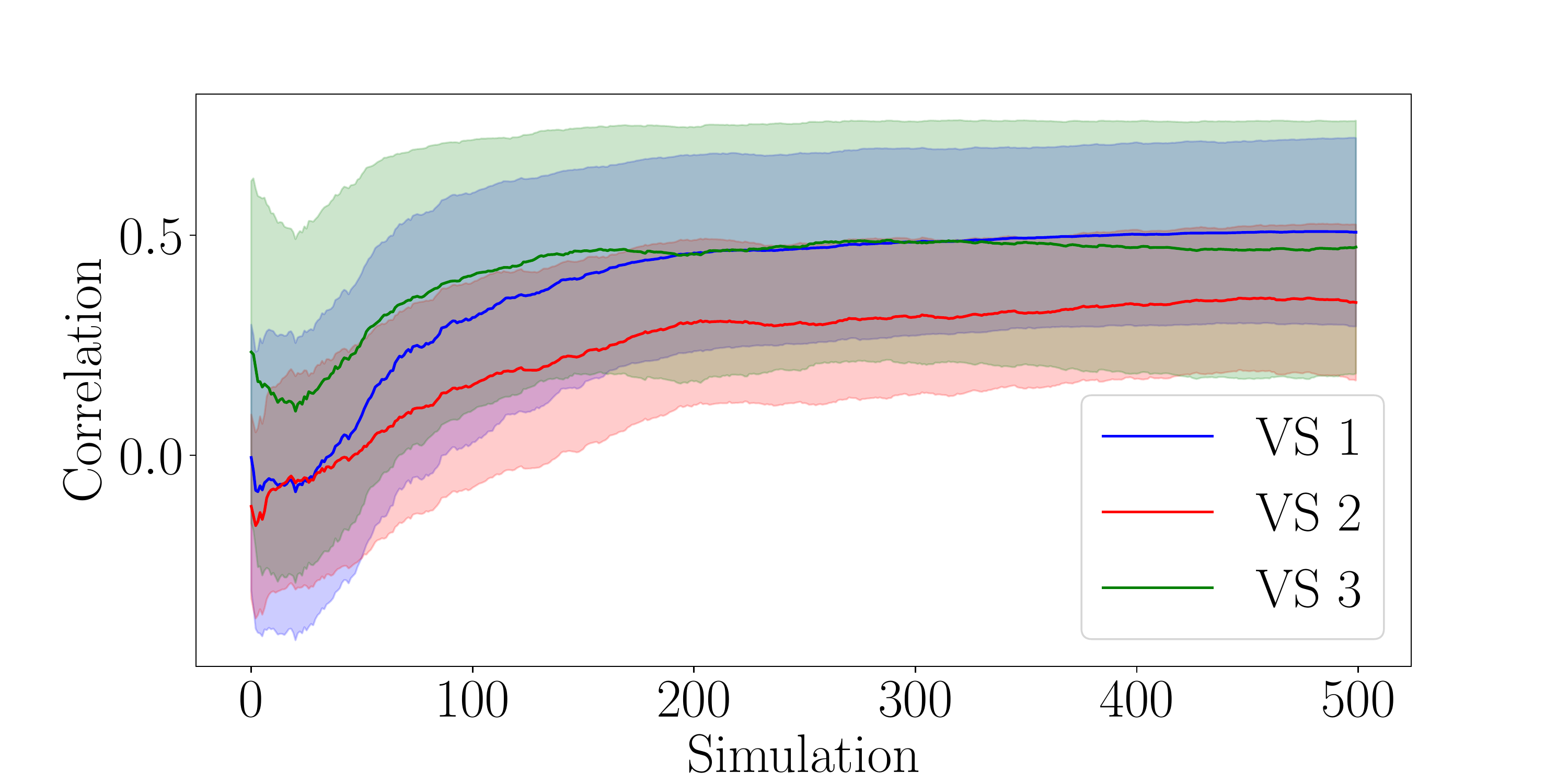}
      \caption{Case study C}
      \label{fig:nnC}
    \end{subfigure}
    \caption{Correlation between predicted global fitness and real global fitness of validation set during learning.}
    \label{fig:nn}
  \end{figure}

  \begin{figure}[h!]
    \centering
    \begin{subfigure}[t]{0.49\textwidth}
      \centering
      \includegraphics[width=\textwidth]{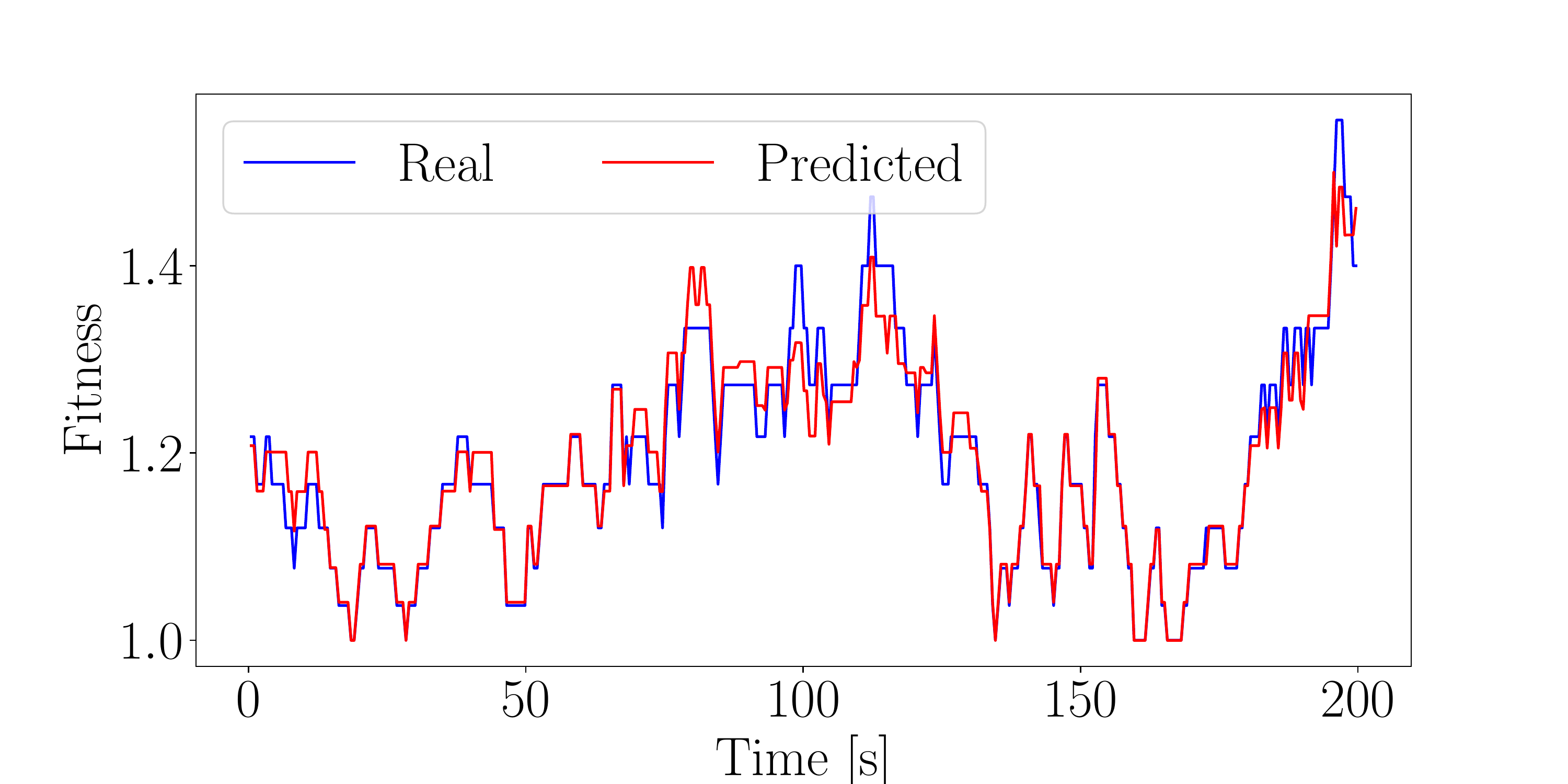}
      \caption{Case study A}
      \label{fig:sampleA}
    \end{subfigure}
    \hfill
    \begin{subfigure}[t]{0.49\textwidth}
      \centering
      \includegraphics[width=\textwidth]{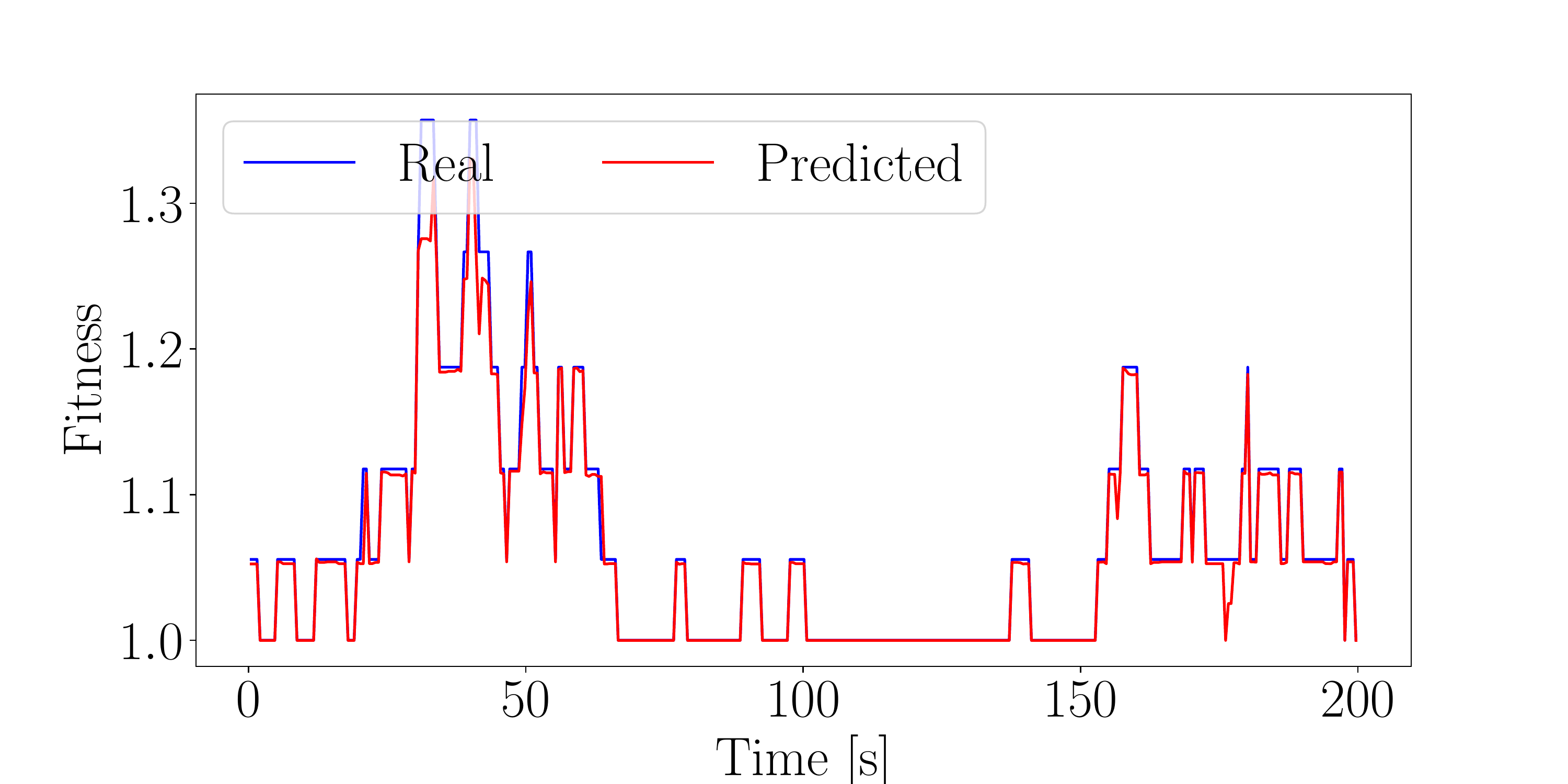}
      \caption{Case study B1}
      \label{fig:sampleB1}
    \end{subfigure}
    \begin{subfigure}[t]{0.49\textwidth}
      \centering
      \includegraphics[width=\textwidth]{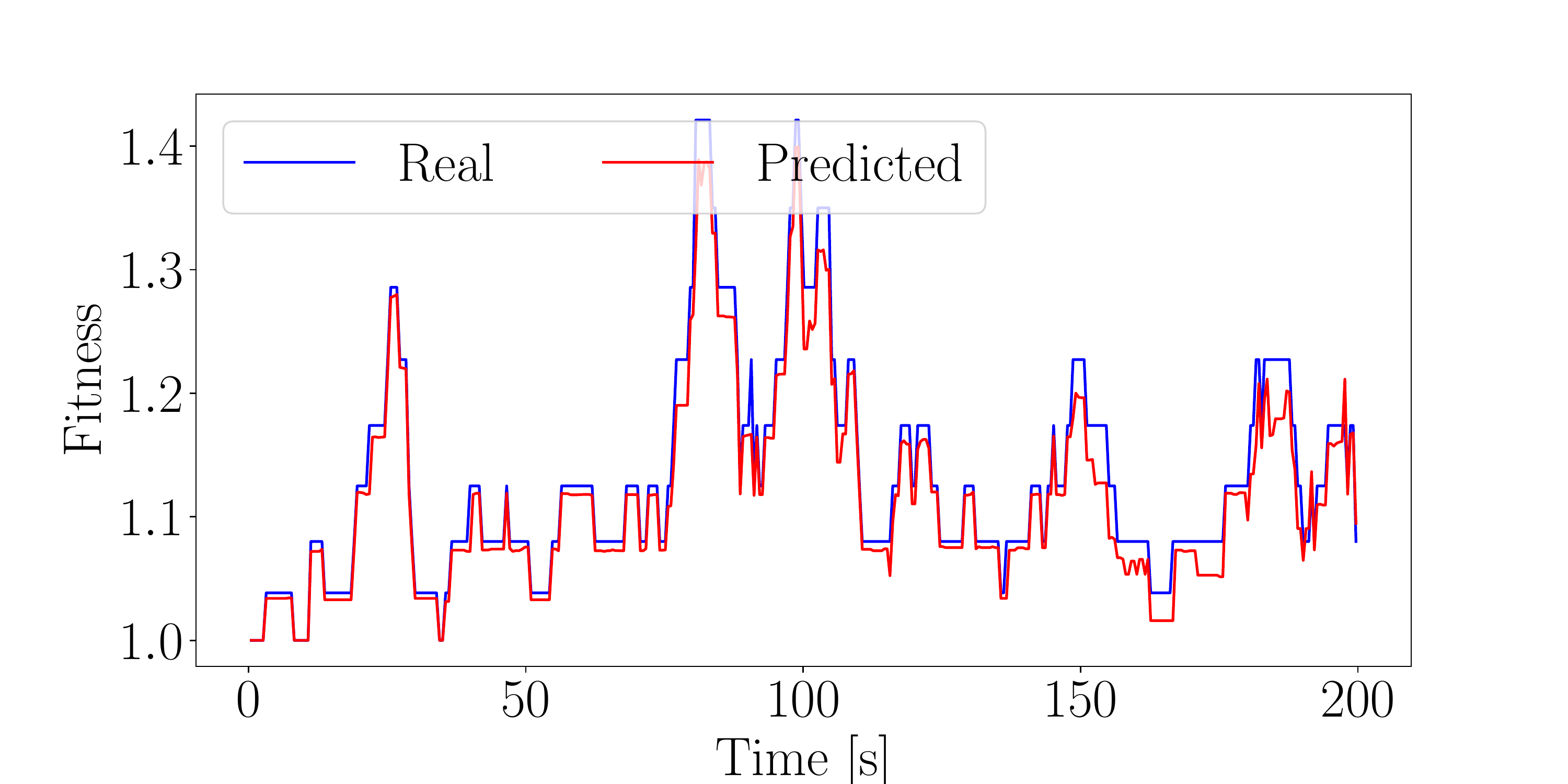}
      \caption{Case study B2}
      \label{fig:sampleB2}
    \end{subfigure}
    \hfill
    \begin{subfigure}[t]{0.49\textwidth}
      \centering
      \includegraphics[width=\textwidth]{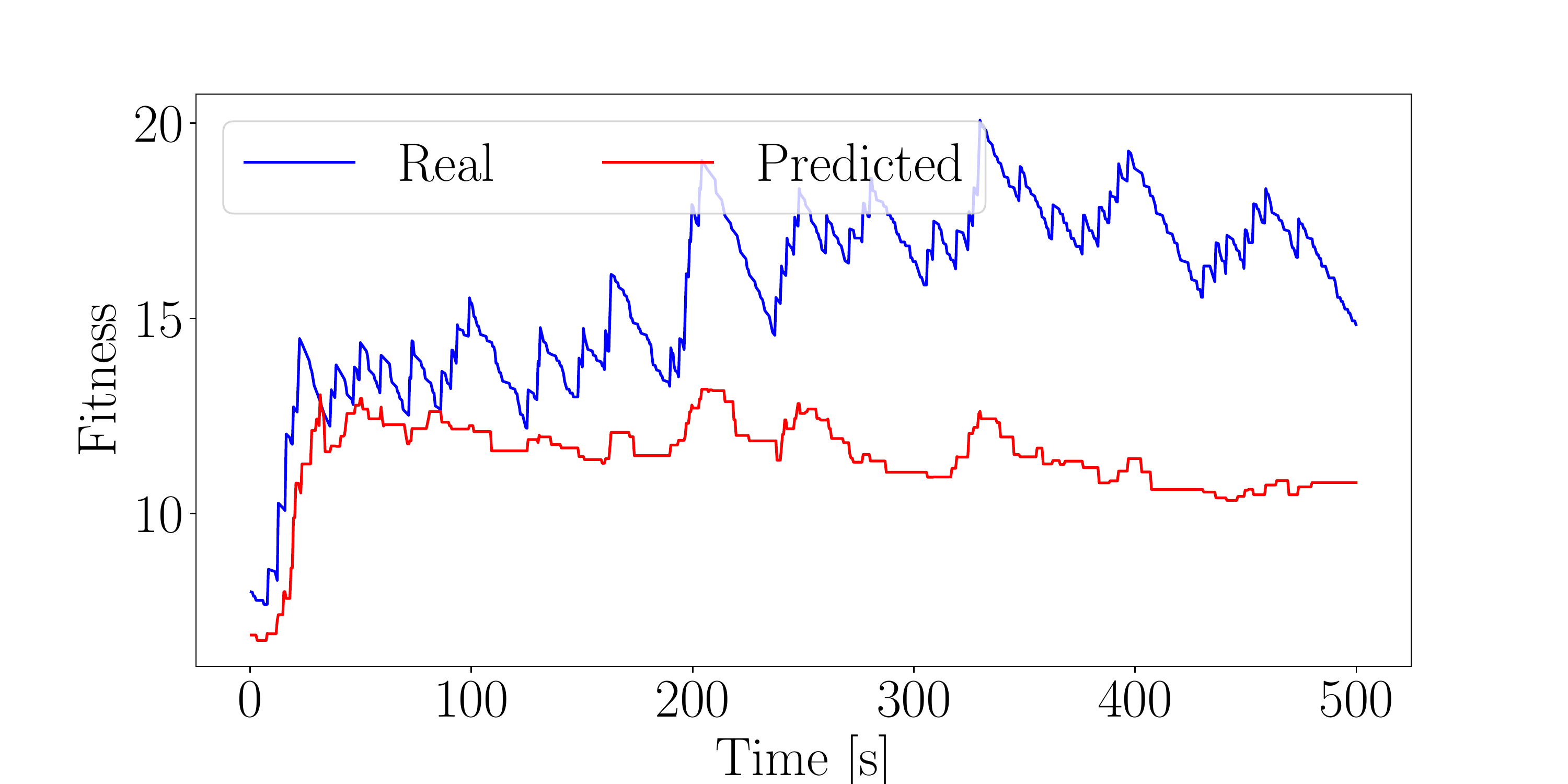}
      \caption{Case study C}
      \label{fig:sampleC}
    \end{subfigure}
    \caption{Sample global fitness estimates over one run, compared to the real fitness.}
    \label{fig:samples}
  \end{figure}

\subsubsection{Model 2: training and results}
  The transitions are estimated by the proportion of times that they occur over the simulations of the training data, as detailed in \secref{sec:model2}.
  \figref{fig:model} shows the convergence of all transition probabilities in the action models in $\mathbf{A}$, estimated over the training simulations.
  The figures shows the difference between the estimate at a given simulation and the final estimate after all simulations, providing a visualization of the convergence.
  It can be observed that more rare transitions follow a less graceful decline, providing only a rough estimate of their true likelihood.
  In practice, this problem will be mitigated when this framework is not used in a standalone procedure, as will be explored in \secref{sec:evo} and \secref{sec:online}.
  For online use, it would be beneficial to adopt a pre-trained model to provide the robots with initial estimates.

  \begin{figure}[t]
    \centering
    \begin{subfigure}[t]{0.49\textwidth}
      \centering
      \includegraphics[width=\textwidth]{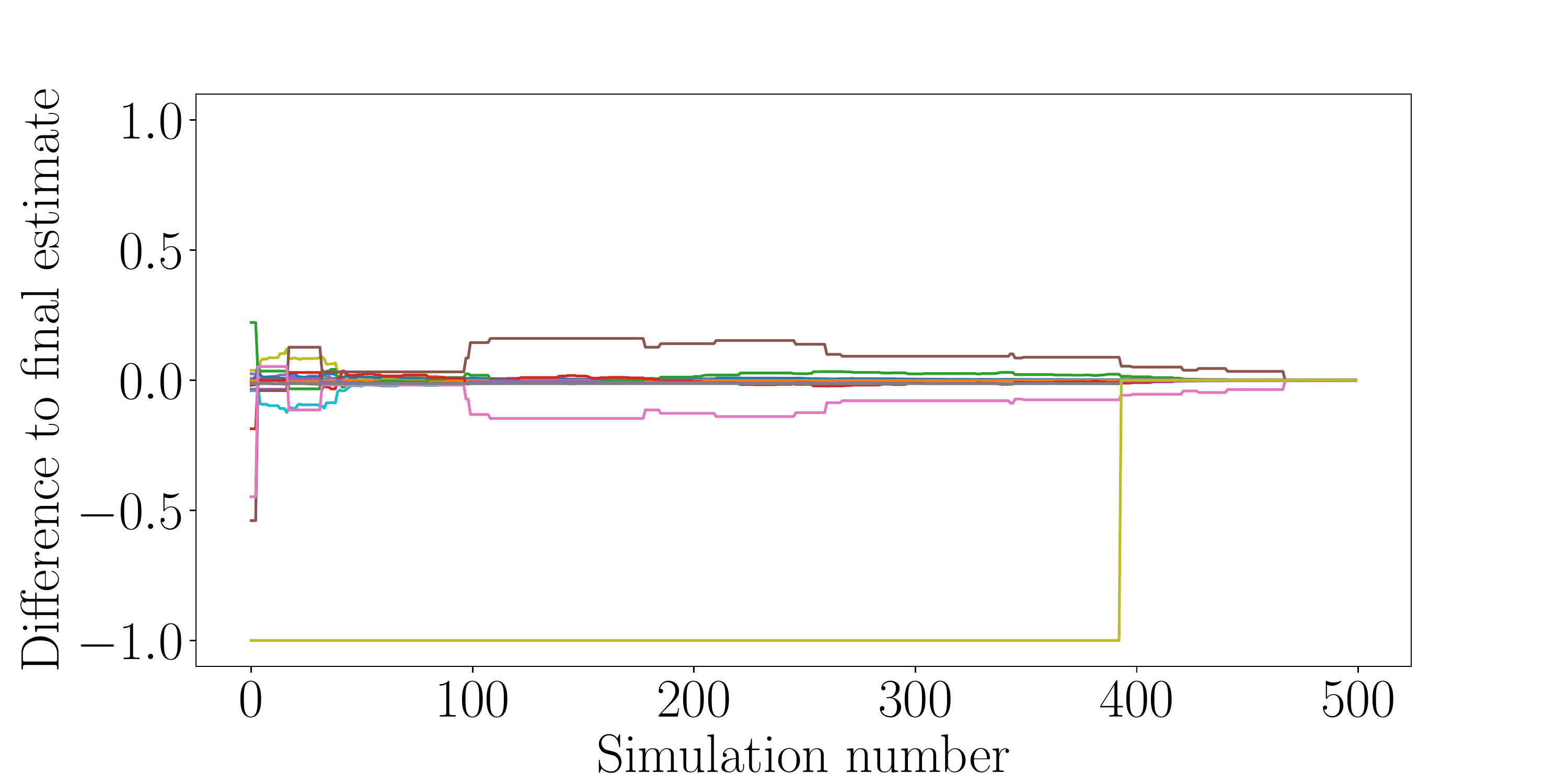}
      \caption{Case study A}
      \label{fig:modelA}
    \end{subfigure}\hfill
    \begin{subfigure}[t]{0.49\textwidth}
      \centering
      \includegraphics[width=\textwidth]{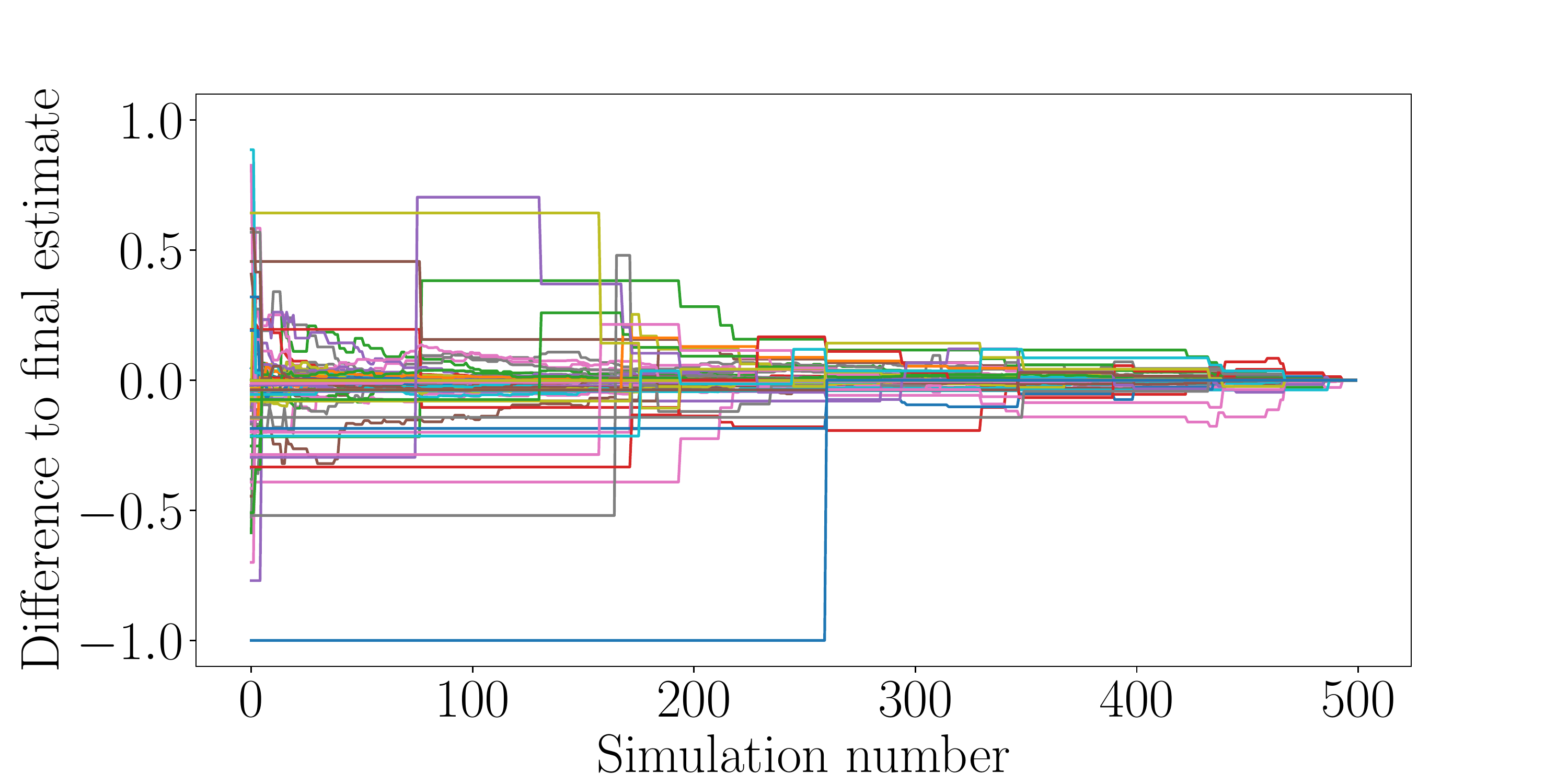}
      \caption{Case study B1}
      \label{fig:modelB1}
    \end{subfigure}
    \begin{subfigure}[t]{0.49\textwidth}
      \centering
      \includegraphics[width=\textwidth]{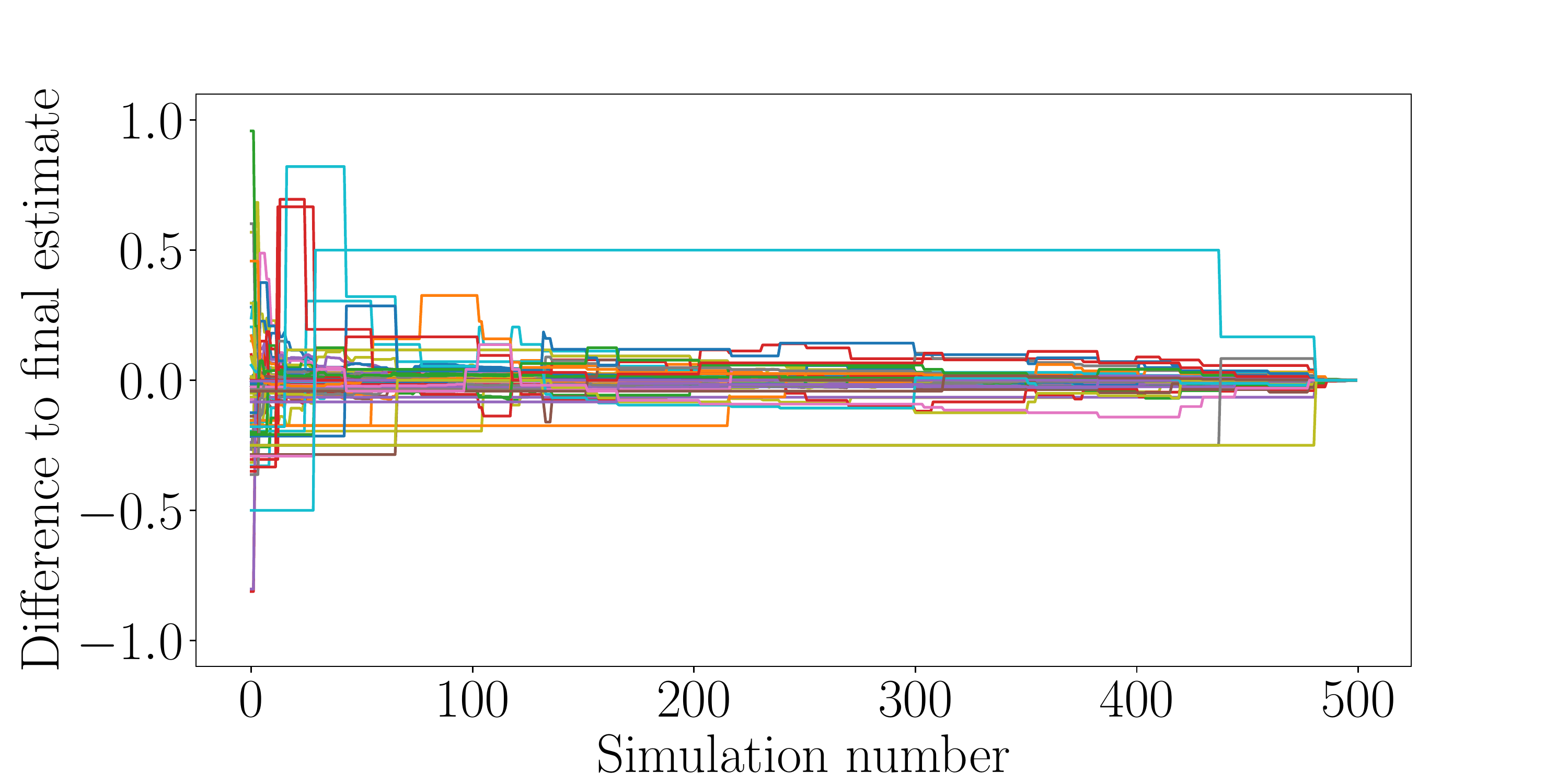}
      \caption{Case study B2}
      \label{fig:modelB2}
    \end{subfigure}\hfill
    \begin{subfigure}[t]{0.49\textwidth}
      \centering
      \includegraphics[width=\textwidth]{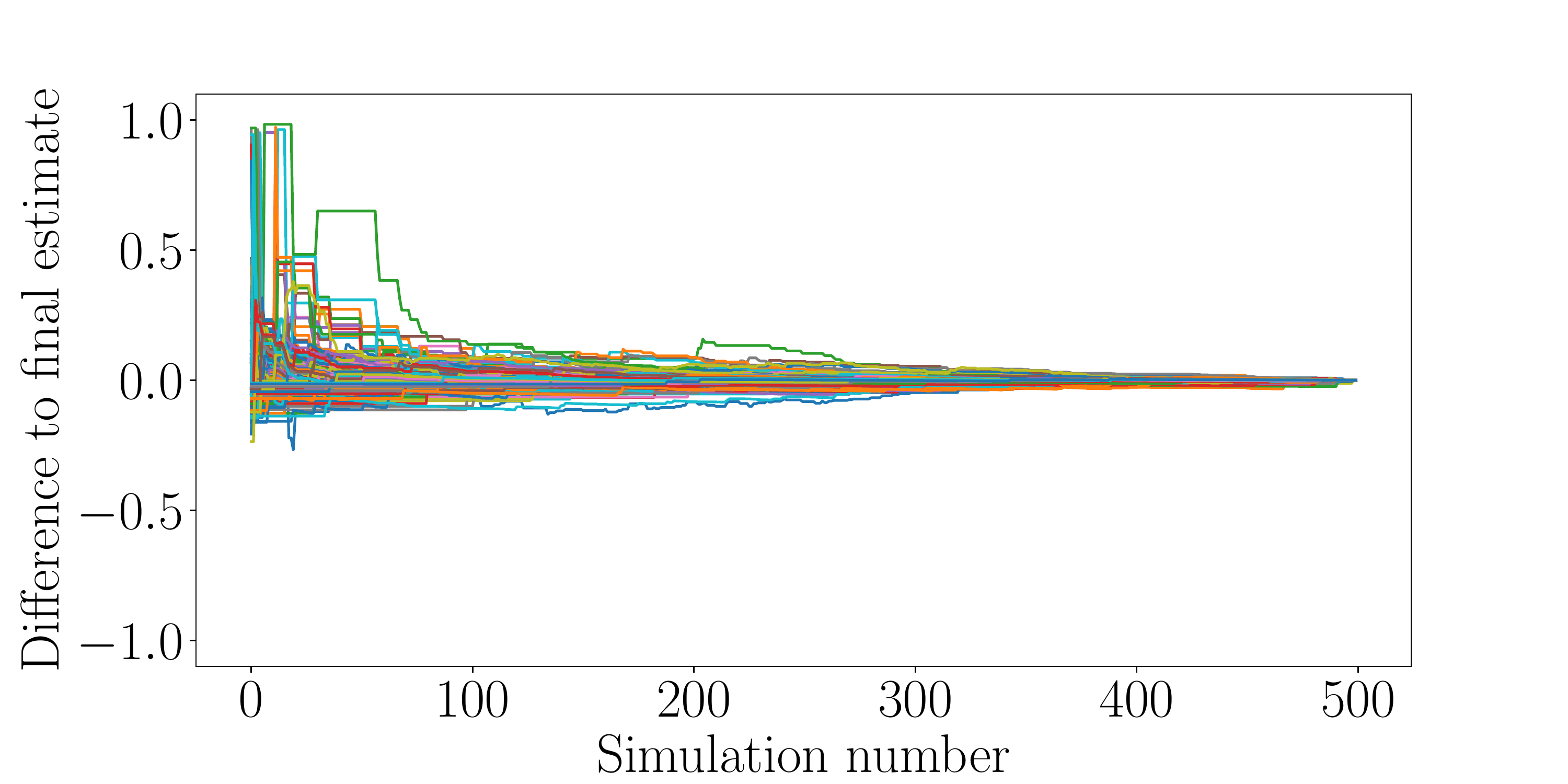}
      \caption{Case study C}
      \label{fig:modelC}
    \end{subfigure}
    \caption{Convergence of transition model probabilities over 500 simulation runs for each task.}
    \label{fig:model}
  \end{figure}

\subsection{Swarm performance evaluation of optimized policy}
\label{sec:results_evaluation}
  Once both models are generated, the set $\mathcal{S}_{des}$ can be extracted and the policies can be optimized accordingly.
  The policy optimization in these results was achieved using an evolutionary algorithm, implemented using DEAP \citep{DEAP_JMLR2012}, in order to optimize the parameters according to the PageRank-based metric from \secref{sec:optimization}.
  For each case study, the performance of the swarm is matched against: 
  1) an original baseline performance, evaluated using 100 random policies, and 2) the performance of a policy that was evolved using a centrally monitored evolutionary algorithm, as is state of the art practice in the literature.
  The performance was developed and tested in arenas of $20 \times 20$ meters.

  \subsubsection{Performance for case studies A, B1, and B2 (aggregation)}
  \label{sec:performance_aggregation}
    In all three cases, the set $\mathcal{S}_{des}$ included all local states featuring one or more neighbors that had been explored during the training runs.
    This is a sensible result:
    the algorithm finds out that the robots should arrive at these local states in order to maximize $\hat{F}_g$.
    In turn, the optimized policy directly aimed at maximizing the probability of ending up with a local state featuring one or more neighbors.
    The swarm's performance is evaluated with 100 simulations of $200~s$.
    The performance of the optimized policies can be seen in Figures \ref{fig:benchmarkA}, \ref{fig:benchmarkB1}, and \ref{fig:benchmarkB2}, where they are compared to a baseline performance and to an evolved performance.
    The mean performance over the evaluation run using the optimized policy can be appreciated in Figures \ref{fig:fitnesslogsA}, \ref{fig:fitnesslogsB1}, and \ref{fig:fitnesslogsB2}.
    Case studies A and B1 match the policy evolved using a standard evolutionary algorithm.
    The performance of the optimized policy for case study B2 did not match the true optimum evolved via evolution.
    This is because it could not exploit global properties and its optimized behavior was greedy in its exploitation of Model 2, which did not allow the robots for optimum exploration in the environment.
    Note that in all cases the optimized policy from our framework was achieved with 500 simulations, which is equivalent to the number of simulations in just \emph{one} generation of the evolutionary algorithm that was employed for the comparison (the evolution used a population of 100 policies that was evaluated 5 times each).
    Case study A was evolved in 100 generations, and case studies B1 and B2 were evolved with $\approx$ 300 generations each.
    To achieve the best of both worlds, the increased efficiency of the model-based approach can be exploited within the evolutionary setup.
    This idea will be explored in \secref{sec:evo}.

    A particularly interesting emergent result can be appreciated in the behavior of the robots in case study B1.
    This case study featured the particular difficulty that the robots were always told to move with forward speed $v_{\mathrm{cmd}}=0.5~m/s$, purposely making aggregation difficult due to their inability to stop.
    The optimized policy, however, learned to fully exploit collision avoidance (i.e., repulsion) forces between robots, which actually created the ability to stop by balancing the commanded forward speed with the repulsion force.
    By exploiting the repulsion forces, the robots learned that it was possible to remain in equilibrium and readily form clusters.
    Other clusters were also formed by robots repeatedly moving around each other.

  \subsubsection{Performance for case study C (foraging)}
  \label{sec:performance_forage}
    For the foraging task, we found that $\mathcal{S}_{des}$ was only composed of states where $f_g\geq-1.2069$.
    The swarm's performance is evaluated with 100 simulations of $200~s$.
    The performance can be appreciated in \figref{fig:benchmarkC} and \figref{fig:fitnesslogsC}, showing that it matches the performance of an evolved policy.
    In all cases, the fitness improves directly from the beginning of a run and stabilizes after approximately $100~s$.
    From these results we can establish that, even though the correlation of Model 1 was found to be lower than for the aggregation task, the extracted desired states and corresponding optimized policy can still match an evolved policy.
    This result may be attributed to the accuracy of Model 2, which in \figref{fig:modelC} can be seen to more quickly converge to its final estimates.

  \begin{figure}[t]
    \centering
    \begin{subfigure}[t]{0.49\textwidth}
      \centering
      \includegraphics[width=\textwidth]{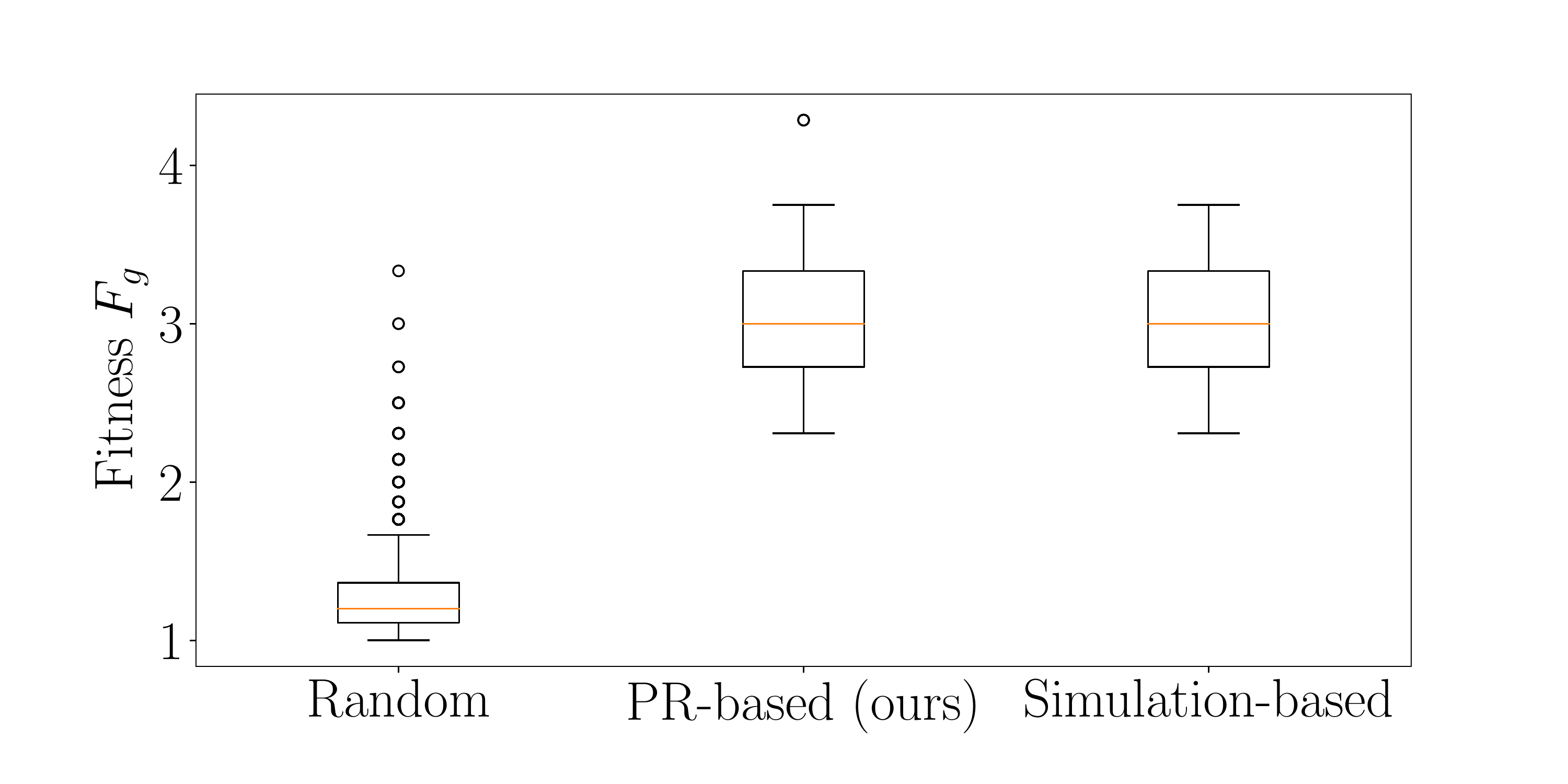}
      \caption{Case study A}
      \label{fig:benchmarkA}
    \end{subfigure}
    \hfill
    \begin{subfigure}[t]{0.49\textwidth}
      \centering
      \includegraphics[width=\textwidth]{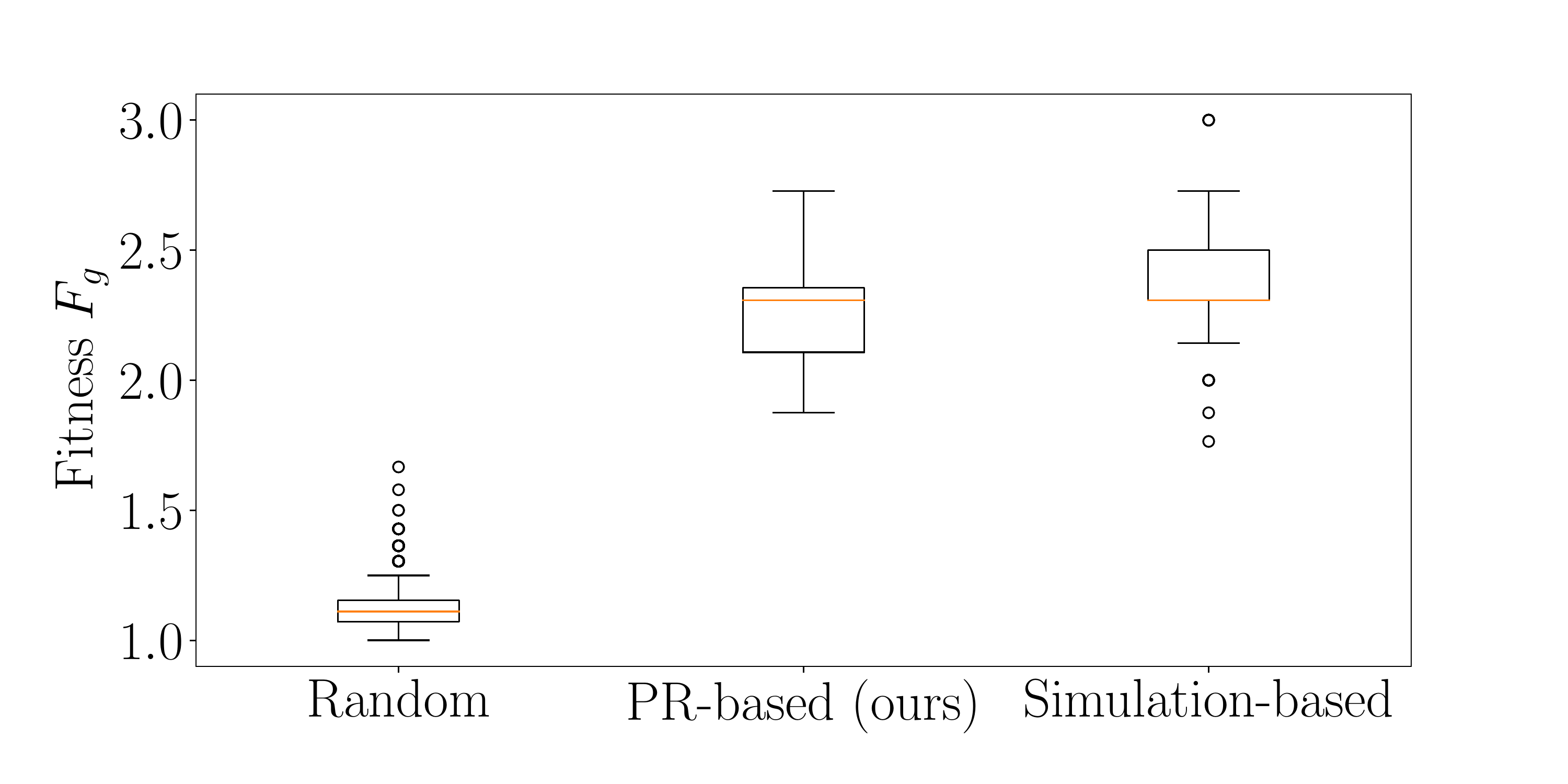}
      \caption{Case study B1}
      \label{fig:benchmarkB1}
    \end{subfigure}
    \begin{subfigure}[t]{0.49\textwidth}
      \centering
      \includegraphics[width=\textwidth]{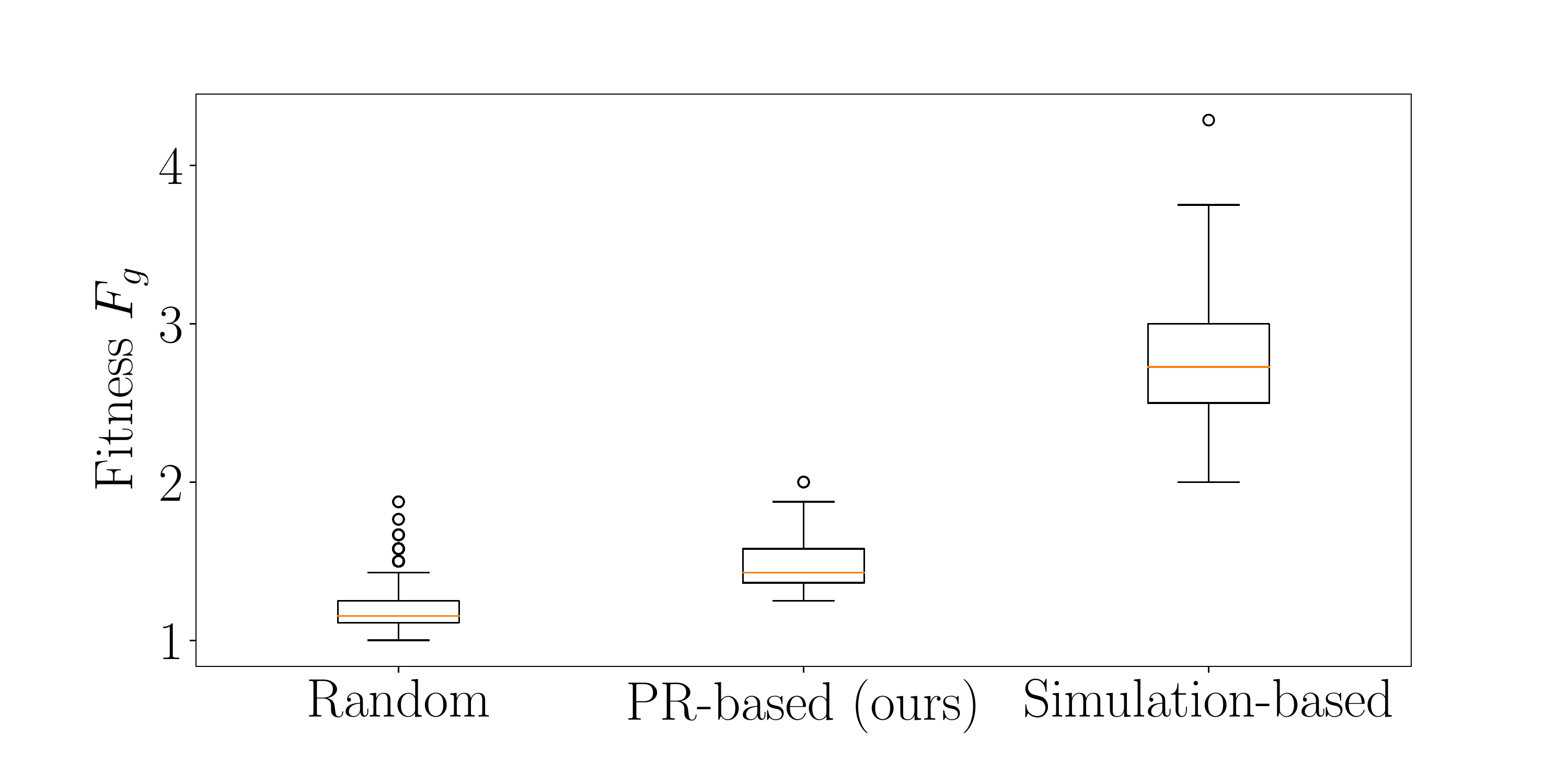}
      \caption{Case study B2}
      \label{fig:benchmarkB2}
    \end{subfigure}
    \hfill
    \begin{subfigure}[t]{0.49\textwidth}
      \centering
      \includegraphics[width=\textwidth]{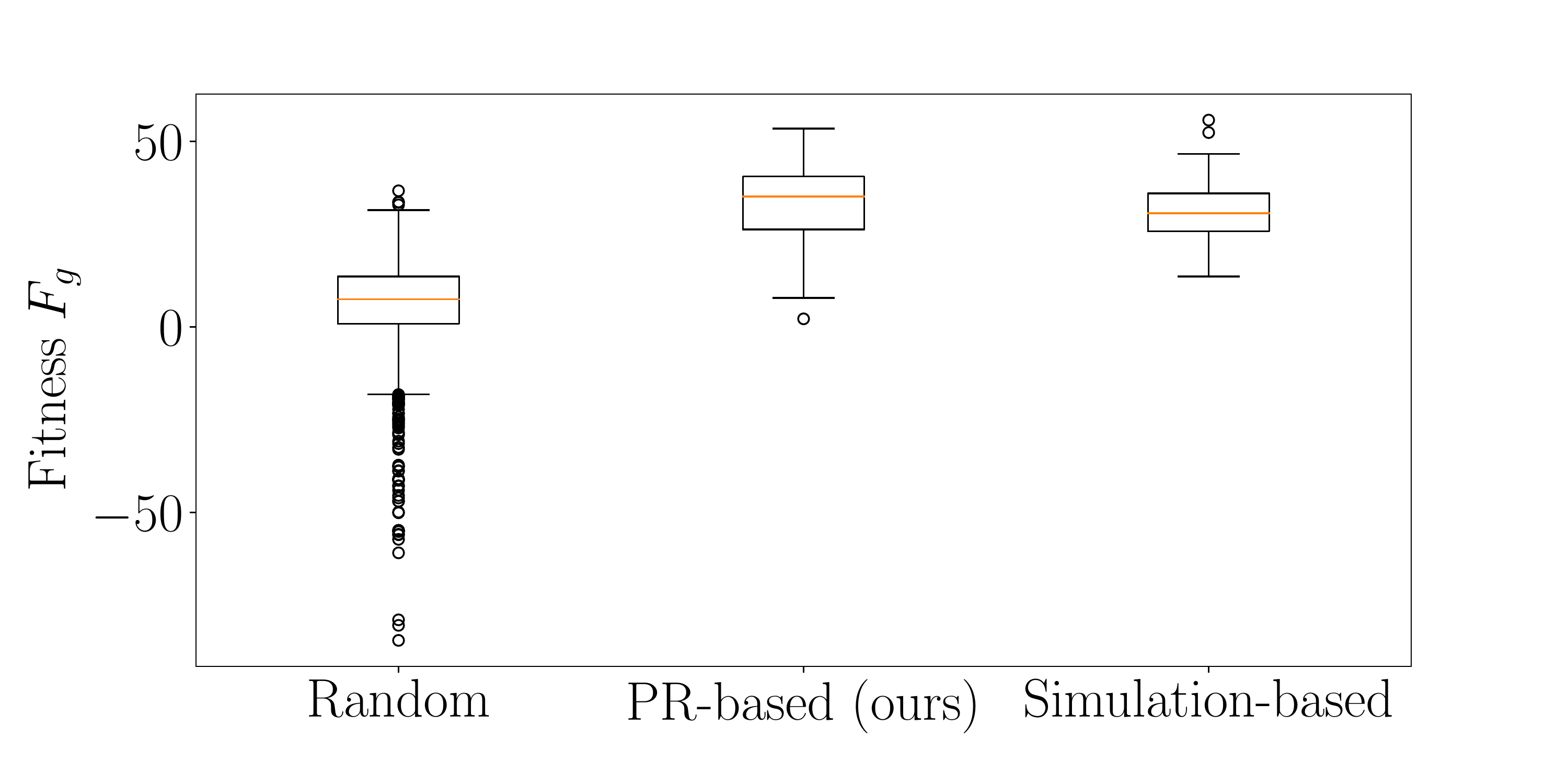}
      \caption{Case study C}
      \label{fig:benchmarkC}
    \end{subfigure}
    \caption{Box and whisker plots of the fitness achieved by the PageRank-optimized policy (denoted ``PR-based'') against the performance of random policies (denoted ``Random'') and optimized policies evolved using a conventional evolutionary algorithm that evaluates the performance via simulation (denoted ``Simulation-based'').}
    \label{fig:benchmark}
  \end{figure}

  \begin{figure}[h!]
    \centering
    \begin{subfigure}[t]{0.49\textwidth}
      \centering
      \includegraphics[width=\textwidth]{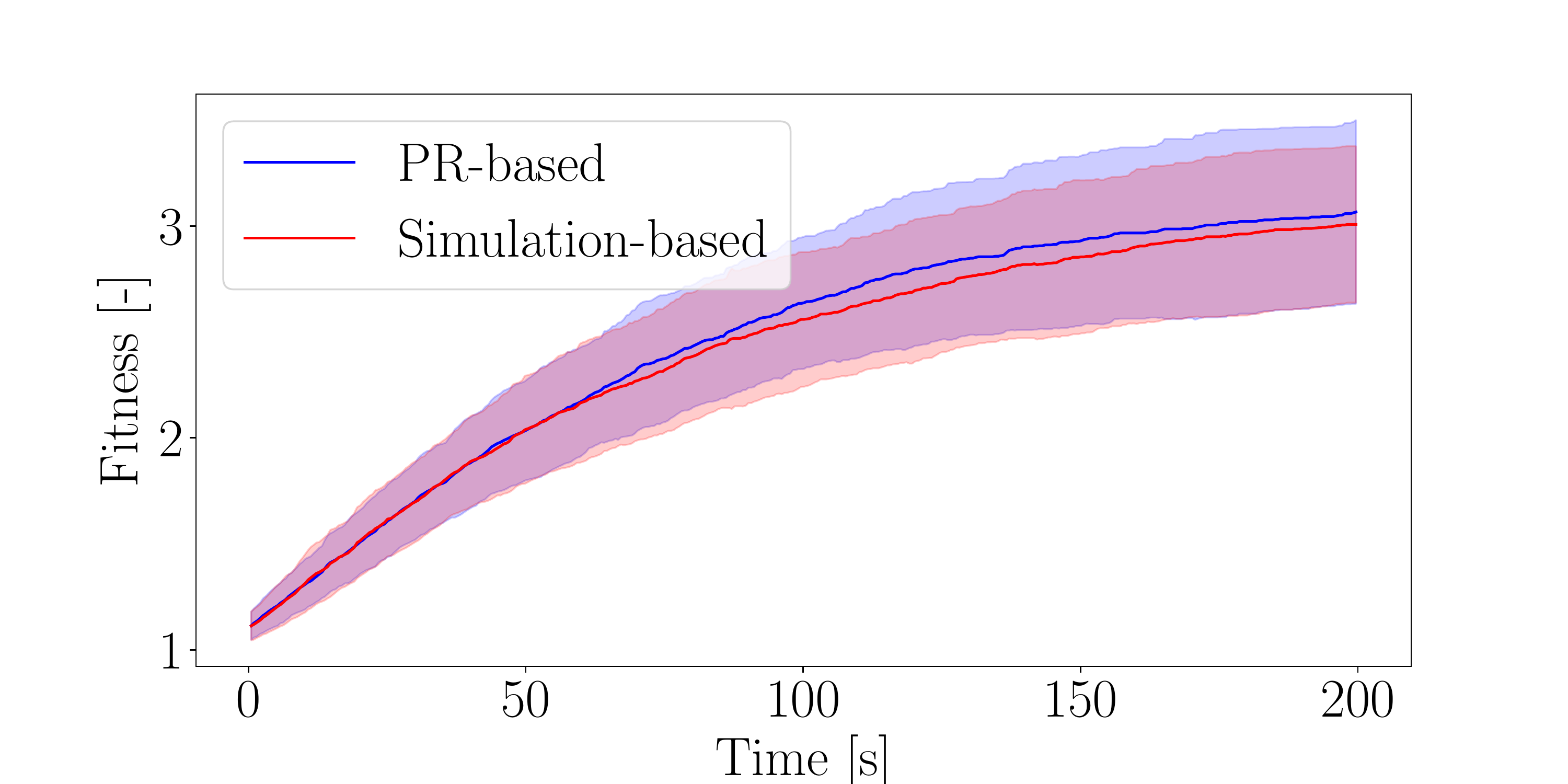}
      \caption{Case study A}
      \label{fig:fitnesslogsA}
    \end{subfigure}
    \hfill
    \begin{subfigure}[t]{0.49\textwidth}
      \centering
      \includegraphics[width=\textwidth]{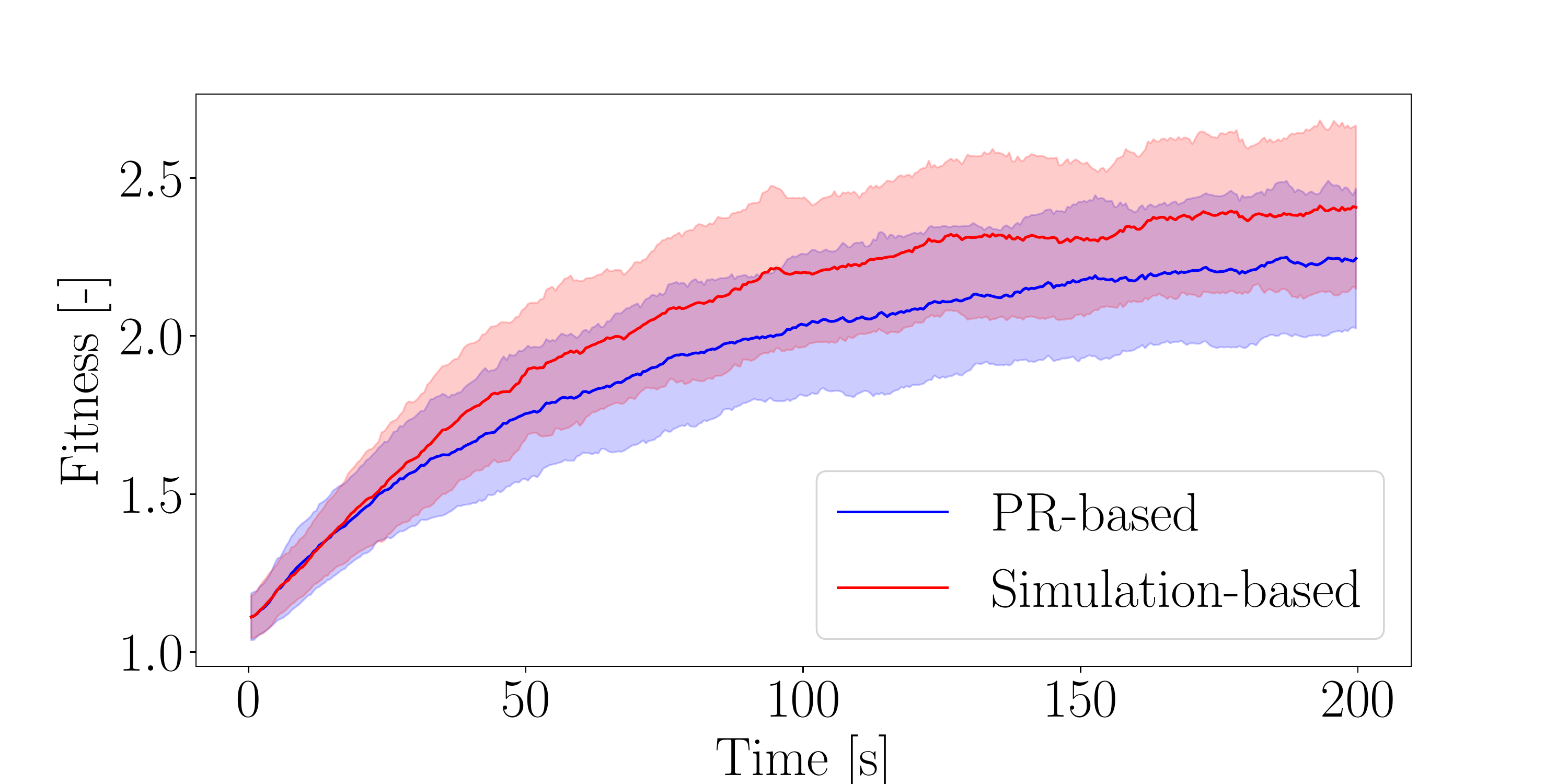}
      \caption{Case study B1}
      \label{fig:fitnesslogsB1}
    \end{subfigure}
    \begin{subfigure}[t]{0.49\textwidth}
      \centering
      \includegraphics[width=\textwidth]{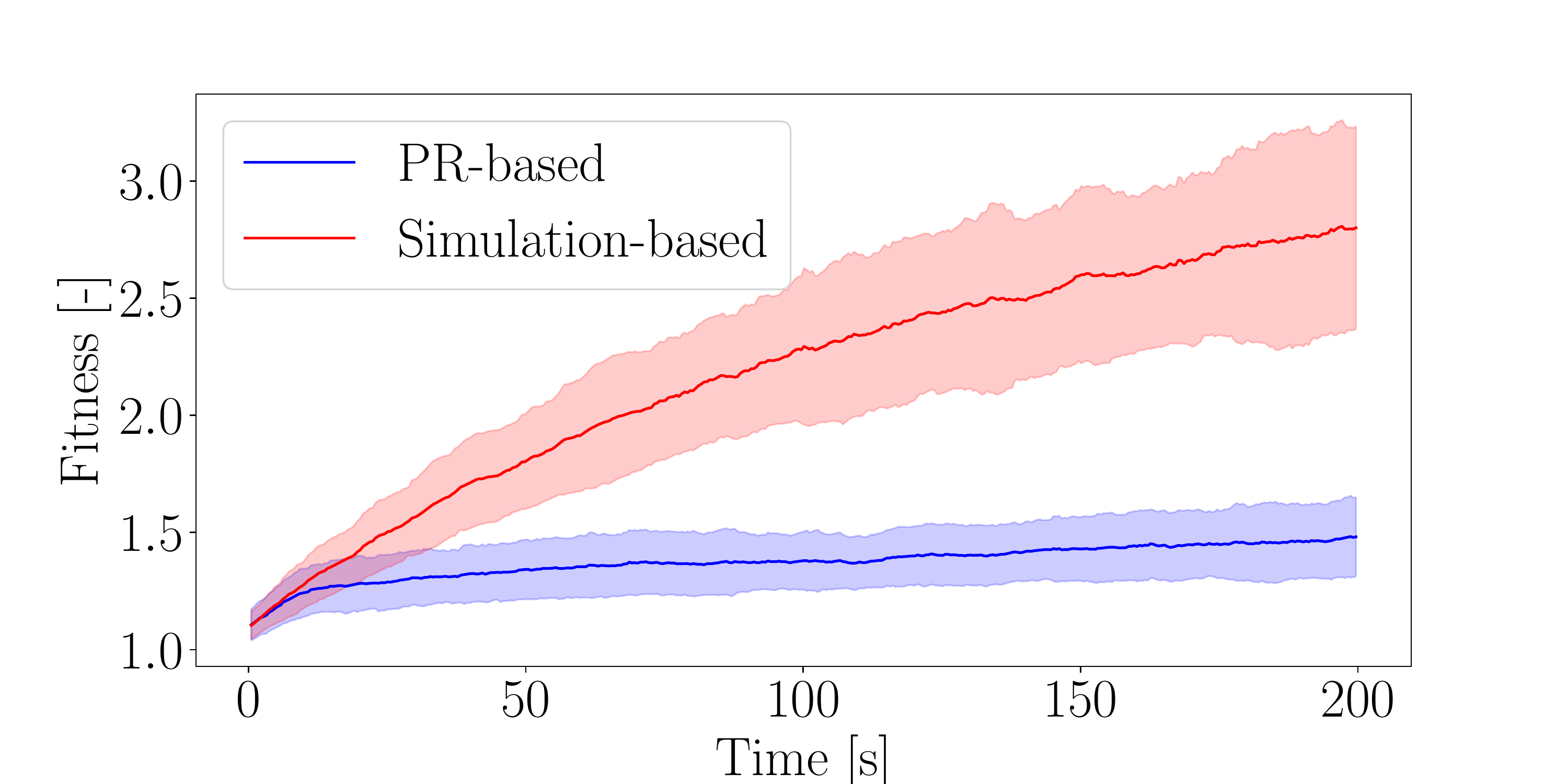}
      \caption{Case study B2}
      \label{fig:fitnesslogsB2}
    \end{subfigure}
    \hfill
    \begin{subfigure}[t]{0.49\textwidth}
      \centering
      \includegraphics[width=\textwidth]{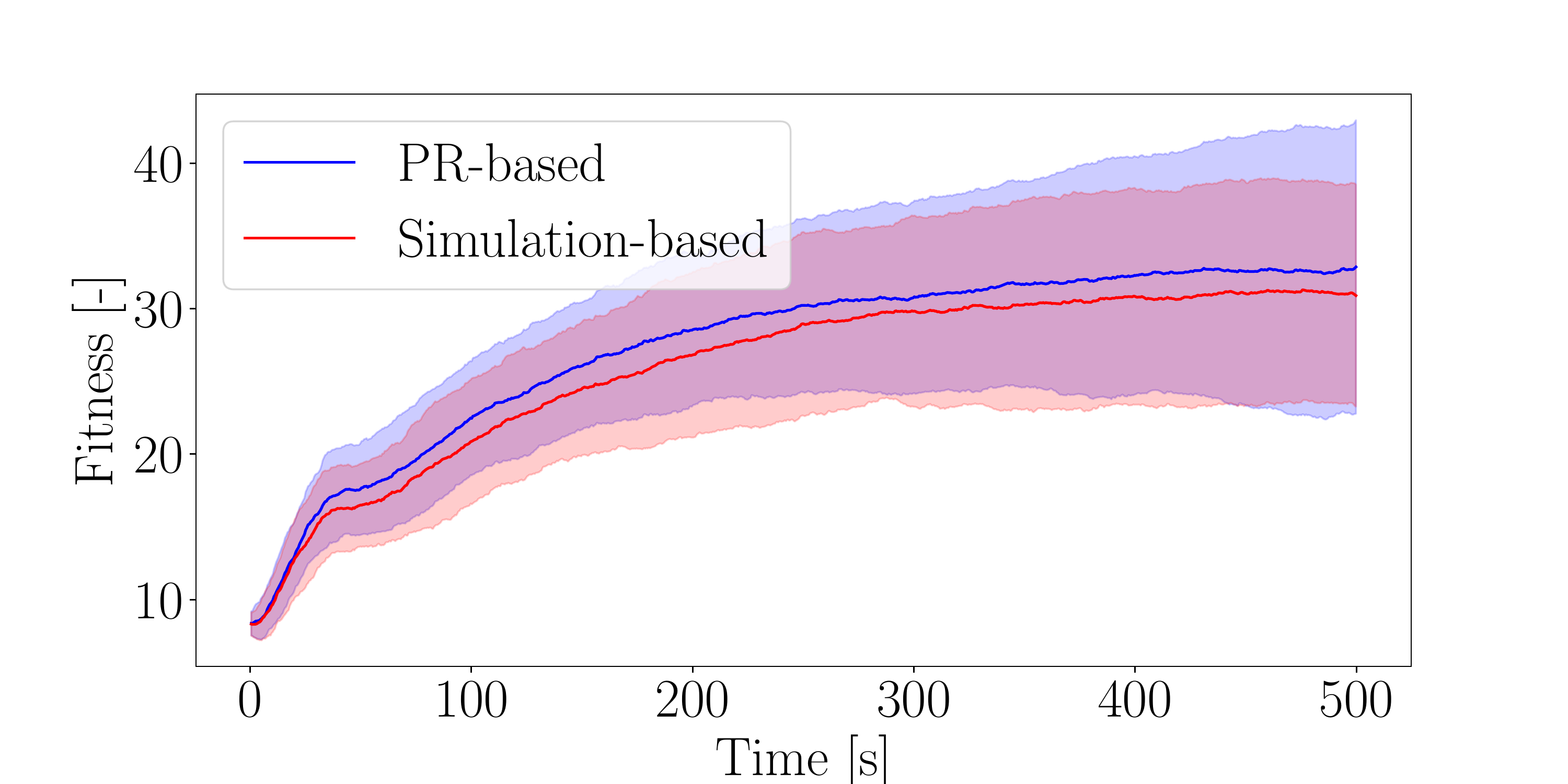}
      \caption{Case study C}
      \label{fig:fitnesslogsC}
    \end{subfigure}
    \caption{Fitness of 100 validation runs using the optimized policies, showing mean and standard deviation.
    As in \figref{fig:benchmark}, ``PR-based'' relates to our PageRank-based optimization, and ``Simulation-based'' refers to a conventional evolutionary algorithm that evaluates the performance using simulations.}
    \label{fig:fitnesslogs}
  \end{figure}

\subsection{Understandability and verification analysis}
\label{sec:results_verification}
  This section discusses how the two models help to better understand the behavior and performance of a swarm.
  The discussion is divided into four parts.
  First, we discuss the ability of Model 1 to find suitable desired local states.
  Then, the policy optimization procedure and its results are analyzed.
  This is followed by a discussion on the identified livelocks, and subsequently deadlocks, for each case study.

  \subsubsection{Global goal to local objectives}
  \label{sec:results_verification_model1}
    Model 1 is a function that relates the distribution of local states in the swarm to a global performance.
    When designing the swarm behavior, it is primarily used as a tool to extract the set of local states that are expected to lead to a high performance by the swarm.
    These local states provide direct insight into the individual goals of the robots and their goals within the swarm.
    Central to the swarming framework in this paper is that the swarm is optimized by balancing a greedy local behavior with multiple ``altruistic'' local objectives that benefit the swarm's performance.
    The desired states for each case study can be seen in green in \figref{fig:pagerankdiff}.
    For the aggregation tasks, it is possible to see that all local states with one or more neighbors are regarded as desired.
    This is a reasonable result, as robots learn to aggregate in small clusters.
    This behavior is also observed by the policy optimized via standard evolution, due to the time constraint on the simulations.
    For the foraging task, the set of desired local states include all local states with $f_g>0$, but also a few local states where $f_g<0$.
    This shows that the robot's goal is primarily to maximize the observed local difference, but also that the robots are ready to accept a small (perceived) loss in food at the nest, as this may still be beneficial for the global performance.

    \subsubsection{Policy optimization, analysis, and inspection}
      The transition model (Model 2) provides information about the relative likelihood of transitioning to each local state.
      Through this, we can study the estimated likelihood of reaching a particular local state.
      \figref{fig:pagerankoriginal} shows the PageRank score for each individual state, separated into the policy model
      ($\mathbf{H}_\pi$)
      and the environment model
      ($\mathbf{E}$).
      For the aggregation case studies (A, B1, and B2, in \figref{fig:pagerankoriginalA}, \figref{fig:pagerankoriginalB1}, and \figref{fig:pagerankoriginalB2}, respectively) it is possible to see that, adopting random policies, robots are most likely to be without neighbors.
      They are progressively less likely to finish with local states with more neighbors.
      For the foraging case study, it can be seen that motion by the robots is most likely to cause a positive observation (local states 15-30), whereas the impact of the environment is most likely to cause a negative observation (local states 0-15).
      The estimated effects of the optimized policy on the PageRank score are shown in \figref{fig:pagerankdiff} for all case studies.
      This PageRank analysis is useful to understand the rationale behind the optimized policies.
      Case study A (\figref{fig:pagerankoriginalA}) is seen to benefit from the environment transitions rather than motion, particularly for states with one or more neighbors.
      Case studies B1 (\figref{fig:pagerankoriginalB1}) and B2 (\figref{fig:pagerankoriginalB2}) show that statistically the system is very unlikely to reach any local state with neighbors.
      This shows that the robot must take certain actions to increase its likelihood of having neighbors, which was observed when we analyzed the optimized behavior.
      Case study C (\figref{fig:pagerankoriginalC}) shows a contrasting correlation between the action model and the environment model, implying that choosing to explore can increase the probability of having an observation with $f_g > 0$ (local states 15-30), with a significantly higher probability of being just above $f_g = 0$.
        \begin{figure}[t]
    \centering
    \begin{subfigure}[t]{0.49\textwidth}
      \centering
      \includegraphics[width=\textwidth]{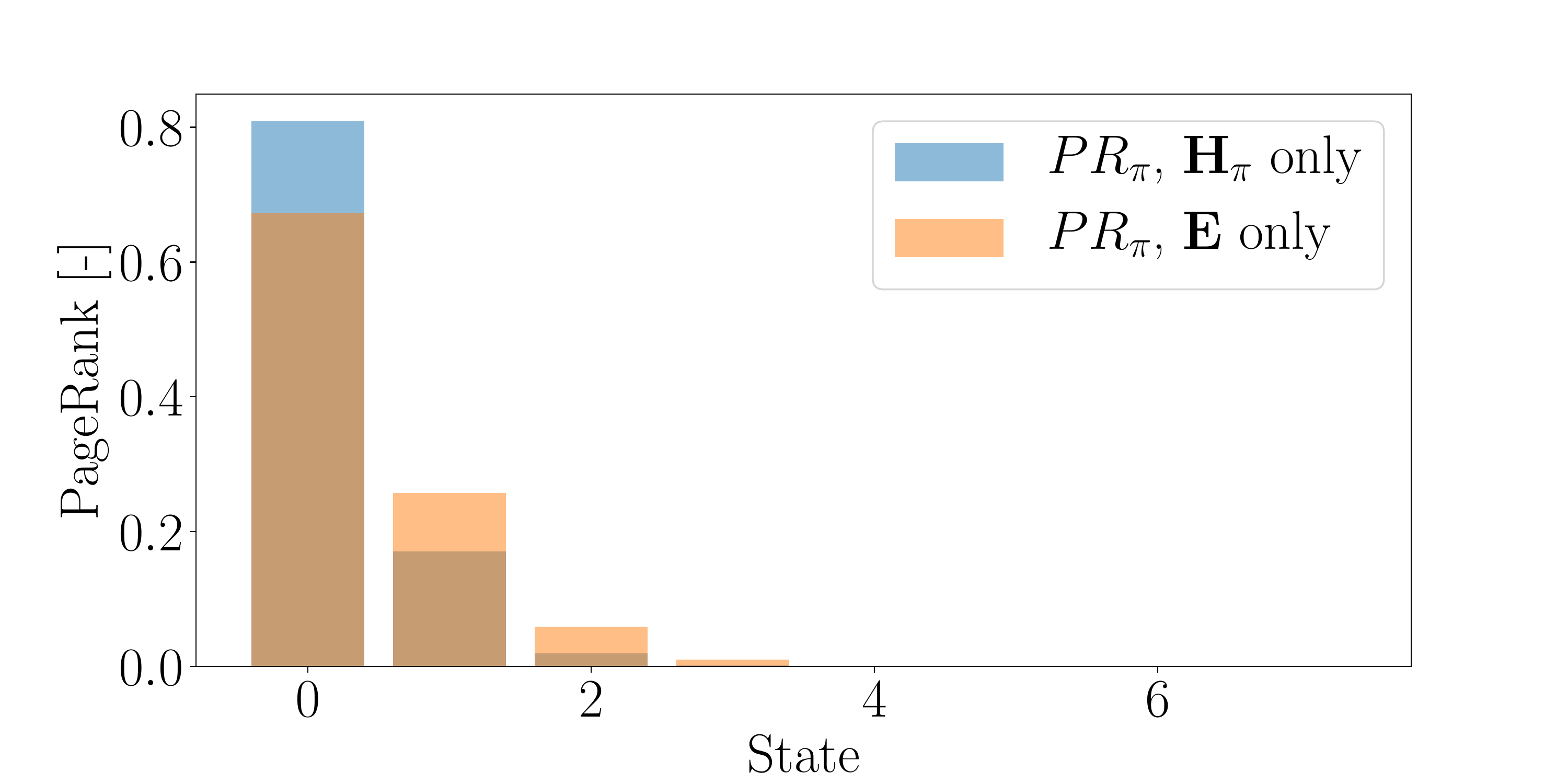}
      \caption{Case study A}
      \label{fig:pagerankoriginalA}
    \end{subfigure}
    \hfill
    \begin{subfigure}[t]{0.49\textwidth}
      \centering
      \includegraphics[width=\textwidth]{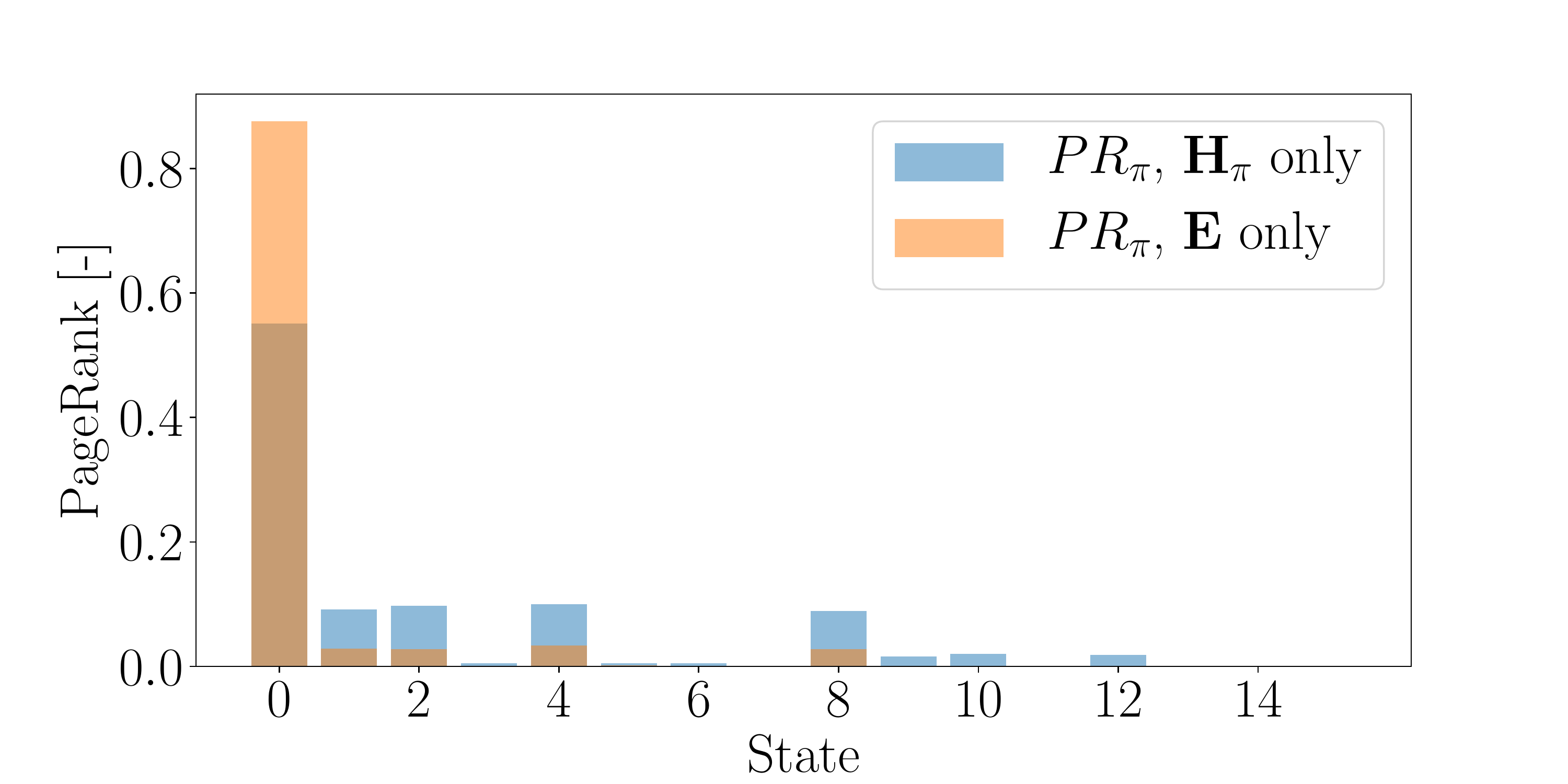}
      \caption{Case study B1}
      \label{fig:pagerankoriginalB1}
    \end{subfigure}
    \begin{subfigure}[t]{0.49\textwidth}
      \centering
      \includegraphics[width=\textwidth]{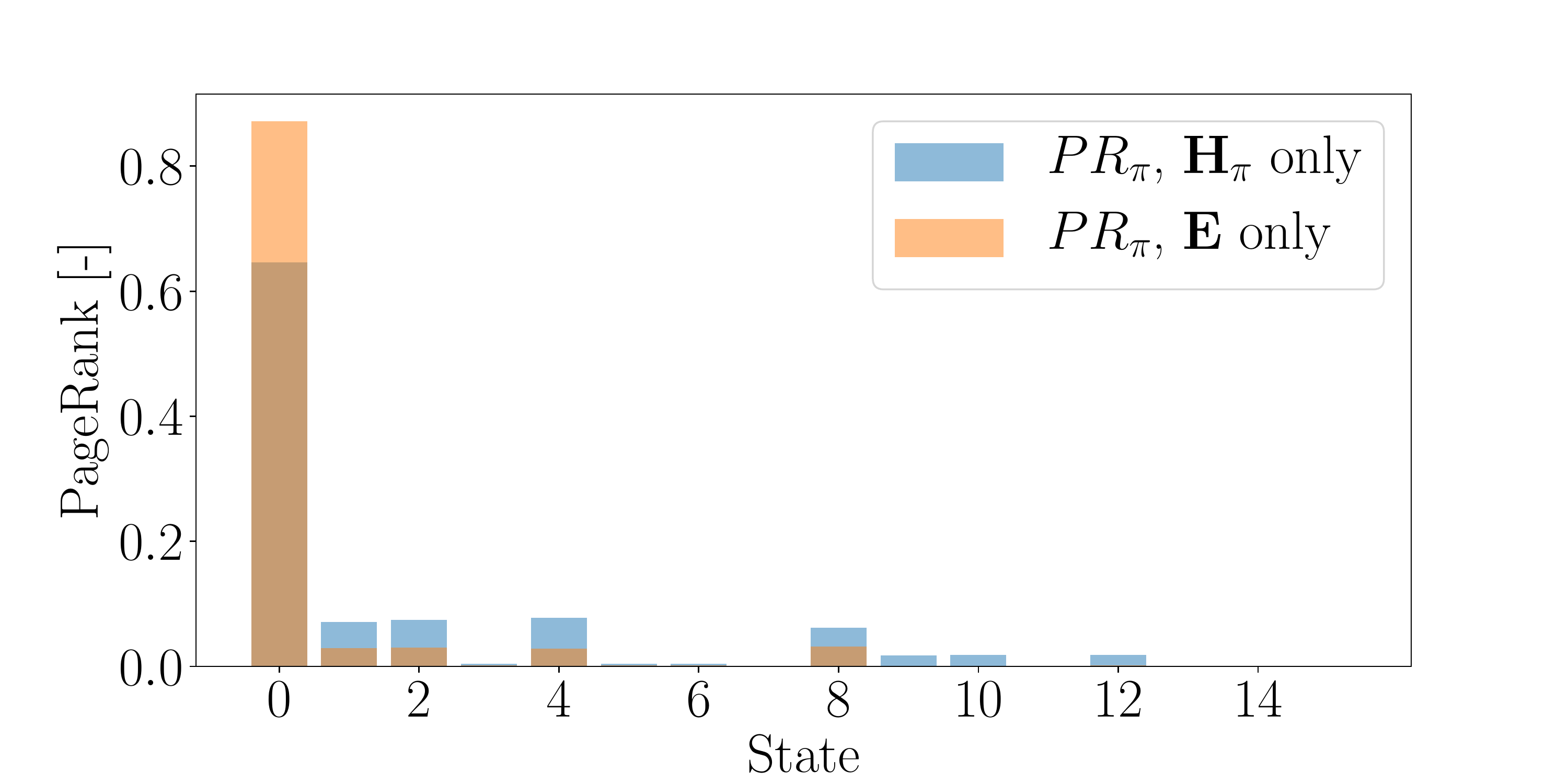}
      \caption{Case study B2}
      \label{fig:pagerankoriginalB2}
    \end{subfigure}
    \hfill
    \begin{subfigure}[t]{0.49\textwidth}
      \centering
      \includegraphics[width=\textwidth]{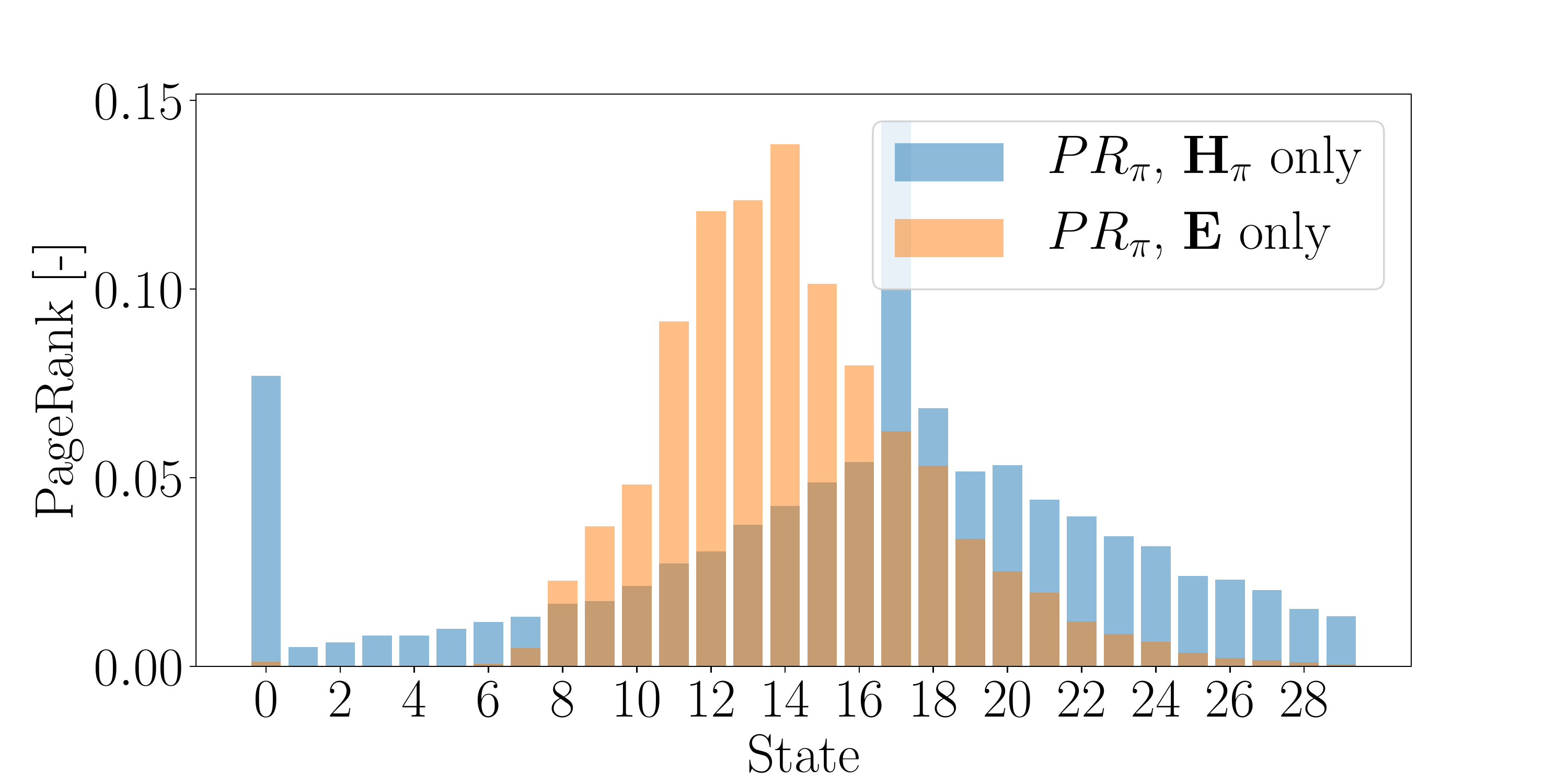}
      \caption{Case study C}
      \label{fig:pagerankoriginalC}
    \end{subfigure}
    \caption{PageRank scores of the active model $G_\mathcal{S}^{H_\pi}$ and the environment model $G_\mathcal{S}^{E}$, from models built using the training data sets.}
    \label{fig:pagerankoriginal}
  \end{figure}

  \begin{figure}[h!]
    \centering
    \begin{subfigure}[t]{0.49\textwidth}
      \centering
      \includegraphics[width=\textwidth]{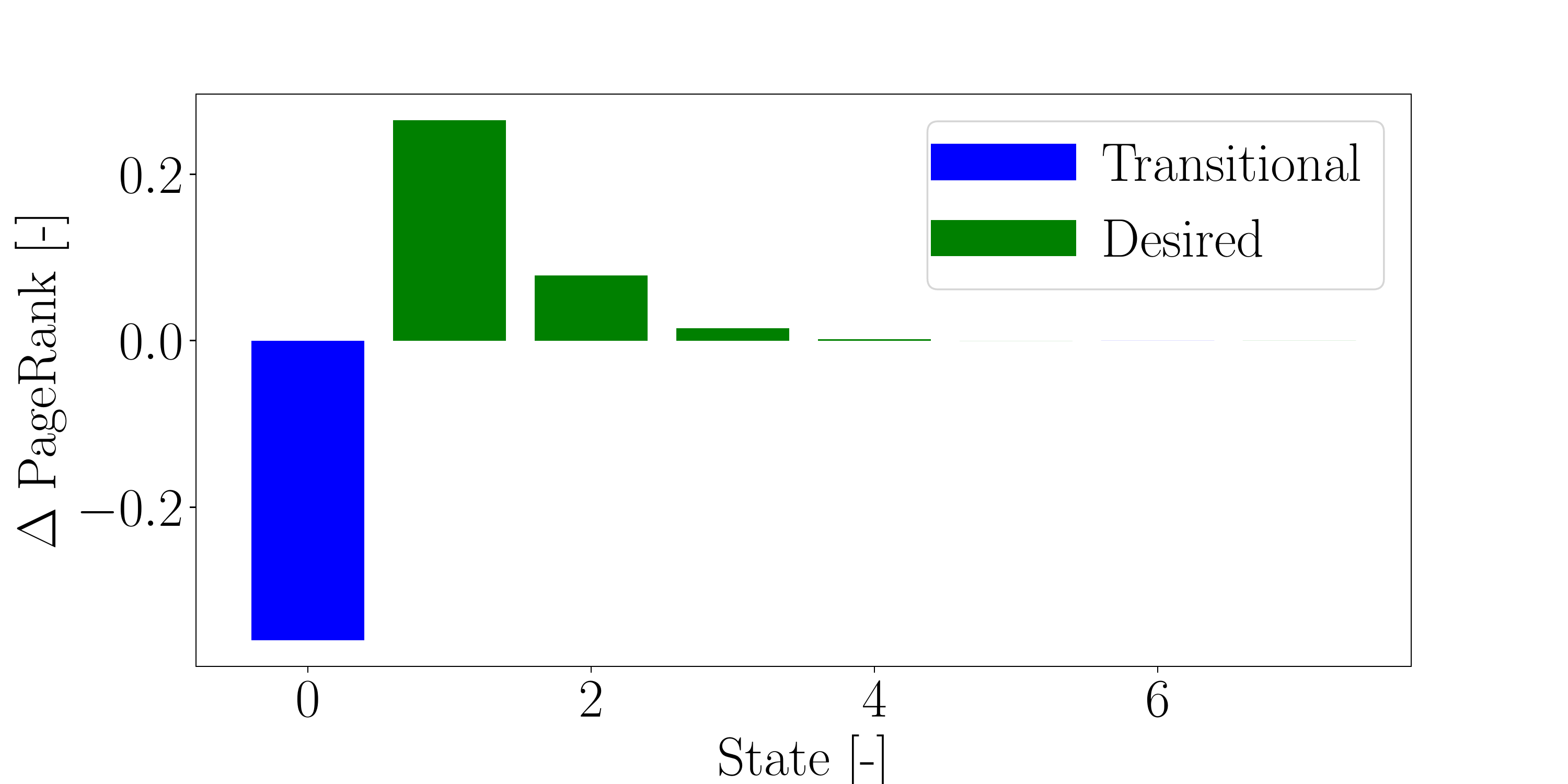}
      \caption{Case study A}
      \label{fig:pagerankdiffA}
    \end{subfigure}
    \hfill
    \begin{subfigure}[t]{0.49\textwidth}
      \centering
      \includegraphics[width=\textwidth]{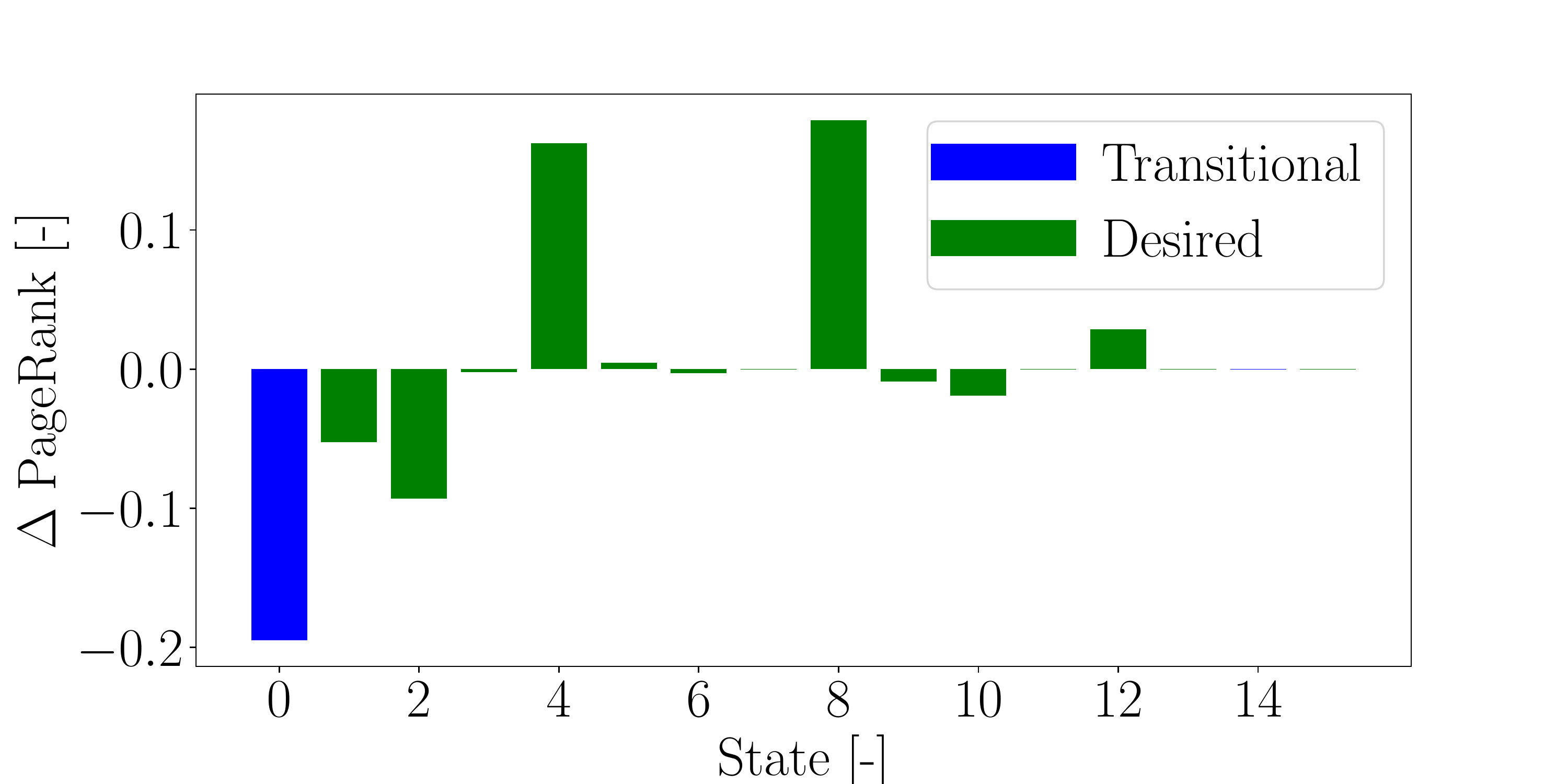}
      \caption{Case study B1}
      \label{fig:pagerankdiffB1}
    \end{subfigure}
    \begin{subfigure}[t]{0.49\textwidth}
      \centering
      \includegraphics[width=\textwidth]{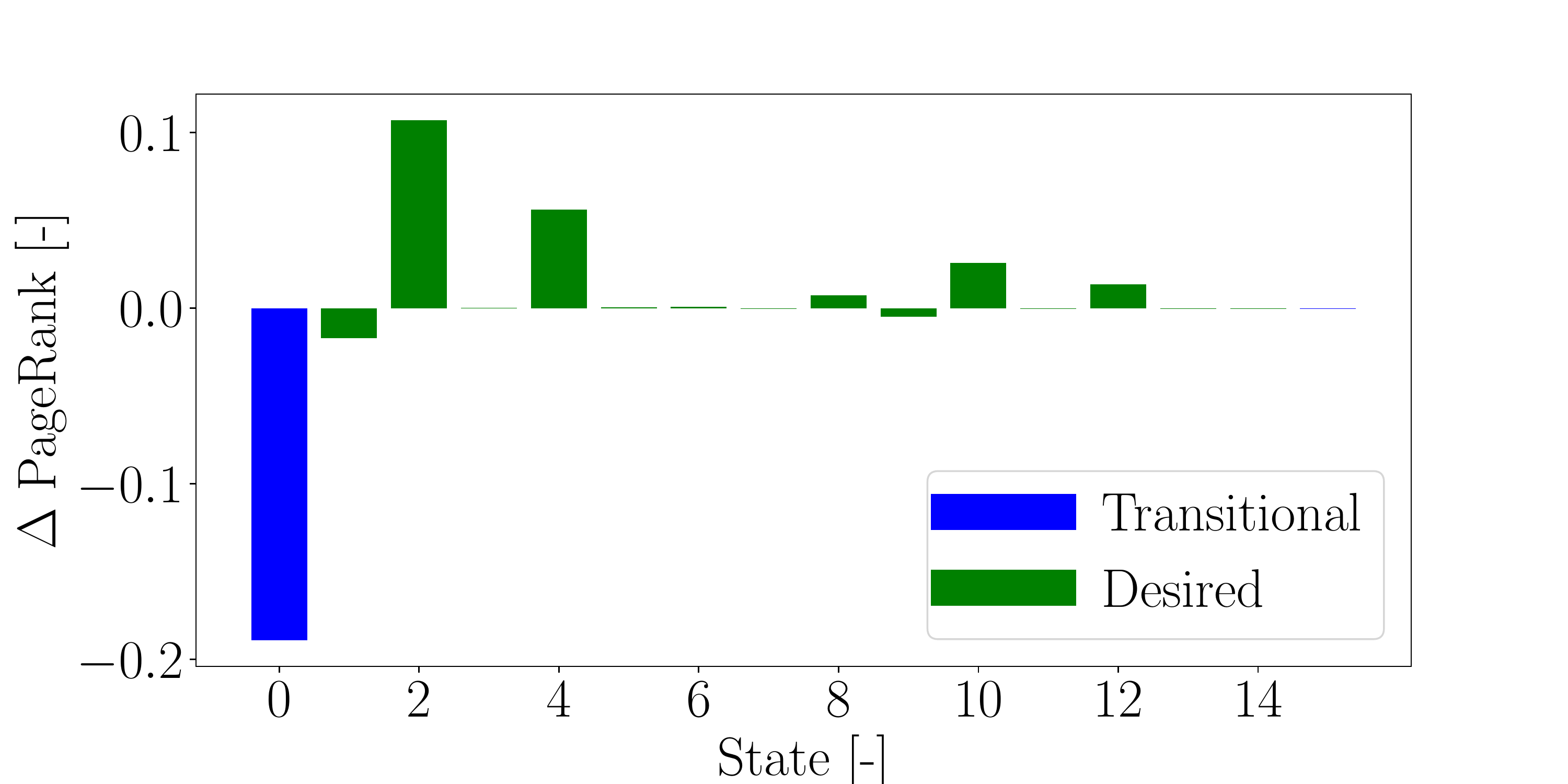}
      \caption{Case study B2}
      \label{fig:pagerankdiffB2}
    \end{subfigure}
    \hfill
    \begin{subfigure}[t]{0.49\textwidth}
      \centering
      \includegraphics[width=\textwidth]{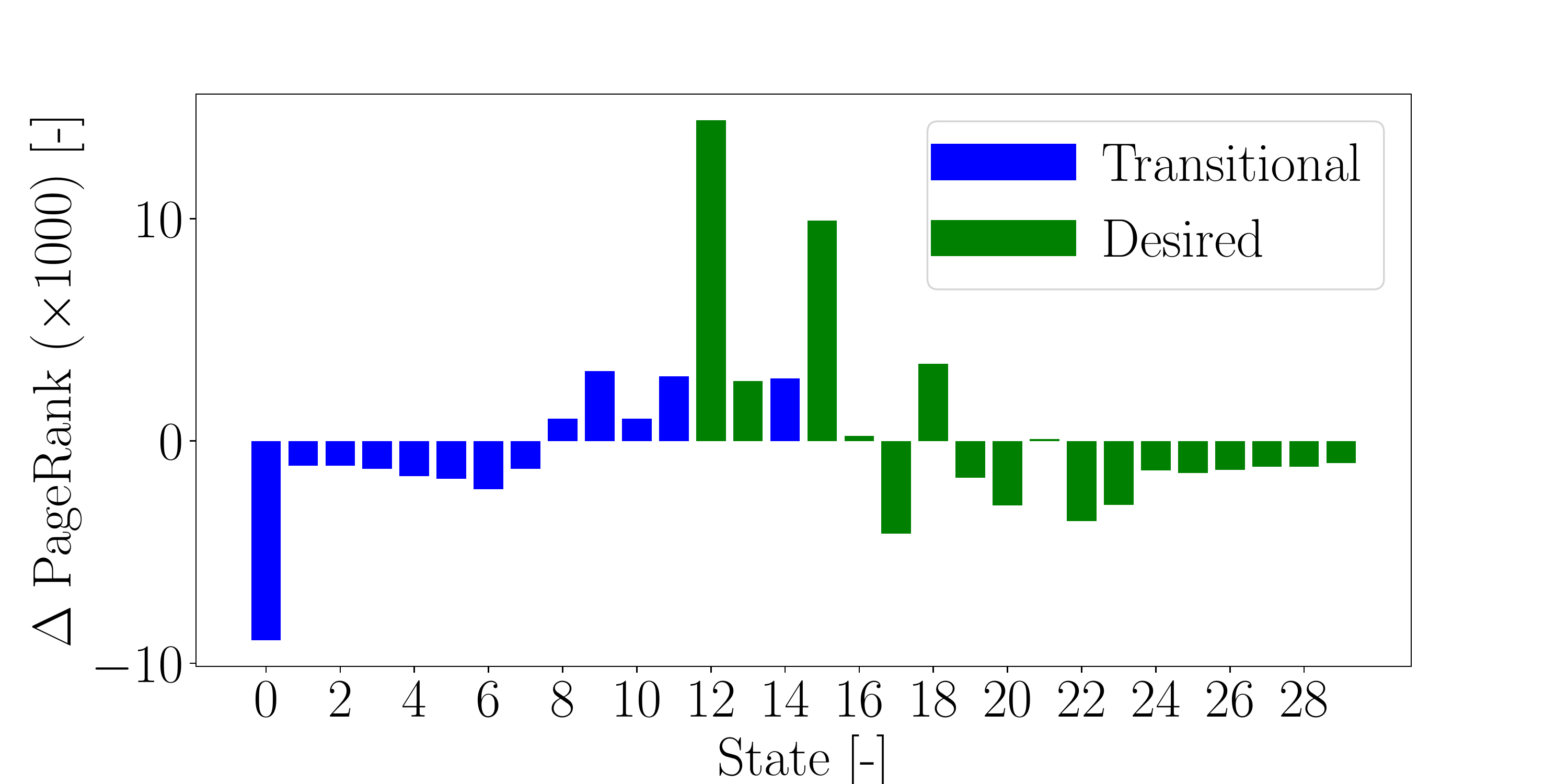}
      \caption{Case study C}
      \label{fig:pagerankdiffC}
    \end{subfigure}
    \caption{Difference in PageRank scores between the models learned using the training data sets (as in \figref{fig:model}), and the estimated model at the end of the policy optimization, given $\pi^\star$.
    We can observe how generally the likelihood for states in the set $\mathcal{S}_{des}$ (``desired'' states) increases while the likelihood of the other states (``transitional'' states) decreases.}
    \label{fig:pagerankdiff}
  \end{figure}
    \subsubsection{Analysis of potential deadlocks}
      Deadlocks can be identified by studying the set of static local states as detailed in \secref{sec:verification_deadlocks}, analyzed here for the four different case studies that have been designed.

      \begin{itemize*}
        
        \item For case study A, the static local state set is $\mathcal{S}_{static}=\{s_1,s_2,s_3,s_4\}$, corresponding to observations with one to four neighbors.
        This means that small aggregate clusters will form, due to the limitation of the robot's sensing.
        This can also be seen to be a likely outcome via \figref{fig:pagerankoriginalA}.
        Note that a similarly greedy behavior was also observed in the policy that was evolved using a standard methodology, due to the time constraint of 200 seconds on each simulation.

        \item For case studies B1 and B2, no direct deadlock could be identified directly from the model as the set $\mathcal{S}_{static} = \emptyset$.
        However, the PageRank analysis of the estimated model reveals that the probability of creating small clusters is high in relation to the probability of creating larger clusters.
        This can be seen in \figref{fig:pagerankoriginal} and \figref{fig:pagerankdiff}, by the significantly larger (change in) PageRank score for those local states (local states $s_1$, $s_2$, $s_4$, and $s_8$, which only feature one filled sector) in relation to the rest.
        
        \item For case study C, the foraging task, deadlock conditions were found from the set $\mathcal{S}_{static}=\{s_0,s_2,s_3,s_4,s_6,s_9,s_{12}\}$.
        If all robots feature these local states, according to model tuned to the optimized policy, they will not take actions (i.e., explore).
        In practice, however, these local states are not truly static.
        They are mitigated by the environment, which will cause the local states to change as the robots reevaluate food at the nest.
        The only true identified deadlock is then the one where all robots have a local state $s_0$, and the swarm will not recover.
      \end{itemize*}

    \subsubsection{Analysis of potential livelocks}

      \tabref{tab:counterexamples} shows the result of checks on the conditions of Propositions \ref{prop:1} and \ref{prop:2} to determine potential livelocks from the transition model with policy $\pi^\star$.
      The results are presented in \tabref{tab:counterexamples} and discussed in detail below.
      
      \begin{itemize*}
        \item For case study A, the results of which are given in \tabref{tab:counterexamplesA}, a potential livelock could be identified due to a missing path from local state $s_0$ (no neighbors) to local states $s_3$, $s_4$, and $s_5$ (three, four, and five neighbors, respectively).
        This means that, in the event that it is necessary for one or more robots in the swarms to achieve one of these local states in order to complete an aggregate, then this may not happen and the swarm may cycle between and/or through less successful global states.

        \item For case studies B1 and B2, no potential livelocks have been identified through the conditions.
        This is because all local states have the potential to transition sufficiently and provide sufficient motion in the swarm.

        \item For case study C, local states $s_0$, $s_2$, $s_3$, $s_4$, $s_6$, and $s_9$ do not feature paths to local states $s\in\mathcal{S}_{des}$ as a result of exploration outside of the nest.
        However, the environment changes are sufficient to avoid livelocks.
        The environment changes in two ways: 
        1) one or more other robots arrive at the nest with food, such that the robots waiting at the nest perceive a difference in food between their reassessment intervals, or 
        2) the robots at the nest consume food, which also leads to perceiving a difference in food between reassessments.
        Both have the potential of liberating the swarm from a livelock.

      \end{itemize*}

  \begin{table}[t]
  \caption{
    Results of check for the conditions in Propositions \ref{prop:1} and \ref{prop:2} following the policy optimization.
    No bounds were used for the probabilities in the policy in order to avoid deadlocks and/or livelocks.
    P 1 refers to the condition of Proposition \ref{prop:1}. 
    P 2.1 and P 2.2 refer to the conditions of Proposition \ref{prop:2}, in order.
  }
  \label{tab:counterexamples}
  \begin{subtable}[t]{\linewidth}
    \centering
    \begin{tabularx}{\textwidth}{| p{0.15\textwidth} | p{0.1\textwidth} | p{0.2\textwidth} | X |}
      \hline \textbf{Condition} & \textbf{Outcome} & \textbf{Counterexamples} & \textbf{Impact} \\ \hline
      {P 1} & 
      False & 
      Missing paths:\newline
      (0, \{3,4,5\})
      & A robot cannot transition from having 0 neighbors to having 3 or more neighbors. This means that, in the event that this is necessary to achieve an equilibrium state at the global level (where having one or two neighbors is not possible), there may be a livelock situation.\\\hline
      {P 2.1} & True & - & - \\\hline
      {P 2.2} &
      False & 
      Missing paths: \newline
      (0, \{3,4,5\}) & Robots can transition to active local states, but a robot with a local state $s_0$ will not transition to local states with 3 or more neighbors. As per the result of P 1, this has the potential to create a livelock at the global level. \\
      \hline
    \end{tabularx}
    \caption{Case study A, 
    $\mathcal{S}_{static}=\{s_1,s_2,s_3,s_4\}$, 
    $\mathcal{S}_{des}=\{s_1,s_2,s_3,s_4,s_5\}$. 
    No information on $s_6$ and $s_7$.}
    \label{tab:counterexamplesA}
  \end{subtable}
  \begin{subtable}[t]{0.45\textwidth}
    \centering
    \begin{tabularx}{\textwidth}{| p{0.33\textwidth} | X|}
      \hline
      \textbf{Condition} & \textbf{Outcome} \\
      \hline
      {P 1  } & True \\
      {P 2.1} & True \\
      {P 2.2} & True \\
      \hline
    \end{tabularx}
    \caption{Case study B1, 
    $\mathcal{S}_{static} = \emptyset$, 
    $\mathcal{S}_{des}=\{s_{1-15}\}$.}
    \label{tab:counterexamplesB1}
  \end{subtable}
  \hfill
  \begin{subtable}[t]{0.45\textwidth}
    \centering
    \begin{tabularx}{\textwidth}{| p{0.33\textwidth} | X |}
      \hline
      \textbf{Condition} & \textbf{Outcome}\\
      \hline
      {P1  } & True \\
      {P2.1} & True \\
      {P2.2} & True \\
      \hline
    \end{tabularx}
    \caption{Case study B2, 
    $\mathcal{S}_{static} = \emptyset$, 
    $\mathcal{S}_{des}=\{s_{1-14}\}$.}
    \label{tab:counterexamplesB2}
  \end{subtable}
    \begin{subtable}[t]{\textwidth}
    \centering
    \begin{tabularx}{\textwidth}{| p{0.15\textwidth} | p{0.1\textwidth}| p{0.2\textwidth} | X |}
      \hline
      \textbf{Condition} & \textbf{Outcome} & \textbf{Counterexamples} & \textbf{Impact} \\
      \hline
      P 1 & 
      False & 
      Missing paths:\newline
      (0, \{12,13,15-29\}), \newline
      (2, \{12,13,15-29\}),\newline
      (3, \{12,13,15-29\}),\newline
      (4, \{12,13,15-29\}),\newline
      (6, \{12,13,15-29\}),\newline
      (9, \{12,13,15-29\})&
      Local states $s_0$, $s_2$, $s_3$, $s_4$, $s_6$, and $s_9$ do not feature paths to $s\in\mathcal{S}_{des}$ as a result of exploration outside of the nest.
      \\
      P 2.1 & 
      True & 
      - & -\\
      P 2.2 & 
      True & 
      - & -\\
      \hline
    \end{tabularx}
    \caption{Case study C, 
    $\mathcal{S}_{static}=\{s_0,s_2,s_3,s_4,s_6,s_9,s_{12}\}$, 
    $\mathcal{S}_{des}=\{s_{12},s_{13},s_{15-29}\}$.}
    \label{tab:counterexamplesC}
  \end{subtable}
  \end{table}

\subsection{Hybrid approach with an evolutionary algorithm}
\label{sec:evo}

  In this section, the model-based framework is used \emph{together} with an evolutionary algorithm that evaluates the performance of the swarm by means of simulation.
  Both models required in the framework are trained directly using the simulation runs that are used to evaluate the population of policies.
  At the end of each generation, the models are trained and a policy, optimized via the PageRank fitness, is generated using these models.
  One member of the population holds the model-based policy at all times.
  Using this hybrid model-augmented approach, the evolution is found to be more efficient.
  In addition, the models that are obtained can be a useful side product to investigate and analyze the behavior of the swarm.

  The performance of the hybrid setup is benchmarked against a standard evolutionary algorithm.
  All evolutions are implemented using the DEAP package, using a population of 100, each evaluated 5 times, in order to assess the mean performance.
  The only difference between the two is the inclusion of the model-based approach as described in the paragraph above.
  The results for case studies A, B1, B2, and C are shown in \figref{fig:evoA}, \figref{fig:evoB1}, \figref{fig:evoB2}, and \figref{fig:evoC}, respectively.
  For all case studies, it can be see that the model-based version approaches the optimal performance more quickly than evolution alone.
  The model-based approach is most impactful in the early stages.
  Note that this is also the case for case study B2, for which the standalone policy optimization results were found to be less successful.

  \begin{figure}[t]
    \centering
    \begin{subfigure}[t]{0.49\textwidth}
      \centering
      \includegraphics[width=\textwidth]{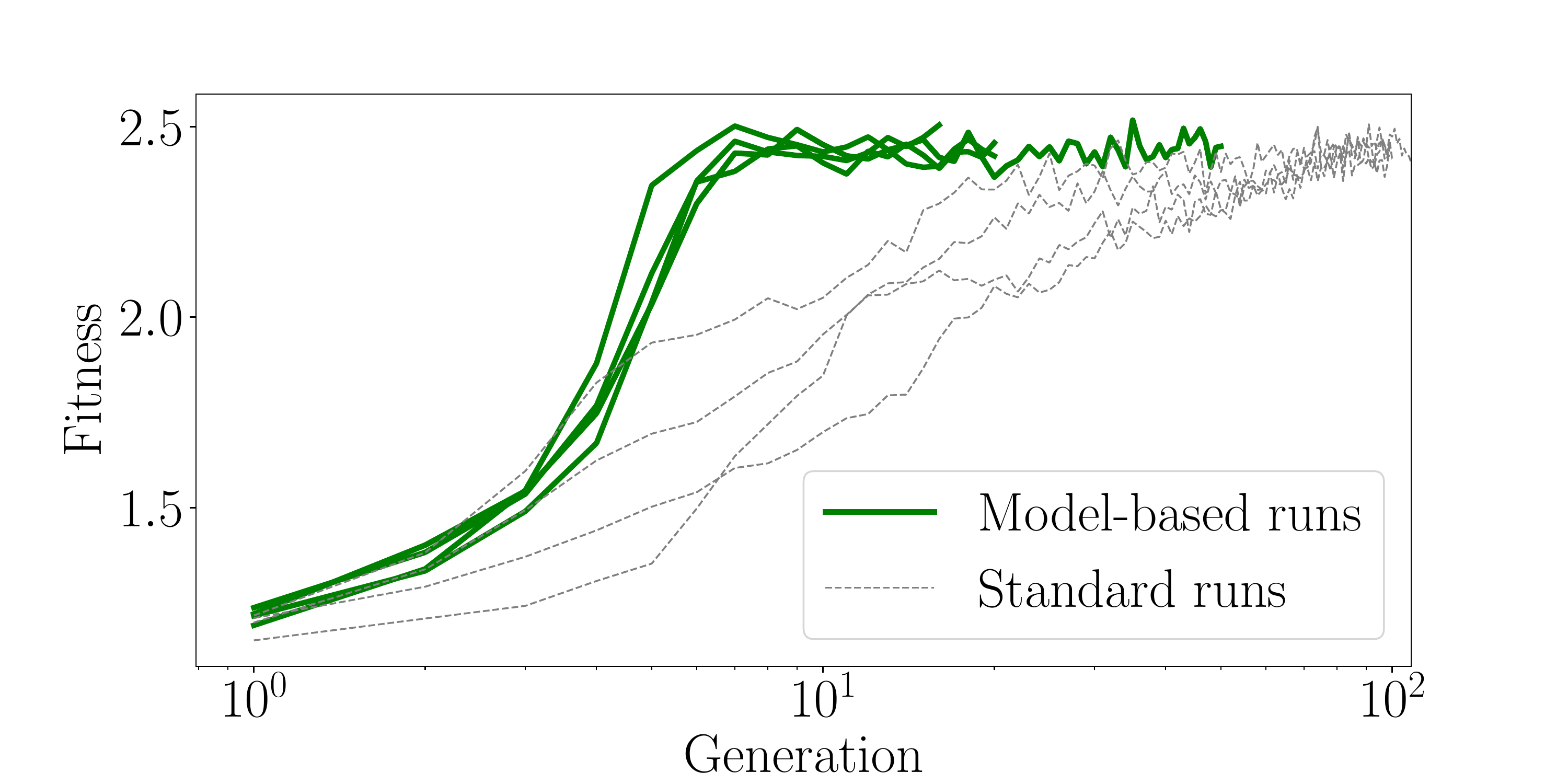}
      \caption{Case study A}
      \label{fig:evoA}
    \end{subfigure}
    \hfill
    \begin{subfigure}[t]{0.49\textwidth}
      \centering
      \includegraphics[width=\textwidth]{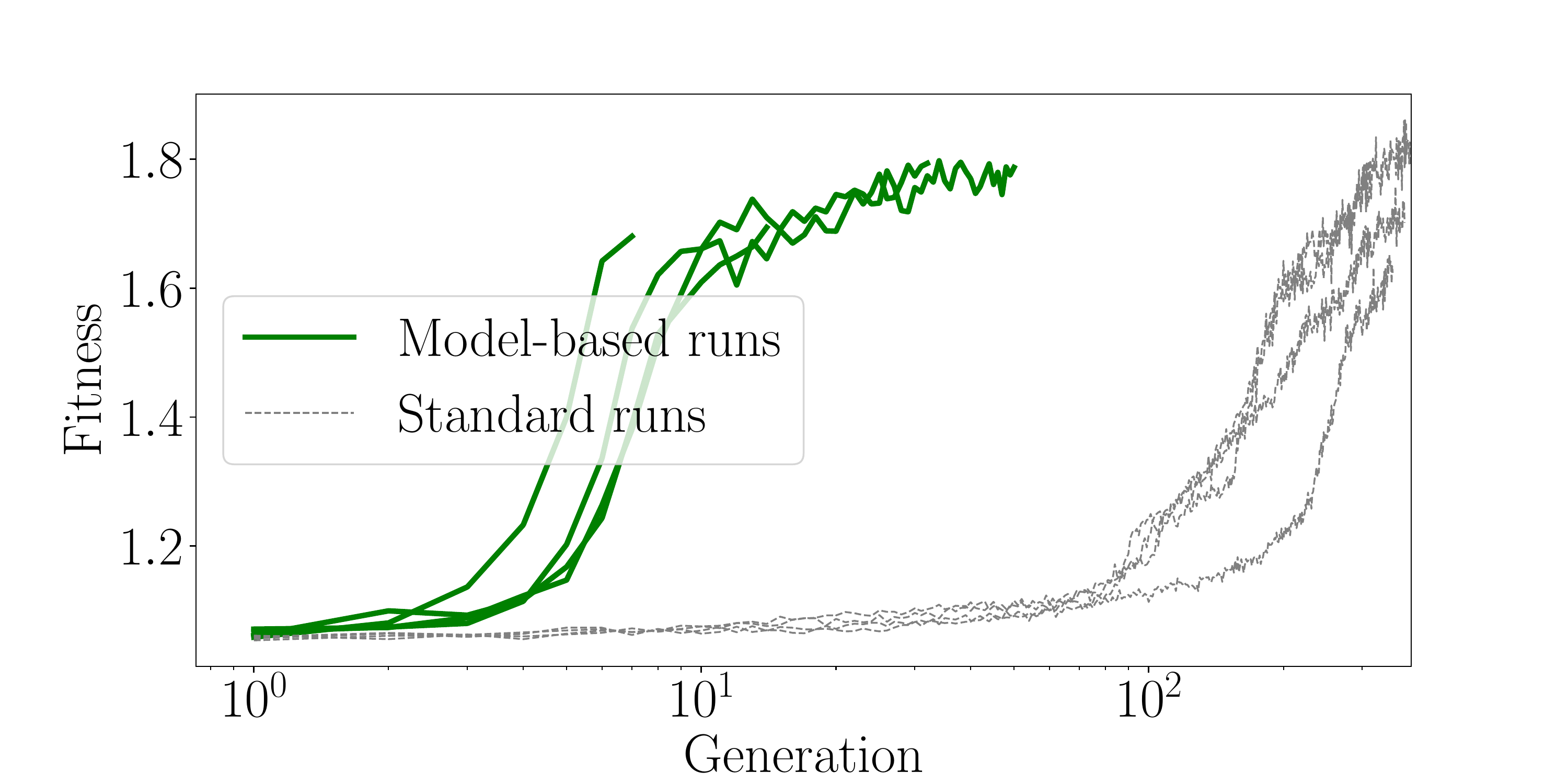}
      \caption{Case study B1}
      \label{fig:evoB1}
    \end{subfigure}
    \begin{subfigure}[t]{0.49\textwidth}
      \centering
      \includegraphics[width=\textwidth]{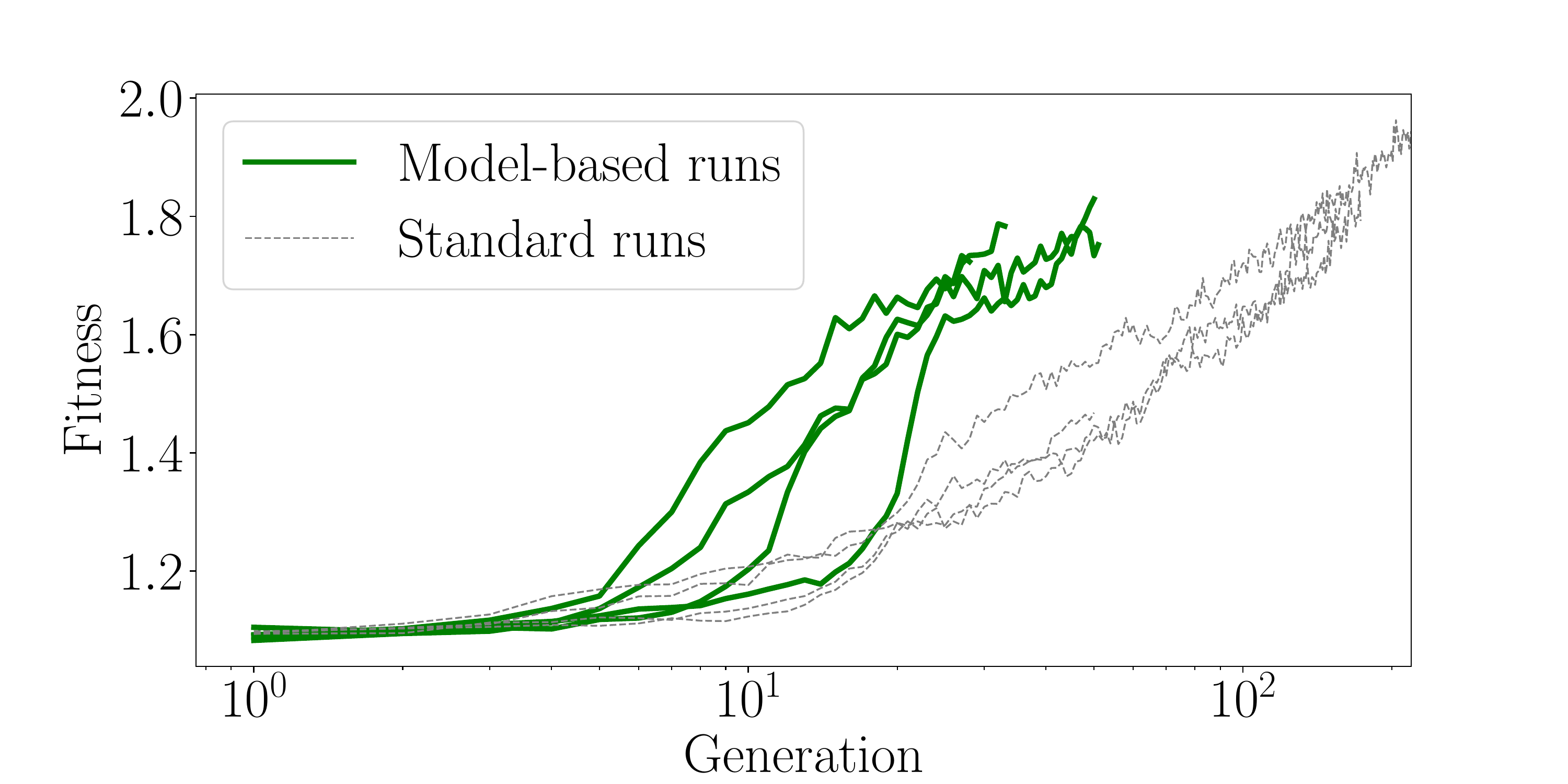}
      \caption{Case study B2}
      \label{fig:evoB2}
    \end{subfigure}
    \hfill
    \begin{subfigure}[t]{0.49\textwidth}
      \centering
      \includegraphics[width=\textwidth]{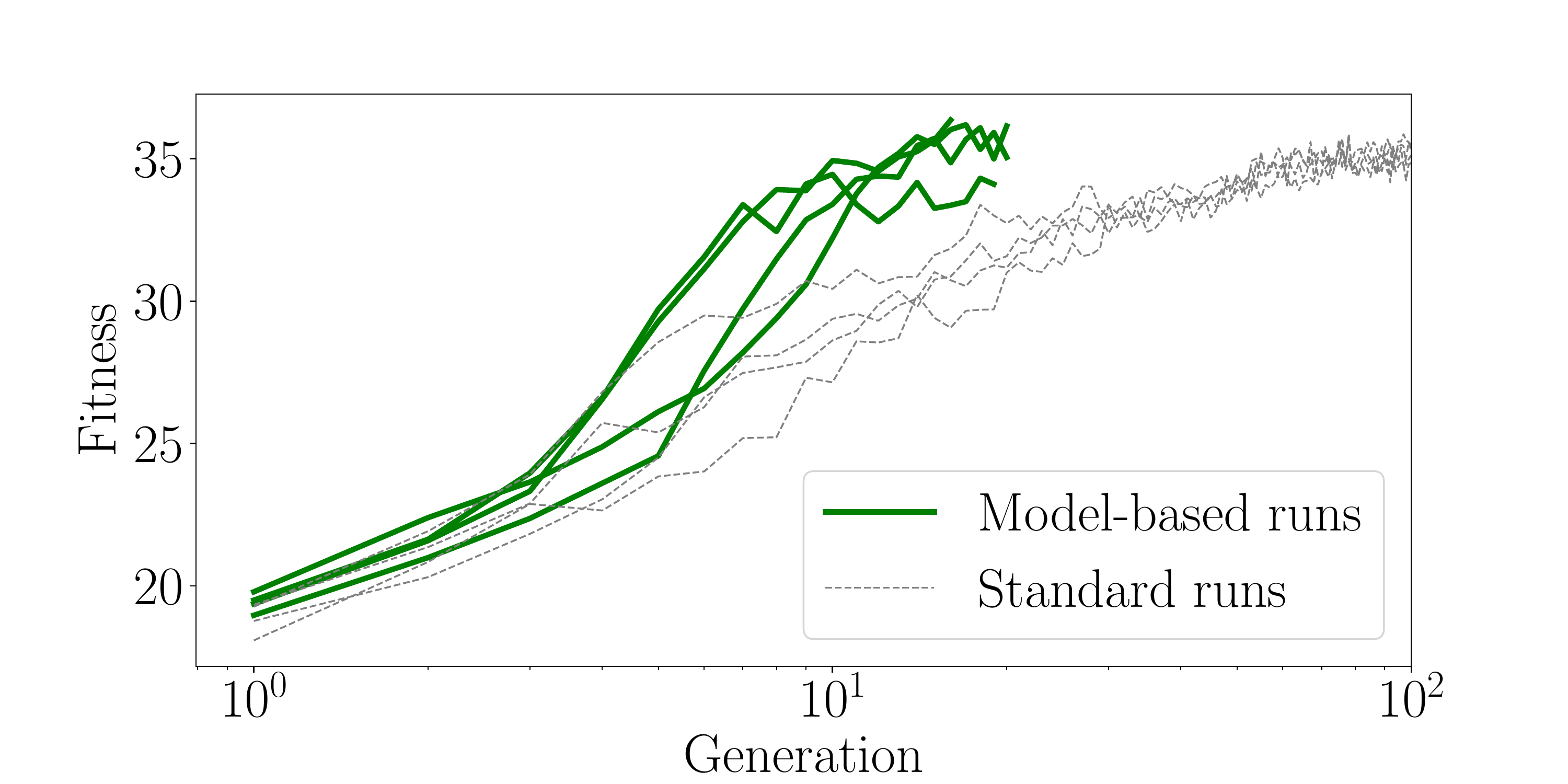}
      \caption{Case study C}
      \label{fig:evoC}
    \end{subfigure}
    \caption{Evolution of optimal policies using a standard evolutionary approach and a model-based hybrid evolutionary approach described in \secref{sec:evo}.
    In all cases, the model-based evolutionary approach outperforms the learning rate of the standard approach, all other parameters being equal. Note that the horizontal axis is logarithmic.}
    \label{fig:evo}
  \end{figure}

\subsection{Framework implementation with online learning}
\label{sec:online}
  This section briefly explores how the proposed framework can also be combined with online learning.
  The transfer learning properties of Model 1 allow the extraction of an adequate set $\mathcal{S}_{des}$ for a given task.
  Given $\mathcal{S}_{des}$, during the online learning phase, the robots only need to estimate the local state transfer model (Model 2) and optimize their policy accordingly.
  As the transition model is local to the robots, it can be estimated online by each robot.
  Two variations are explored:
  \begin{enumerate*}
    \item \emph{Fully local heterogeneous models.}
    In this case, each robot estimates the transition model based on its own experiences.
    In turn, this leads to heterogeneous policies where each robot attempts to achieve the best performance according to its own transition model.
    \item \emph{Shared local model.}
    In this case, all robots combine their experiences into a shared transition model.
    This leads to homogeneous behaviors as all robots base their policy on the same model.
  \end{enumerate*}

  The approach was implemented and studied for case studies A and B1.
  In all cases, no pre-trained version of Model 2 was used.
  Model 2 was estimated entirely online.
  The performance of the strategy was tested for case studies A and B1 over ten independent runs.
  The results are shown in Figures \ref{fig:onlineA} and \ref{fig:onlineB1}.
  It can be seen that the robots in the swarm learn better policies that lead to a higher performance over time, as the model improves.
  As is particularly noticeable for case study A, this is most effective for the shared local model, given that it can converge toward reliable estimates in less time.

  \begin{figure}[t]
    \centering
    \begin{subfigure}[t]{0.49\textwidth}
      \centering
      \includegraphics[width=\textwidth]{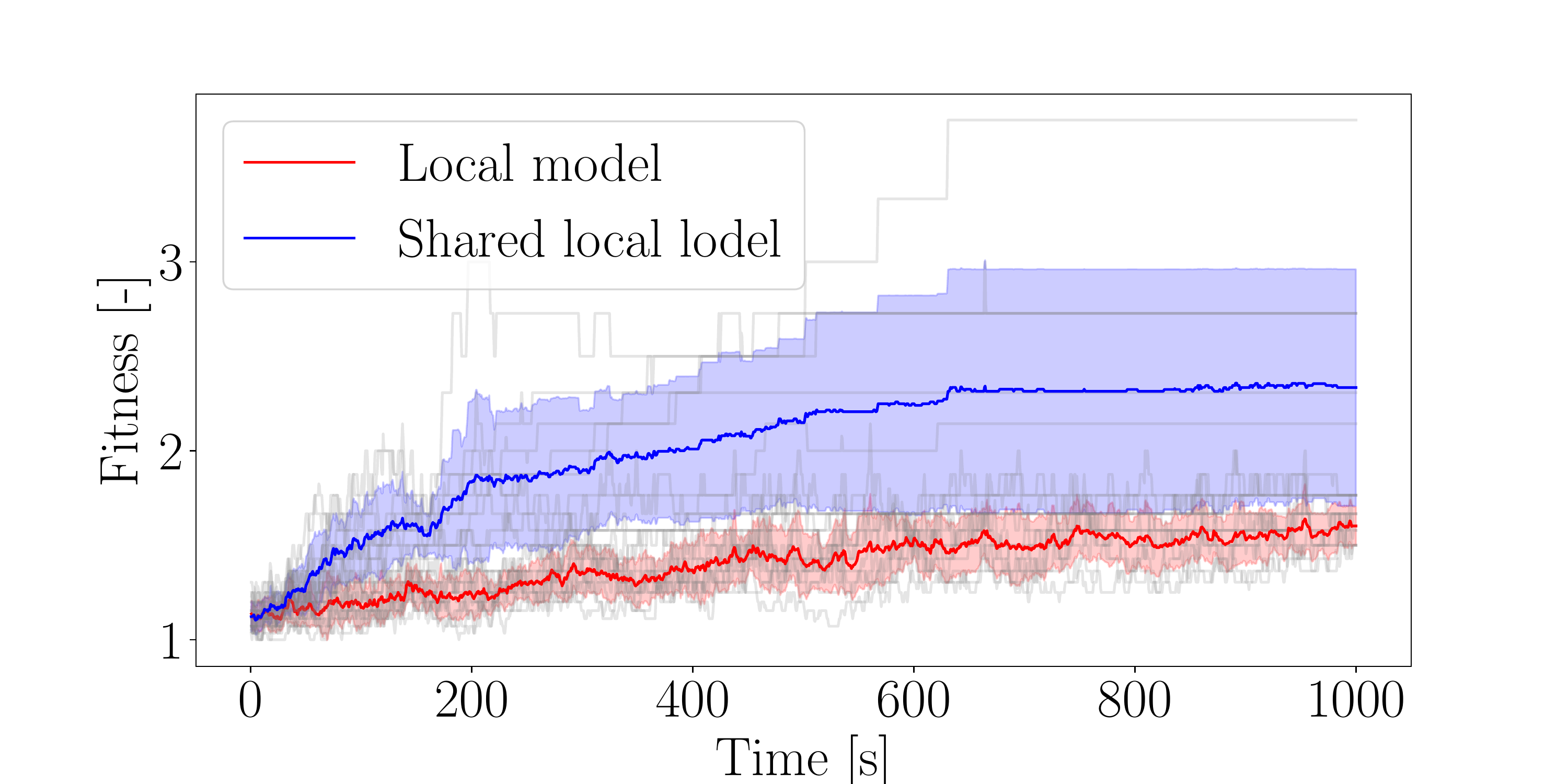}
      \caption{Case study A}
      \label{fig:onlineA}
    \end{subfigure}
    \hfill
    \begin{subfigure}[t]{0.49\textwidth}
      \centering
      \includegraphics[width=\textwidth]{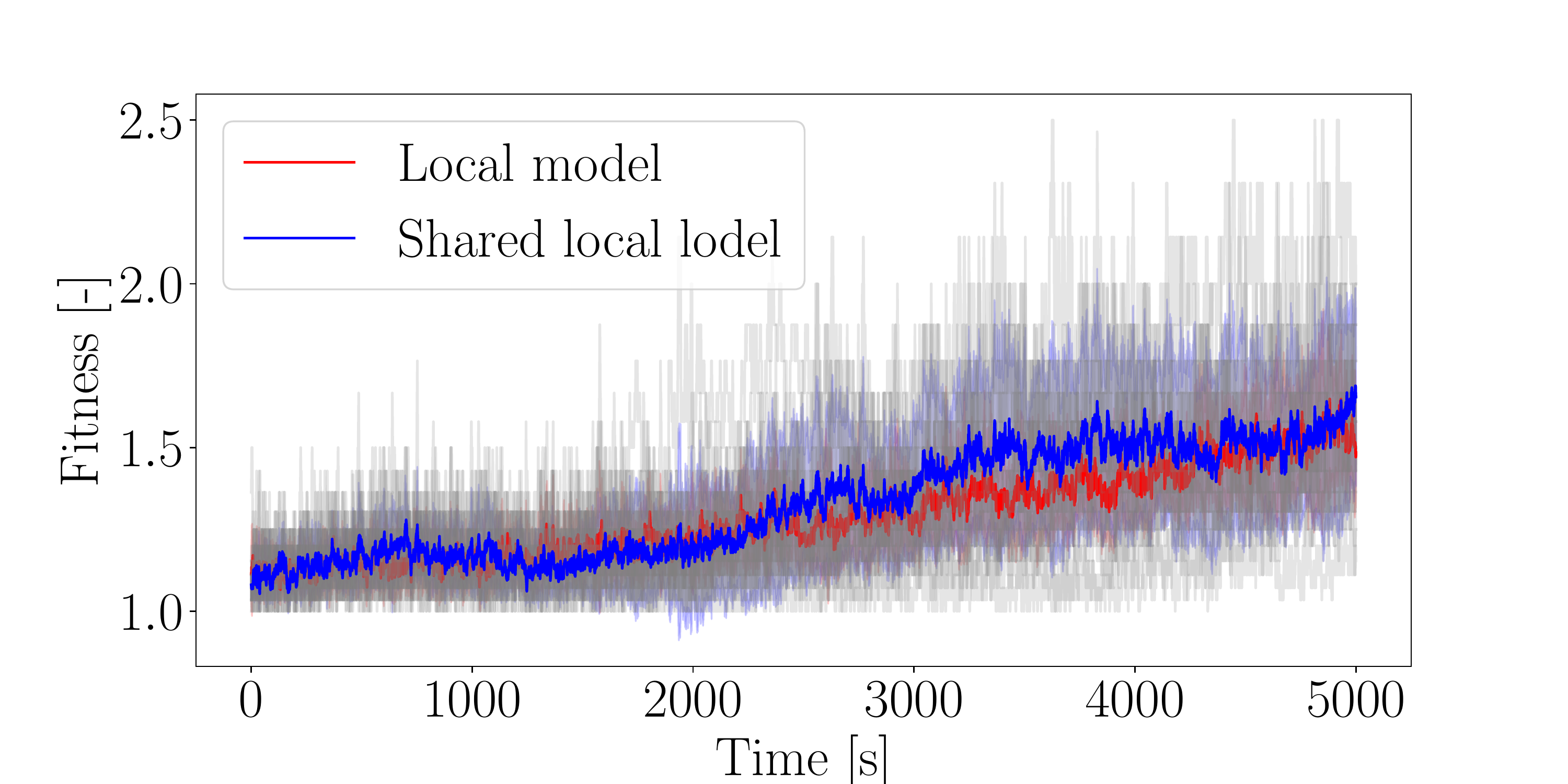}
      \caption{Case study B1}
      \label{fig:onlineB1}
    \end{subfigure}
    \caption{Global performance with online learning during a single run, with no initial transition model.
    The mean performance is shown by the colored lines, with the margins indicating the standard deviation.
    The surrounding gray lines show individual runs.
    For reference, as it can also be seen in Figures \ref{fig:benchmark} and \ref{fig:fitnesslogs}, an evolved run of Study Case A reached, on average, a $F_g(t) \approx 3.0$ at $t=200~s$, and an optimized run of Study Case B1 reached, on average, a $F_g(t) \approx 2.5$ at $t=200~s$.
    Note, however, that the results in this figure are from stand-alone runs, each with a blank model at time $t=0~s$.
    }
    \label{fig:online}
  \end{figure}

\section{Discussion on the proposed framework}
\label{sec:ch5_discussion}
  
  This section discusses the advantages and limitations of the model-based framework proposed in this paper.
  This is done in \secref{sec:ch5_advantages} and \secref{sec:ch5_limitations}, respectively.
  Where relevant, ideas and avenues for future research in the context of swarm robotics research are provided.

  \subsection{Advantages and insights}
  \label{sec:ch5_advantages}

    \begin{itemize}
  
      \item \textbf{Learned micro-macro model.} 
        To the best of our knowledge, this is the first time that a model of the swarm's micro-macro relationship has been learned from data, and then used as a tool to automatically design a swarm behavior.
        The advantages of this are:
          1) it is possible to predict the swarm's global performance from the (distribution of) local states,
          and 
          2) it is possible to extrapolate the local states that increase the global performance.
        In particular, a deep neural network can model a complex relationship directly from data, without a priori knowledge.
        The micro-macro function has shown to be scalable to different swarm sizes, and transferable across swarms and environments.
        This makes it possible to reuse or transfer the model without having to train new models every time.
        Such properties could motivate a larger endeavor by the research community to produce highly accurate micro-macro models and/or training datasets for common swarming tasks.
        For common tasks, datasets can be constructed that could become useful across the research community for purposes extending beyond the one in this paper.
        Similar community-wide efforts have been successful in other fields, such as computer vision, which now holds a plethora of publicly available and high quality models and datasets.
        
      \item \textbf{Transparent transition model.} 
        The transition model was used to model the expected local state transitions experienced by the robots.
        A PFSM architecture was chosen specifically to increase transparency and traceability.
        It was readily used for policy optimization and analysis.
        Its transparency enabled us to inspect and predict elements of a swarm's behavior, without solely relying on empirical testing.

      \item \textbf{Sample efficient policy optimization.} 
        The number of simulations required to train Model 1 and Model 2 were comparable with the ones required by a single generation of a standard (i.e., simulation-based) evolutionary approach.
        The reason for this sample efficiency is that the model-based framework exploits all the events in a simulation run, rather than only assessing its cumulative/final performance.
        The hybrid setup introduced in \secref{sec:evo} provides an interesting way to combine the benefits of both approaches.

      \item \textbf{Online model-based learning.} 
        Online learning and adaptability are complex challenges for swarm robotics. 
        Using the approach in this paper, the robots can generate a model of their experiences and adapt their behavior accordingly, without any a priori knowledge.
        Although the robots act greedily with respect to their own local model (as no global information is available) their common local objectives can guide them toward a high global performance.
        One primary issue is the balance between exploration and exploitation, as exploration implies slower convergence, yet exploitation implies that the solution will likely be subpar.
        In practice, online learning may be best used as a fine-tuning approach for the purposes of adaptability, whereby the robots are loaded with a transition model that is fine-tuned online.
    \end{itemize}

  \subsection{Limitations}
  \label{sec:ch5_limitations}

    \begin{itemize}

    \item \textbf{Difficulties with delayed/temporal effects.}
      The micro-macro model implemented in this work faced issues with temporal effects.
      These were apparent with case study C (foraging).
      The temporal effects in case study C were the results of the following aspects.
      \begin{enumerate}
        \item The value of the global fitness $F_g(t)$, given the same vector $\mathbf{P_s}(t)$, depends on prior actions/global states.
        \item As food is replenished in the environment, a bias forms toward areas that are further from the nest and/or harder to reach.
        This bias increases over time, as the robots first find the food near the nest before venturing out further.
      \end{enumerate}
      These effects could not be captured by the feed-forward setup of Model 1.
      Temporal properties could be represented using a recurrent neural network instead of a feed-forward one.
      However, this would also require changes in the remainder of the framework.

    \item \textbf{Inaccurate transition model.}
      An inaccurate transition model can be responsible for an inaccurate assessment of the system as well as a policy with poor performance.
      One solution is to improve the model via additional simulations in order to improve the estimates, potentially also building on other models of similar systems, if available.
      In addition, it is possible for the robots to fine-tune the model online, and then re-optimize their policy accordingly.
      Fine-tuning a model online can also help to overcome situations that were not simulated or modeled.
      Alternative model structures should also be investigated aside from the tabular model used here.
    
    \item \textbf{Sub-optimal policy optimization.}
      Not all case studies matched the performance of the standard evolutionary procedure. 
      This is because the performance achieved by the policies is inherently sub-optimal as a result of a policy optimization based on Model 2.
      Whereas a global optimizer can tune to the specific environment, this cannot be directly done when optimizing based on a local model.
      The policy optimization, which is based on a local model, prompts greedy behaviors by the individual robots.
      However, the common local objectives, which maximize the global performance, can guide this behavior to a behavior that benefits the swarm as a whole.
      In addition, temporal aspects pertaining to the performance are not modeled.
      This means that the robots aim to achieve their objectives as quickly as possible, whereas a global optimizer may optimize the behavior for the time given to the task.
      Finally, the transition model assumes that the other robots in the swarm behave in a certain way, which is an assumption that is violated once the policy is optimized and the behavior of all robots changes.

    \item \textbf{State space and action space scalability in the transition model.} 
      The discretized nature of PFSMs has two limitations: 1) it prevents continuous states from being used, and 2) it is subject to scalability issues as the size of the local state space and action space increases, which increases the PageRank evaluation time.
      Approximate model architectures can help to mitigate this issue and should be the subject of future research.
    \end{itemize}

\section{Conclusion}
\label{sec:b5_conclusion}

  This paper proposed a model-based framework to extract local behaviors for robotic swarms.
  This framework uses a neural network model to estimate the effect of local states on the global performance.
  As this model is a direct mapping of local states to global performance, it is not dependent on the dynamics of the swarm and it can be potentially transferred across different swarms and environments.
  The model is used to find desired local states, i.e., local states that are expected to maximize the performance of the swarm.
  A second model, a probabilistic model of the local state transitions that a robot in the swarm can experience, is then used to optimize and analyze the behavior.
  The complexity of this second step does not increase with the number of robots in the swarm, which is an attractive property.
  Overall, the behavior can be found very efficiently.
  Together with the experience necessary to learn the models, it approximately requires the equivalent of a single generation in an evolutionary algorithm.
  This model-based framework provides a way to develop behaviors that are understandable and to which we can also apply verification checks in order to determine possible pitfalls.

  Future work should be focused on experimenting with alternative model architectures, potentially solving problems pertaining to temporal aspects as well as discretization.
  This, however, should be done while still keeping transparency and verifiability in mind.
  Additionally, we encourage further investigations into how the framework can be integrated within current practices, building on the setups explored in \secref{sec:evo} and \secref{sec:online}.
  In the hybrid evolutionary setup, for example, the models allowed a performance improvements with marginal computational increase, as the framework directly uses the simulation logs used to evaluate a controller in order to construct and use the models.
  Further evaluations need to be performed in order to better understand the limitations and advantages of the framework within these contexts.
  Finally, we believe that a community-wide effort to build high quality models and datasets for swarms (robotic, but also natural) could be useful across the research community, for purposes extending beyond the scope of this work.

\section*{Code}

  \begin{itemize}
    \item The code used in this paper can be downloaded at 
    \url{https://www.github.com/coppolam/SI_framework}.
    This also includes scripts that will automatically download and setup the simulator and the data.
    \item The simulator can be downloaded separated at \url{https://github.com/coppolam/swarmulator/tree/SI_framework}
    \item All data used in this paper can be downloaded at: \url{https://surfdrive.surf.nl/files/index.php/s/DtU5rW7za4DeNb5}.
  \end{itemize}

\bibliographystyle{plainnat}
\bibliography{bibliography}

\end{document}